\documentclass{article} 
\usepackage{iclr2027_conference_arxiv,times}

\usepackage{amsmath,amsfonts,bm}

\def\eqref#1{equation~\ref{#1}}

\def\1{\bm{1}}

\DeclareMathAlphabet{\mathsfit}{\encodingdefault}{\sfdefault}{m}{sl}
\SetMathAlphabet{\mathsfit}{bold}{\encodingdefault}{\sfdefault}{bx}{n}

\DeclareMathOperator*{\argmax}{arg\,max}
\DeclareMathOperator*{\argmin}{arg\,min}

\usepackage[utf8]{inputenc} 
\usepackage[T1]{fontenc}    
\usepackage[colorlinks]{hyperref}       
\usepackage{url}            
\usepackage{booktabs}       
\usepackage{amsfonts}       
\usepackage{nicefrac}       
\usepackage{microtype}      
\usepackage{xcolor}         
\usepackage{mathtools}          
\usepackage{mathrsfs}           
\usepackage{graphicx}           
\usepackage{subcaption}         
\usepackage[space]{grffile}     
\usepackage{url}                
\usepackage{lipsum}             
\usepackage{amsthm,amsmath,amsfonts,environ,amssymb}
\usepackage[makeroom]{cancel}
\usepackage{standalone}
\usepackage{pgfplots}
\usepackage{wrapfig}
\usepackage{multicol}
\usepackage[addedmarkup=uline, defaultcolor=magenta, todonotes={textsize=scriptsize, textwidth=2cm}, authormarkuptext=name, commandnameprefix=always, xcolor]{changes}
\usepackage{dsfont}
\usepackage[scr=boondox]{mathalfa}
\usepackage{bbm}

\hypersetup{urlcolor=gray, citecolor=gray, linkcolor=gray}

\pgfplotsset{compat=1.18}
\pgfplotsset{set layers}
\usepgfplotslibrary{colormaps}
\usepgfplotslibrary{fillbetween}
\usepgfplotslibrary{patchplots}
\usetikzlibrary{positioning}
\newcommand{\norm}[1]{\left\lVert #1 \right\rVert}
\newcommand*\dd{\mathop{}\!\mathrm{d}}
\newcommand{\cS}{\mathcal{S}}

\newcommand{\cD}{\mathcal{D}}

\newcommand{\cQ}{\mathcal{Q}}

\usepackage{thmtools}
\usepackage{thm-restate}

\newtheorem{theorem}{Theorem}
\newtheorem{proposition}{Proposition}
\newtheorem{lemma}{Lemma}
\newtheorem{corollary}{Corollary}

\newtheorem{definition}{Definition}

\newtheorem{remark}{Remark}

\definechangesauthor[name={NR}, color=gray]{nr}

\usepackage{silence}
\pgfplotsset{
    colormap={custom_plasma}{
        rgb255=(13, 8, 135)
        rgb255=(84, 2, 163)
        rgb255=(156, 23, 158)
        rgb255=(208, 70, 107)
        rgb255=(240, 117, 78)
        rgb255=(253, 169, 41)
        rgb255=(240, 249, 33)
    }
}

\title{
A Unified Bellman Operator for
Safety-Critical Reinforcement Learning
}

\author{%
  Nishanth Arun Rao\textsuperscript{1,*} \quad
  Royina Karegoudra Jayanth\textsuperscript{2,*} \\
  \textbf{Benjamin Eysenbach}\textsuperscript{2} \quad
  \textbf{Jaime Fernández Fisac}\textsuperscript{1} \\
  \textsuperscript{1}Department of Electrical and Computer Engineering \\
  \textsuperscript{2}Department of Computer Science \\
  Princeton University \\
  \texttt{\{nr1053, rj5498, eysenbach, jfisac\}@princeton.edu}
}

\iclrarxivcopy
\begin{document}

\maketitle
\let\thefootnote\relax\footnotetext{*Equal contribution}

\begin{abstract}
Reinforcement learning in safety-critical domains requires maximizing task performance while strictly adhering to safety constraints. Existing safe reinforcement learning paradigms \textit{typically} force a trade-off: they either require \textit{a priori} knowledge to provide strict safety guarantees (e.g., safety filters), or they enable joint learning but only satisfy safety constraints on average. In this work, we propose a novel Bellman operator that unifies performance and safety objectives into a \textit{joint} value function.
  We show that temporal difference learning with the joint Bellman operator converges under a two-timescale stochastic approximation framework. On the fast timescale, the \textit{safety value} of the learning joint policy is estimated, while the \textit{joint value} is estimated on the slow timescale. Convergence is ensured by formulating the limiting dynamics as an \textit{occupation-averaged} differential inclusion, and showing that it asymptotically converges to a set of limiting optimal safety-constrained task value functions. \textit{Theoretically}, once converged, the resulting optimal policy maximizes task return while maintaining safety at all times. Empirical evaluations on continuous control tasks with \textit{neural approximations} demonstrate stable convergence with near-zero safety violations at test time.
\end{abstract}

\section{Introduction}\label{sec:introduction}

Biological agents instinctively prioritize safety while mastering diverse tasks to ensure efficient, harm-free exploration. Safe Reinforcement Learning (Safe RL) captures this paradigm, aiming to learn near-optimal policies with strict safety assurances. Existing methods typically formalize this balance either by treating \textit{safety as a soft constraint} during optimization or by theoretically enforcing \textit{hard safety guarantees}.

Some of the early works formalized safe RL as a \textit{Constrained Markov Decision Process} (CMDP), solving safety-constrained optimization problems using Lagrangian relaxation with dual multiplier methods \citep{altman1999constrained, garcia2015comprehensive, ray2019benchmarking}, with subsequent extensions via trust-region projections \citep{achiam2017constrained, yang2020projection, xu2021crpo, milosevic2025embedding}, state-augmentation \citep{sootla2022saute}, and offline distribution shaping \citep{yao2024oasis}. While these model-free techniques jointly optimize task performance and safety, they inherently bound \textit{expected cumulative costs}, yielding only probabilistic safety guarantees. Conversely, model-based methods achieve strict safety via \textit{predictive rollouts}, \textit{epigraph formulation} of the value function (for deterministic control problems), or by strictly avoiding unsafe future sets \citep{huangsafedreamer, jayant2022model, pfrommer2022safe, altarovici2013general}. However, they suffer from compounding model biases and computational overhead to infer safe actions in real-time, particularly in environments with complex or unmodeled dynamics. Motivated by these tradeoffs, we seek a model-free approach that strictly adheres to safety specifications while optimizing task performance.

Some of the control-theoretic approaches that rigorously hard-guarantee safety are based on \textit{Hamilton-Jacobi (HJ) reachability analysis}, casting safety and reach-avoid problems as \textit{zero-sum differential games} \citep{isaacs1965differential, mitchell2005time}. While deep RL techniques have enabled tractable neural solutions for the time-discounted HJI avoid/reach-avoid Bellman operators for safety analysis \citep{fisac2019bridging, hsu2021safety, bajcsy2019efficient}, it is not immediately obvious how these techniques can be extended to generic tasks specified by dense, continuous reward functions. Consequently, a dominant paradigm employs a \textit{decoupled approach}, where an \textit{a priori learned} task-oblivious \textit{safety filter} preemptively intercepts and projects nominal unsafe task actions onto a plausible safe set \citep{bui2021real, hsu2024safety_filter, nguyen2024gameplay, cai2025guardian, banerjee2024dynamic}. However, these switch-type interventions only occur near the boundary, causing erratic, jittery interventions when a safety-oblivious task policy persistently tries to exit the boundary. This is seldom desirable, especially in human-in-the-loop scenarios \citep{oh2025safety}.

Recent offline \textit{feasibility-guided models} instead utilize diffusion, flow-matching, or world-models to approximate the largest feasible region using HJ-reachability and steer task-driven actions to strictly lie within the safe envelope \citep{fisor2024, safefql2026, caofosp}. However, they either introduce iterative online inference latency or rely on post-hoc distillation that weakens strict safety guarantees. Alternatively, using \textit{precomputed} safety value functions as \textit{Control Barrier Functions} ($Q$-CBF) \citep{hsu2024safety_filter, pmlr-v283-lu25a, oh2025safety, oh2026synthesisdeploymentmaximalrobust} avoids planning latency, but necessitates computing action projections online. Additionally, relying on external projections introduces \textit{action aliasing}, which can be detrimental \citep{markgraf2026safe}. If the safe action set is known a priori, neural-based methods such as \citep{suttle2024sampling} can hard-guarantee safety; we instead assume \textit{no access to an oracle safe action set} and learn from data. Thus, a growing subset of literature advocates for \textit{unified architectures} that embed safety directly into the policy \citep{pcrpo2024, gu2024enhancing, pmlr-v235-cen24b}. Yet, jointly optimizing rewards and costs frequently destabilizes learning due to gradient conflicts,  necessitating algorithmic workarounds like gradient manipulation and explicit latent decoupling \citep{pcrpo2024, gu2024enhancing, kim2025conflictaverse, pmlr-v235-cen24b, yangq}.

In this paper, we combine \textit{safety and task objectives directly at the value-function level}, rather than at the policy level, and jointly optimize \textit{safety-constrained task performance} using a single \textit{joint policy} via a novel, \textit{joint Bellman operator} under adversarial disturbances. Our contributions are: 

\textbf{Unified Formulation:} We propose a robust Bellman operator 
and derive the corresponding objective function from robust optimal stopping theory. 
We propose a two-timescale stochastic approximation framework that enables learning of its fixed point, the \emph{joint value function}.

\textbf{Theoretical Safety and Optimality Guarantees:} We prove forward invariance of the safe set under the joint policy, and show convergence to the optimal robust task performance within the safe set.

\textbf{Empirical Validation:} We conduct experiments using neural approximations on a set of continuous control tasks from the \texttt{Gymnasium} and \texttt{SafetyGymnasium} suites under adversarial disturbances, demonstrating stable training convergence and near-zero safety violations during evaluation. 

\vspace{-4mm}
\section{Preliminaries}\label{sec:preliminaries}\vspace{-2mm}
This section introduces key definitions, followed by preliminaries on Constrained Markov Decision Processes and Hamilton-Jacobi safety analysis. Refer to Appendix \ref{app:preliminaries_section} for more details.
\vspace{-2mm}
\subsection{Definitions and Notations}\label{subsec:defns_notation}\vspace{-1mm}
    We consider a discounted MDP with finite state, action and disturbance spaces: $\mathcal{S} \subset \mathbb{R}^n$, $\mathcal{A} \subset \mathbb{R}^m$, and $\mathcal{D} \subset \mathbb{R}^l$ respectively, with the transition kernel $P(\cdot|s,a,d) \in \Delta(\mathcal{S})$. A deterministic joint policy $\pi = (\mu, \nu)$ pairs the mappings $\mu: \mathcal{S} \rightarrow \mathcal{A}$ and $\nu: \mathcal{S} \times \mathcal{A} \rightarrow \mathcal{D}$, yielding $\pi(s) = (\mu(s), \nu(s, \mu(s)))$. The reward and the safety margin are given by $r: \mathcal{S} \times \mathcal{A} \times \mathcal{D} \rightarrow [0, R_{\text{max}}], R_{\text{max}} >0$ w.l.o.g, and $g: \mathcal{S} \rightarrow [-G_{\text{min}}, G_{\text{max}}]$, with $G_{\text{min}}, G_{\text{max}} > 0$ respectively. We subsume the extended class-$\mathcal{K}_e$ scaling function $\beta(\cdot)$ into $g(s)$ to align with task return scales (unless stated otherwise). The \textit{constraint set} is $\mathcal{F}^c := \{s \in \mathcal{S}: g(s)\ge0\}$, complementing the failure set $\mathcal{F}$. The \textit{viability kernel} $\Omega \subseteq \mathcal{F}^c$ is the subset of states from which a control policy exists that can keep the agent away from the failure set $\mathcal{F}$ for all times, regardless of the adversarial disturbance. Finally, 
    the value and action-value function spaces are defined as $\mathcal{V}:=\mathbb{R}^{|\mathcal{S}|}$ and $Q:=\mathbb{R}^{|\mathcal{S}||\mathcal{A}||\mathcal{D}|}$, equipped with the supremum norm $\Vert\cdot\Vert_{\infty}$.
\vspace{-3mm}
\subsection{Constrained Markov Decision Processes}\label{subsec:cmdps}\vspace{-1mm}
A Constrained Markov Decision Process (CMDP) \citep{altman1999constrained} extends an MDP by introducing auxiliary costs to encode safety constraints. It is defined by the tuple $\mathcal{M}=\left(\mathcal{S, A, P}, r, c, \gamma, \mu_0\right)$, where $r$ and $c$ are the task reward and constraint cost functions, respectively, $\gamma\in[0,1)$ is the discount factor, and $\mu_0:\mathcal{S}\rightarrow[0,1]$ denotes the initial state distribution. For a stochastic policy $\pi:\mathcal{S}\rightarrow\Delta(\mathcal{A})$, let $J_r(\pi)$ and $J_c(\pi)$ 
denote the expected discounted cumulative return and cost. The CMDP objective seeks an optimal policy that maximizes returns subject to a predefined safety threshold $d\in\mathbb{R}$, typically solved 
via Lagrangian relaxation, yielding an equivalent unconstrained min-max saddle-point problem: $\min_{\lambda\ge 0}\max_{\pi\in\Pi} \left( J_r(\pi) - \lambda(J_c(\pi)-d) \right)$,
where $\lambda\ge 0$ is the Lagrange multiplier and $\Pi$ is the set of stationary stochastic policies. More details can be found in \citep{altman1999constrained, achiam2017constrained, tesslerreward, stooke2020responsive}.

\subsection{Hamilton-Jacobi Safety Analysis}\label{subsec:hji_safety_analysis}
Hamilton-Jacobi (HJ) reachability analysis seeks to find the \textit{safe set} $\Omega\subseteq\mathcal{F}^c$ \citep{mitchell2005time}. Formulated as a \textit{zero-sum differential game}, the robust safety value function is $V^\phi(s) = \inf_{\mathcal{N}[\pmb{d}]}\sup_{\pmb{a}} \inf_{t \geq 0} g(\xi_{s_0}^{\pmb{a}, \mathcal{N}[\pmb{d}]}(t))$, where $\xi^{\pmb{a},\mathcal{N}[\pmb{d}]}_{s_0}(\cdot)$ denotes the trajectory starting at $s_0$ and evolving according to the control and disturbance trajectories ($\pmb{a}, \mathcal{N}[\pmb{d}]$), with $\mathcal{N}[\pmb{d}]$ the set of \textit{non-anticipative disturbance strategies} \citep{evans1984differential, fisac2017general}. By Dynamic Programming, $V(s)$ is the unique viscosity solution to the Hamilton-Jacobi-Isaacs (HJI) variational inequality.
As the corresponding infinite-horizon Bellman operator does not induce a contraction mapping, \citep{fisac2019bridging} introduces a discount factor $\gamma \in (0,1)$ to enable TD-learning, resulting in the time-discounted robust Bellman equation under stochastic dynamics $V(s) = (1-\gamma) g(s) + \gamma \min \left\{ g(s), \max_{a \in \mathcal{A}}\min_{d\in\mathcal{D}} \mathbb{E}_{s'\sim p(\cdot|s, a, d)}[V(s')]\right\}$
More details can be found in \citep{mitchell2005time,  fisac2019bridging, fisac2017general,evans1984differential}.

\vspace{-2mm}
\section{Learning Safety-Constrained Task Action-Value Function}\label{sec:defining_joint_action_value_function}
This section presents our unified formulation to learn a joint safety-constrained task value function. We show that the corresponding \emph{joint Bellman operator} induces a contraction mapping when the \textit{auxiliary} safety value function is held fixed. We further see that the coupled system possesses a discontinuity due to changing joint greedy policies, and discuss its resolution.
\subsection{The Joint Formulation}\label{subsec:joint_formulation}
For clarity, we first introduce the joint Bellman equation to intuit its formulation. In the presence of adversarial disturbances, the \textit{time-discounted} \textit{task action-value function} $Q^\psi$ purely corresponding to task return maximization can be extended to the following robust stochastic equation:

\par \vspace{-2\parskip} {\normalsize
\begin{equation}\label{eq:robust_det_bell_eqn}
    Q^{\text{task}}(s,a,d) =  r(s,a,d) + \gamma \mathop{\mathbb{E}}_{s' \sim p(\cdot | s,a,d)}\left[ \max_{a' \in \mathcal{A}} \min_{d' \in \mathcal{D}} \ Q^{\text{task}}(s', a', d')\right]
\end{equation}}
Next, consider the \textit{time-discounted} safety value function, as introduced in \citep{fisac2019bridging}. We'd like to use this value function to instead assess the safety values of some given policy $\pi:=(\mu, \nu)$, yielding the following \textit{safe action-value function} evaluated $\pi$:
\par \vspace{-3\parskip} {\normalsize\begin{equation}\label{eq:safe_action_value}
    Q_\pi^\phi(s,a,d) = (1-\gamma)g(s) + \gamma \min\left\{ g(s), \mathop{\mathbb{E}}_{s' \sim p(\cdot | s,a,d)} \left[\max_{a' \in \mu(s')} \min_{d' \in \nu(s',a')} Q_\pi^\phi (s' , a', d') \right]\right\}
\end{equation}}
We can use $Q_\pi^\phi$ to cap the further accrual of task rewards, i.e., \textit{truncate task accrual} for the policy $\pi$ when the corresponding task action-value goes above $Q_\pi^\phi$, enabling the agent to instead optimize \textit{task performance subject to safety}. Since $\gamma$ serves as the \textit{probability of episode continuation}, by the notion of \textit{principled time-discounting} \cite{fisac2019bridging}, we can combine Eq.~\ref{eq:safe_action_value} with Eq.~\ref{eq:robust_det_bell_eqn}, yielding the stochastic time-discounted \textit{joint action-value function} as follows:
\begin{equation}\label{eq:joint_st_value_eqn}
\begin{aligned}
    Q(s,a,d) =& (1-\gamma) \min \{ r(s,a,d), g(s) \} \\
    &+ \gamma \min \left\{ r(s,a,d) + \mathop{\mathbb{E}}_{s' \sim p(\cdot | s,a,d)} \left[\max_{a' \in \mathcal{A}} \min_{d' \in \mathcal{D}} Q(s', a', d')\right], Q_\pi^\phi(s,a,d) \right\}
\end{aligned}
\end{equation}
where, with probability $1-\gamma$, the agent receives the reward $\min\{r(s, a,d), g(s)\}$, or continues with probability $\gamma$. The corresponding \textit{maximin joint policy} $\pi:=(\mu, \nu)$ is defined through a \textit{measurable selection} 
as:
\begin{align}\label{eq:greedy_mu_nu}
    \mu(s):=\argmax_{a\in\mathcal{A}} \min_{d\in\mathcal{D}} Q(s, a, d), \qquad \nu(s, \mu(s)):=\argmin_{d\in\mathcal{D}}Q(s, \mu(s), d)
\end{align}
 Here, the disturbance gets the advantage of observing the controller's action, while the controller must \textit{optimize} for the adversarial disturbance. 
Note that we do not assume $Q^\phi_\pi$ to be given; instead, 
$Q_\pi^\phi$ is learned alongside $Q$, estimating the safety values of the joint policy $\pi$. Note the \textit{coupling}: the joint policy $\pi$ is derived from the joint value $Q$, which is evaluated by $Q^\phi_\pi$, whose values in turn update $Q$. Next, we show that the safety and joint Bellman operators are contractions given a fixed policy and fixed safety value, respectively. The joint value function is \textit{conceptualized} in Fig. \ref{fig:fig_joint_and_region}.

\subsection{Safety Action-Value of the Joint Policy}\label{subsec:contraction_safety_action_value}
Given a joint policy $\pi$, the safety Bellman operator $\mathcal{T}^\phi_\pi: \mathcal{Q} \rightarrow \mathcal{Q}$ corresponding to Eq. \ref{eq:safe_action_value} is defined as:
\par \vspace{-4\parskip} {\normalsize\begin{equation}\label{eq:safety_bellman}
    \mathcal{T}^\phi_\pi\left[Q^\phi\right](s,a,d) = (1-\gamma)g(s) + \gamma \min\left\{ g(s), \mathop{\mathbb{E}}_{s' \sim p(\cdot | s,a,d)} \left[\max_{a' \in \mu(s')} \min_{d' \in \nu(s',a')} Q^\phi(s' , a', d') \right]\right\}
\end{equation}}
Note that $\nu(\cdot, a')$ refers to the specific sample $a'$ drawn from $\mu(\cdot)$ by the $\max_{a'\in\mu}$ operation.
\begin{restatable}[Contraction Mapping of $\mathcal{T}^\phi_\pi$]{proposition}{safetybellmancontraction}
\label{prop:safety_bellman_contraction}
For any arbitrary fixed joint policy $\pi$, the time-discounted Bellman operator defined in Eq.~\ref{eq:safety_bellman} induces a $\gamma$-contraction mapping under the supremum norm. That is, for a given discount factor $\gamma \in [0,1)$, and any $Q^\phi_\pi, \widetilde{Q}^\phi_\pi \in \mathcal{Q}$, we have $\norm{\mathcal{T}^\phi_\pi\left[Q^\phi_\pi\right] - \mathcal{T}^\phi_\pi\left[\widetilde{Q}^\phi_\pi\right]}_\infty \leq \gamma \norm{Q^\phi_\pi - \widetilde{Q}^\phi_\pi}_\infty$. (Proof in Appendix \ref{app:contraction_safe_action_value})
\end{restatable}

Hence, there exists a unique fixed point, $Q^\phi_\pi$, which we call the safety action-value of the fixed joint policy $\pi$. The corresponding safety value function is obtained as $V^\phi_\pi (s) = \max_{a\in\mu} \min_{d\in\nu} Q^\phi_\pi (s,a,d)$. The set $\Omega^\pi = \{ s \in \mathcal{S}: V^\phi_\pi (s) \geq 0\}$ is then the \textit{safe set of policy} $\pi$ (under $\gamma\rightarrow 1$ as discussed in \cite{fisac2019bridging} i.e., the set of all states from which the joint policy $\pi$ keeps the agent safe. Note that, in general, $\Omega^\pi \subseteq \Omega$, where the maximal robust safe set, $\Omega$, (viability kernel) is given by $\Omega = \{ s \in \mathcal{S}: \exists \; \mu \text{ such that } \forall \nu, \; V^\phi_{(\mu,\nu)} (s) \geq 0\}$.

\subsection{Joint Action-Value Function}\label{subsec:contraction_joint_action_value}
Next, for any fixed safety action-value function, $Y \in \mathcal{Q}$, the joint Bellman operator $\mathcal{T}_Y: \mathcal{Q} \rightarrow\mathcal{Q}$ corresponding to Eq. \ref{eq:joint_st_value_eqn} is given by:
\par \vspace{-3\parskip} {\normalsize
\begin{equation}\label{eq:joint_bellman}
\begin{aligned}
\mathcal{T}_Y[Q](s,a,d) = (1-\gamma) m(s,a,d)+ \gamma \min \Bigl\{ r(s,a,d) +\mathop{\mathbb{E}}_{s' \sim p(\cdot | s,a,d)} \left[V_Q(s')\right], Y(s,a,d) \Bigr\}
\end{aligned}
\end{equation}}
where $m(s,a,d) = \min \{ r(s,a,d), g(s)\}$ and $V_Q(s) = \max_{a \in \mathcal{A}} \min_{d \in \mathcal{D}} Q(s,a,d) $.
\begin{restatable}[Contraction Mapping of $\mathcal{T}_Y$]{proposition}{jointbellmancontraction}
\label{prop:joint_bellman_contraction}
For any arbitrary, fixed $Y \in \mathcal{Q}$, the Bellman operator defined in Eq.~\ref{eq:joint_bellman} induces a $\gamma$-contraction mapping under the supremum norm. That is, for a given discount factor $\gamma \in [0,1)$, and any $Q, \widetilde{Q} \in \mathcal{Q}$, we have $\norm{\mathcal{T}_Y[Q] - \mathcal{T}_Y[\widetilde{Q}]}_\infty \leq \gamma \norm{Q - \widetilde{Q}}_\infty$. (Proof in Appendix \ref{app:contraction_joint_action_value})
\end{restatable}

Thus, for a fixed safety action-value function $Y$, there exists a unique fixed point $Q_Y$ of the joint Bellman operator.

\begin{figure}
    \centering
    \includegraphics[width=0.42\linewidth, trim={1cm 1cm 1cm 1.0cm}, clip]{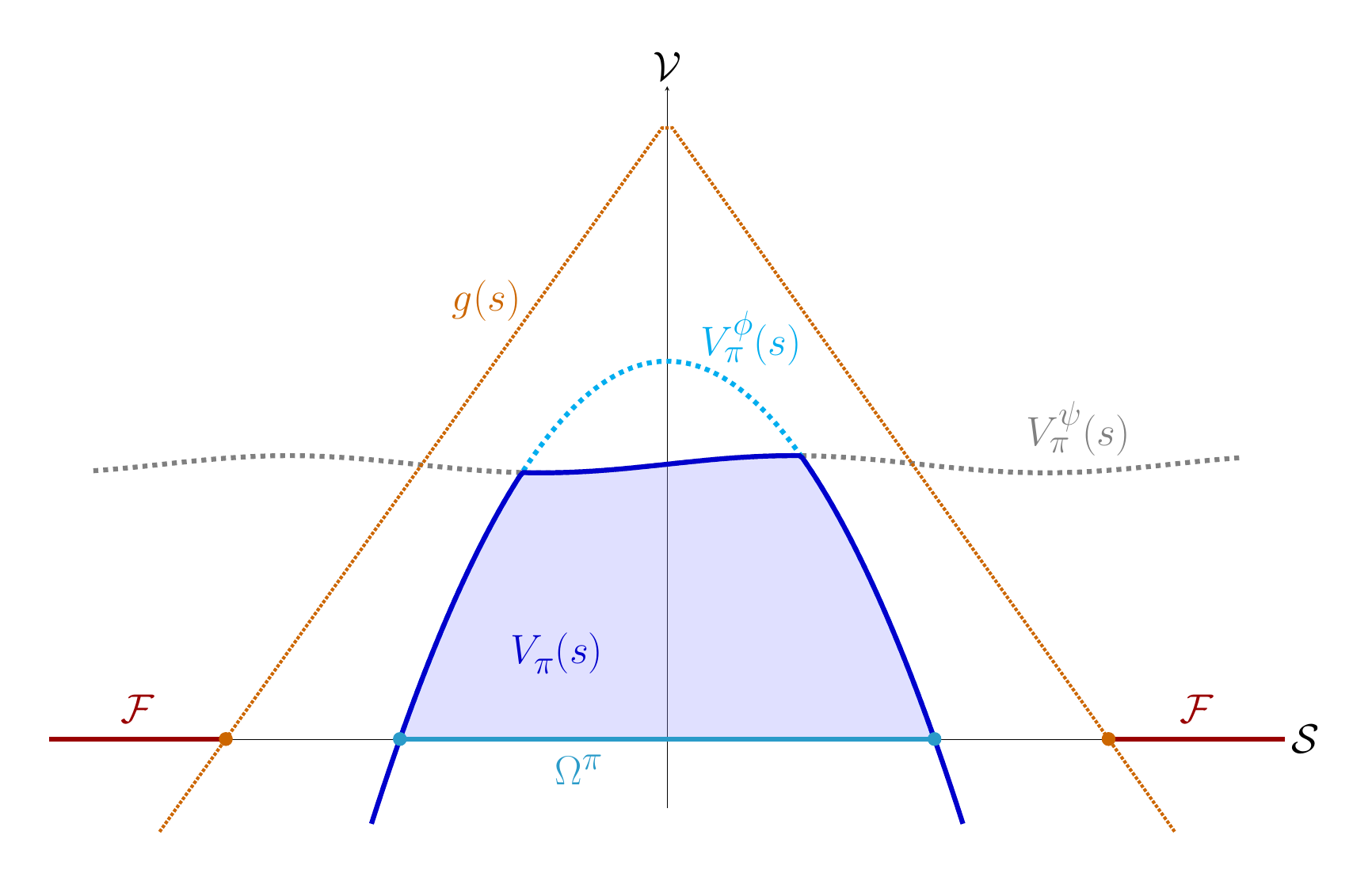}\hspace{0.6cm}
    \includegraphics[width=0.32\linewidth, trim={0cm 2.0cm 2.0cm 2.0cm}, clip]{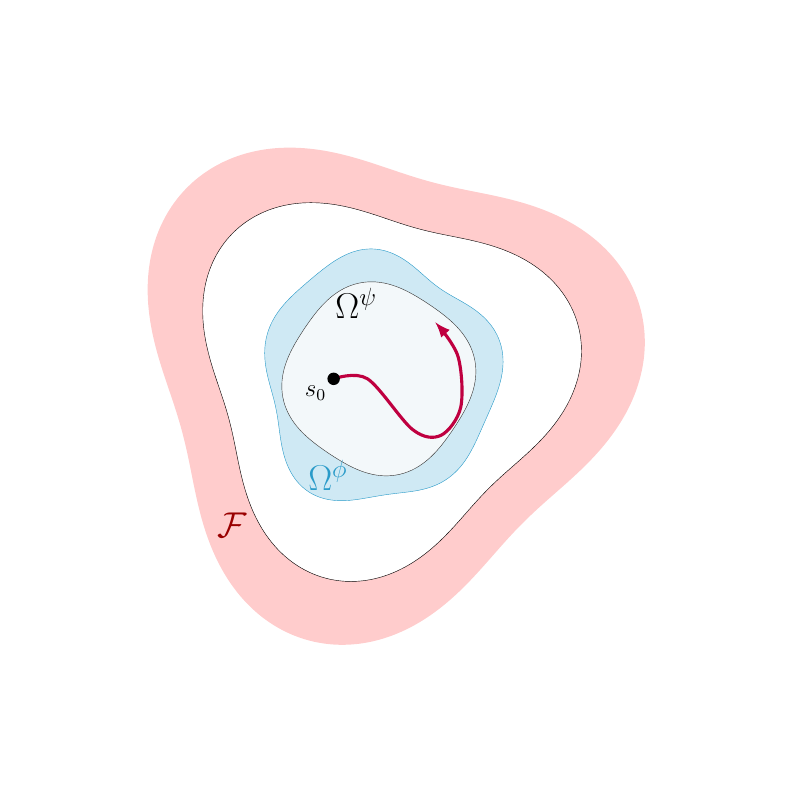}
    \caption{(Left) The Joint Value function is \textit{conceptualized} as a 2D figure, with the state-space indicated on the $x$-axis, and joint values on the $y$-axis. (Right) The corresponding regions within the safe set $\Omega^\pi$ are depicted. The locus of points where $V^\phi_\pi = V^\psi_\pi$ indicates the boundary between $\Omega^\phi$ and $\Omega^\psi$ as defined in Definition \ref{defn:opt_regions}. Trajectories must stay within $\Omega^\psi$ to maximize episodic return.}
    \label{fig:fig_joint_and_region}\vspace{-5mm}
\end{figure}

\subsection{Objective Function of Joint Action-Value Function}\label{subsec:objective_function}
Given the value functions $(Q, Y)$, we derive the \emph{joint objective function} from robust optimal stopping theory to provide a principled interpretation of the joint value function. To this end, we first define the task-optimizing and safety-maximizing regions, as illustrated in Fig. \ref{fig:fig_joint_and_region}. For a given joint action-value function $q$, safety critic $y$ and the joint policy $\pi=(\mu, \nu)$, let $V^\psi_q(s) := r(s, \mu(s), \nu(s, \mu(s))) + \mathbb{E}[V_q(s')]$, where $V_q(s) = \max_a \min_d q(s,a,d)$. Define the safety value function $V^\phi_\pi(s) := Q^\phi_\pi (s, \mu(s), \nu(s, \mu(s))$.


\begin{restatable}[Task and safety optimizing regions]{definition}{optregions}
\label{defn:opt_regions}
    The region where the task value function is strictly lower than the safe value function is defined as the task-optimizing region: $\Omega^\psi = \left\{ s \in \mathcal{S} \; \middle\vert \; V^\psi_q(s) < V^\phi_\pi(s) \right\}$. The region where the task value function is greater than the safe value function is defined as the safety-optimizing region: $\Omega^\phi = \left\{ s \in \mathcal{S} \; \middle\vert \; V^\psi_q(s) \geq V^\phi_\pi(s) \right\}$.
\end{restatable}

\begin{restatable}[Objective Function]{theorem}{objectivefunc}
\label{defn:objective_fun}
Given a Borel-measurable safety maximizing region, $\Omega^\phi \in \mathcal{B}(\mathcal{S})$, define the first switching time $\tau : \mathcal{W} \rightarrow \{ 0, \cdots, T\}$ as $\tau = \inf \{k \in \{0, \cdots, T  \} : s_k \in \Omega^\phi\}$, 
where $T$ is the time horizon and $\inf \emptyset = T$ ($T\rightarrow \infty$ for infinite horizon setting). Let $\Lambda_k$ be the set of stopping times $\tau$ bounded by time horizon $k$, i.e., $0 \leq \tau \leq k$.
Let $\omega$ be trajectories sampled using the joint policy as $a_t,d_t \sim \pi(s_t)$, with the dynamics $s_{t+1} \sim p(\cdot| s_t, a_t, d_t)$. Then, the corresponding joint objective function that the joint policy maximizes is given by: (Proof in Appendix \ref{app:obj_function_derivation})
\begin{equation}\label{eq:general_obj_func}\hspace{-0.4cm}
    \resizebox{0.97\textwidth}{!}{
        $J(\pi) := \inf_{\tau \in \Lambda_T} \mathop{\mathbb{E}}_{\omega \sim \pi} \left[\sum_{t=0}^{\tau -1} \gamma^t r(s_t, a_t, d_t) + \gamma^\tau \left( \underbrace{ (1-\gamma) \min \{r(s_\tau, a_\tau, d_\tau),g(s_\tau)\} +\gamma V^\phi_Y(s_\tau)}_{:=J_\tau} \right) \middle\vert s_0 = s \right]$
    }
\end{equation}
\end{restatable}\vspace{-1mm}

Intuitively, the stopping time $\tau \in \Lambda_T$ serves as a \textit{safety inspector}: once the trajectory leaves $\Omega^\psi$, the further accrual of expected return is replaced by the \textit{terminal safety value} $J_\tau$ (dominated by $V_Y^\phi$ as $\gamma\rightarrow1$). Hence, task performance can be optimized as long as the trajectory remains inside $\Omega^\psi$. We next analyze the coupled evolution of the safety and joint value functions.

\subsection{Coupled Fixed Point}\label{subsec:greedy_maximin}
Let $\pi_Q = (\mu_Q, \nu_Q)$ be the maximin policy derived from $Q$, as in Eq. \ref{eq:greedy_mu_nu}. Since the map $Q \mapsto \pi_Q$ is generally discontinuous, small perturbations of $Q$ can change the argmax/argmin completely. The optimal solution we seek is a pair $(Q^*, Y^*)$ satisfying $Y^* = Q^\phi_{\pi_{Q^*}} , \quad Q^* = \mathcal{T}_{Y^*} Q^*$.
This is a coupled system: the safety critic $Y^*$ equals the safety value of the greedy policy of $Q^*$, and $Q^*$ is the fixed point of its own Bellman operator, ``\textit{truncated}'' by $Y^*$. 
Thus, in the iteration 
$Q_{k+1} = \mathcal{T}_{Q^\phi_{\pi_{Q_k}}}(Q_k)$, the composition $Q \mapsto \mathcal{T}_{Q^\phi_{\pi_Q}}(Q)$ is not guaranteed to induce a contraction by our current analysis, as $Q^\phi_{\pi_Q}$ can change discontinuously due the discontinuous mapping $Q \mapsto \pi_Q$. 
This motivates a two-timescale stochastic approximation learning framework in Section \ref{sec:two_timescale_approx}, where the joint and safety value functions are learned on separate timescales. 

\section{Two-Timescale Stochastic Approximation}\label{sec:two_timescale_approx}
In this section, we analyze the coupled learning dynamics in finite/tabular settings. 
For clarity, we present a synchronous update analysis here, under which every state-action-disturbance coordinate is updated at each iteration; extension to asynchronous settings is discussed in Appendix \ref{app:asynchronous_ext}. Let $\hat{\mathcal{T}}^{\phi}_{n+1, \pi_{Q_n}}$ and $\hat{\mathcal{T}}_{n+1, Y}$ denote the 
stochastic estimators of the Bellman targets\footnote{Note that Assumption \ref{ass:martingale_noise} concerns \textit{unbiased estimation} of the Bellman target itself.}, with $\pi_{Q_n}\in\Pi(Q_n)$, where for $q\in\cQ$, $\Pi(q)$ denotes the finite correspondence of the maximin policies:
\begin{align}
    \Pi(q):=\Bigg\{ (\mu, \nu):\mu(s)\in\argmax_a\min_d q(s, a, d), \,\,\, \nu(s, \mu(s)) \in \argmin_d q(s, \mu(s), d) \Bigg\} \label{eq:ties_greedy_policies}
\end{align} 

\begin{restatable}[Conditional unbiasedness and Martingale-difference noise \citep{tsitsiklis1996analysis}]{assumption}{asspmartingalediffnoise}
\label{ass:martingale_noise} Let $\{\mathcal{H}_n\}_{n\ge 0}$ denote the filtration generated by the learning history up to iteration $n$. Then,
  \begin{align}
    \mathbb{E}\left[ \hat{\mathcal{T}}^{\phi}_{n+1, \pi_{Q_n}}[Y_n] \middle\vert \mathcal{H}_n \right] = \mathcal{T}^{\phi}_{\pi_{Q_n}}[Y_n], \quad \mathbb{E}\left[ \hat{\mathcal{T}}_{n+1, Y_n}[Q_n] \middle\vert \mathcal{H}_n \right] = \mathcal{T}_{Y_n}[Q_n]
\end{align}
\end{restatable}
\begin{restatable}[Step‑size conditions ~\citep{robbins_monro}]{assumption}{asspsteps}
\label{ass:steps}
  The step-sizes, $\alpha_n,\beta_n \in (0,1]$, satisfy Robbins-Monro conditions: 
  $\sum_{n=0}^\infty \alpha_n = \sum_{n=0}^\infty \beta_n = \infty$,
  $\sum_{n=0}^\infty (\alpha_n^2+\beta_n^2) < \infty$, $\lim_{n \rightarrow \infty} \alpha_n \rightarrow 0$, $\lim_{n \rightarrow \infty} \beta_n \rightarrow 0$, and additionally,
  $\lim_{n\to\infty} \alpha_n/\beta_n = 0$.
\end{restatable}

Additionally, the total martingale noise is assumed to have uniformly bounded conditional variance under bounded iterates. Next, under Temporal difference learning, the bootstrapped updates with step‑size sequences $\{\alpha_n\},\{\beta_n\}$ are given by:
\begin{equation}
\resizebox{0.92\textwidth}{!}{$
  Y_{n+1} = Y_n + \beta_n \left[ \hat{\mathcal{T}}^{\phi}_{n+1, \pi_{Q_n}}[Y_n] - Y_n + \xi^{\phi}_{n+1}\right], \quad 
  Q_{n+1} = Q_n + \alpha_n \left[ \hat{\mathcal{T}}_{n+1, Y_n}[Q_n] - Q_n + \xi_{n+1} \right]$}
\end{equation}
where $\xi^\phi_{n+1}$ and $\xi_{n+1}$ are martingale difference perturbations~\citep{borkar2009stochastic}, with $\mathbb{E}[\xi^{\phi}_{n+1}|\mathcal{H}_n] = \mathbb{E}[\xi_{n+1}|\mathcal{H}_n] = 0$. As the map $Q\mapsto Y_{\pi_{Q}}$ need not be continuous or Lipschitz, the standard two-timescale stochastic approximation need not apply. Thus, we analyze the fast timescale recursion as a stochastic recursive inclusion in Section \ref{subsec:fast_timescale} using the \textit{set-valued two-timescale framework} ~\citep{10.1287/moor.2019.1037, borkar2023stochastic}.

\subsection{Fast Timescale: Safety Critic Tracking}\label{subsec:fast_timescale}

Under Assumption \ref{ass:steps}, even with a \textit{quasi-static} $Q_n$, at the \textit{switching time}, two or more policies 
can have ties in Eq. \ref{eq:ties_greedy_policies}, but can have different safety-critic evaluations. This motivates a \textit{set-valued} fast tracking analysis. To characterize the asymptotic behavior of the discrete TD recursion, let $h_{\pi}(y):= \mathcal{T}^{\phi}_{\pi}[y] - y$ denote the limiting dynamics induced by Assumption \ref{ass:steps} for each $\pi\in\Pi(q)$. Thus, for a given $q\in\mathcal{Q}$, the limiting differential inclusion is given by:
\begin{align}
    \dot{y}(t) \in \Psi(q, y(t)), \quad \text{with }\, \Psi(q, y):= \overline{\text{conv}}\{ h_{\pi}(y) : \pi\in\Pi(q) \} \label{eq:set_valued_fast_tracking}
\end{align}
where $\overline{\text{conv}}\{\cdot\}$ denotes a closed convex hull of a set. Next, let $ \Lambda(q) := \{ y(0) : y:\mathbb{R}\rightarrow\mathcal{Y} \text{ is a bounded complete solution of Eq. \ref{eq:set_valued_fast_tracking}} \}$
\begin{restatable}[Set-valued fast timescale tracking]{lemma}{safetycritictracking}
\label{lem:safety_critic_tracking}
    Under Assumptions \ref{ass:martingale_noise}-\ref{ass:steps}, suppose the iterates are almost surely bounded. 
    Then, 
    \begin{enumerate}\vspace{-0.35cm}
        \item for every quasi-static $q$, Eq. \ref{eq:set_valued_fast_tracking} admits a nonempty compact global attracting set $\Lambda(q)$;\vspace{-0.25cm}
        \item $q\mapsto\Lambda(q)$ is upper semicontinuous; and \vspace{-0.25cm}
        \item the coupled iterates satisfy $\mathrm{dist}\left( \left(Y_n, Q_n\right), \mathrm{Gr}\,\Lambda \right) \rightarrow 0$ almost surely,
        where $\mathrm{Gr}\,\Lambda := \{ (y, q): y\in\Lambda(q)\}$ is the graph of $\Lambda$. (Proof in Appendix \ref{app:fast_timescale_proofs})
    \end{enumerate}
\end{restatable}
 Note that when there are no ties, $\Pi(q) = \{\pi_q\}$, and Eq. \ref{eq:set_valued_fast_tracking} reduces to an ODE, and $\Lambda(q) = \{ Y_{\pi_q} \}$, recovering the usual single-policy fast timescale behavior locally. 


\vspace{-1mm}
\subsection{Slow Timescale: Averaged Differential Inclusion}\label{subsec:slow_timescale}
Section \ref{subsec:fast_timescale} shows that the coupled iterates asymptotically approach the graph of the compact global attractor $\Lambda(q)$ for some quasi-static $q$. In this section, we show that the slow recursion is governed by a differential inclusion obtained by averaging over admissible occupation measures associated with the equilibrated fast time-scale dynamics. For $y\in\mathcal{Y}$ and $q\in\mathcal{Q}$, define the \textit{slow-timescale Bellman residual map} to be $G(y, q):=\mathcal{T}_{y}[q]-q$. Next, following \citep{10.1287/moor.2019.1037}, let $\mathfrak{D}(q)$ denote the set of admissible limiting occupation measures describing the \textit{equilibrated} fast time-scale dynamics at some $q$; in particular, every $\mathbb{P}\in\mathfrak{D}(q)$ is supported on $\Lambda(q)$. The corresponding \textit{averaged} slow-timescale residual is given by:
\begin{align}
    \overline{G}(q) := \left\{ \int_{\Lambda(q)} G(y, q) \mathbb{P}(dy) \ : \ \mathbb{P}  \in \mathfrak{D}(q) \right\} \label{eq:G_hat_q}
\end{align}
Note that when $\Lambda(q)=\{ y_{\pi_q} \}$, $\mathfrak{D}(q) = \{ \delta_{y_{\pi_q}} \}$, and thus, $\overline{G}(q) = \{ G(y_{\pi_q},q) \}$. The integral in Eq. \ref{eq:G_hat_q} captures the \textit{long-run average effect} of the equilibrated fast timescale behavior in $y$ on the slow-timescale variable $q$.

\begin{restatable}[Slow-timescale differential-inclusion limit]{theorem}{ode}\label{thm:ode}

    Let $\overline{Q}(\cdot)$ denote the continuous affine interpolation of $\{Q_n\}$ on the slow-timescale. Under Assumptions \ref{ass:martingale_noise}-\ref{ass:steps}, Lemma \ref{lem:safety_critic_tracking}, and almost surely bounded iterates $\{(Y_n, Q_n)\}$, $\overline{Q}(\cdot)$ is almost surely an asymptotic pseudotrajectory of the averaged differential inclusion $\dot{q}(t) \in \overline{G}(q(t))$.
    Consequently, its $\omega$-limit set $\mathcal{U}:=\omega(\overline{Q})$ is a nonempty compact invariant internally chain-transitive set of the limiting differential-inclusion dynamics. (Proof in Appendix \ref{app:slow_timescale_proofs})
    
\end{restatable}

\section{Safe Configuration Set and Forward Invariance}\label{sec:forward_invariance}
Having established the limiting dynamics of the discrete TD learning iterates, 
we now show that these dynamics preserve the safety certificate: the set of joint action-value functions whose induced policy is everywhere safe is forward invariant, so a configuration that certifies safety cannot evolve into one that does not.
\begin{restatable}[Kernel-preserving actions]{definition}
    {kernelpreservingactions}\label{defn:kernel_preserving_actions}
    An action is kernel-preserving at state $s$, if it keeps the next state inside the viability kernel against every disturbance $\mathcal{A}^\Omega(s) := \{ a \in \mathcal{A}: P (s' \in \Omega|s,a,d) = 1 \quad \forall \; d \in \mathcal{D} \}$, where $\kappa, \epsilon_0 > 0$ without loss of generality.
\end{restatable}
\begin{restatable}[Index Sets and Relaxed Constraint Sets]{definition}
    {indexsetenlargements}\label{defn:index_sets_enlargements}
    For each state $s \in \Omega$ and each action $a \notin \mathcal{A}^\Omega(s)$, fix a disturbance $\hat{d}(s,a) \in \mathcal{D}$ with $P(s' \notin \Omega|s,a,\hat{d}(s,a))>0$. Then the index sets are defined as follows: $\mathcal{I}:= \{ (s,a,d): s\in \Omega, a \in \mathcal{A}^\Omega(s), d \in \mathcal{D} \}$ and $\mathcal{J} :=\{ (s,a,\hat{d}(s,a)): s\in \Omega, a \notin \mathcal{A}^\Omega(s) \}$.
 Let $\mathcal{C}_\epsilon := \{ q \in \mathcal{Q}: q(x) \geq - \epsilon \; \forall \; x \in \mathcal{I}, q(x) \leq -\kappa + \epsilon \; \forall \; x \in \mathcal{J} \}$ denote the relaxed constraint set, for $\epsilon\ge0$ i.e., $\mathcal{C}_\epsilon$ relaxes each defining inequality of $\mathcal{C}_0$ by $\epsilon$.
\end{restatable}
\begin{restatable}[Safe Configuration Set]{definition}{safecfg}\label{defn:q_cfg}
    The safe configuration set is defined as the set of joint $Q$-functions whose induced policy is 
    kernel-preserving:
    \par \vspace{-2\parskip} {\normalsize\begin{equation}
        \mathcal{Q}_{\text{cfg}} := \mathcal{C}_0 = \{ q \in \mathcal{Q}: q(x) \geq 0 \; \forall\; x \in \mathcal{I}, q(x) \leq -\kappa \; \forall \; x \in \mathcal{J} \}
    \end{equation}}
\end{restatable}
Next, let $\rho(q):= \text{dist}(q, \mathcal{Q}_\text{cfg})$ denote the distance to the set $\mathcal{Q}_\mathrm{cfg}$. Details on the geometry of the set is provided in Appendix \ref{subsubsec:geometry_safe_set}. 
Throughout, we take $\epsilon_0 \in (0, \kappa/2)$, which ensures that the two families of constraints remain strictly separated on all of $\mathcal{C}_{\epsilon_0}$. As the coupled iterates approach the graph of $\Lambda(\cdot)$, safety certification (both positive and negative) must hold throughout $\Lambda(q)$. 
\begin{restatable}[Negative Certification of the fast attractor]{assumption}{attractorcertification}\label{ass:attractor_certication}
    $\exists \;\sigma_\Lambda >0$, such that, at unit margin scale, $y(x) \;\le\; -\sigma_\Lambda \; \forall\; x \in \mathcal{J}, y \in \Lambda_1(q), q \in \mathcal{C}_{\epsilon_0}.$
\end{restatable}
The corresponding condition on $\mathcal{I}$ is preserved in the fast inclusion, and hence, the certificate holds throughout $\Lambda(q)$. Note that $\Pi(q)$ is greedy with respect to $q$ and is hence independent of margin scale, and the safety Bellman operator of Eq.~\ref{eq:safety_bellman} is positively homogeneous of degree one in the margin, i.e., $\Lambda_K(q) = K \Lambda_1(q)$. Let the Assumption~\ref{ass:attractor_certication} hold and let $\Theta := ((1-\gamma)R_{\text{max}} + \kappa)/\gamma$. Then, for every $K \ge \underline{K} := \frac{(1-\gamma) R_{\text{max}} + \kappa}{\gamma \sigma_\Lambda}$, 
one has $y(x) \le -\Theta$ for all $x \in \mathcal{J}$, all $y \in \Lambda(q)$, and all $q \in \mathcal{C}_{\epsilon_0}$. If $\mathcal{J} = \emptyset$, the bound on $K$ and the conclusion hold trivially. 
(Proof in Appendix \ref{app:forward_invariance})

\begin{restatable}[Forward invariance and exponential attraction]{theorem}{forwardinvariance}\label{thm:forward_invariance}
    Let Assumption~\ref{ass:attractor_certication} hold and let $K \ge \underline{K}$. If $\rho(q(0)) < \epsilon_0$, then every solution of $\dot{q}(t)\in\overline{G}(q(t))$ satisfies $\rho(q(t)) \;\le\; \rho(q(0))\, e^{-(1-\gamma)t} \quad \forall \; t \ge 0$.
    In particular, if $q(0) \in \mathcal{Q}_{\text{cfg}}$ then $q(t) \in \mathcal{Q}_{\text{cfg}} \; \forall \;t \ge 0$, and $\mathcal{Q}_{\text{cfg}}$ is exponentially attracting with basin containing the open neighbourhood $\{ q : \rho(q) < \epsilon_0 \}$. (Proof in Appendix \ref{app:forward_invariance})
\end{restatable}
We next check that membership in $\mathcal{Q}_\text{cfg}$ certifies safety of the induced policy. 
\begin{restatable}[Safety of the induced policy]{corollary}{inducedsafety}\label{cor:induced_policy_safety}
    Let $q(0) \in \mathcal{Q}_\text{cfg}$. Then at every $t \ge 0$ the induced policy $\pi_{q(t)}$ satisfies $\inf_{s \in \Omega} V^\phi_{\pi_{q(t)}} (s) \geq 0$ from any $s_0 \in \Omega$, the state trajectory generated by $\mu_{q(t)}$ remains in $\Omega$ almost surely, for all time and against every disturbance sequence. (Proof in Appendix \ref{app:forward_invariance})
\end{restatable}

\section{Convergence to Optimal Safety-Constrained Task Performance}\label{sec:convergence}
Forward invariance ensures that the agent eventually learns a policy that is guaranteed to be safe. In this section, we show that the system not only remains safe but also maximizes task performance in the safe set.

\begin{restatable}[Optimal safety-constrained task value function]{definition}{optjointvaluefunc}\label{defn:opt_joint_value_func}
    Let $\mathcal{Q}_\Omega := \mathbb{R}^{\mathcal{I}} = \{ \mathfrak{q}:\mathcal{I}\rightarrow \mathbb{R} \}$ denote the restricted action-value space on the index set $\mathcal{I}$. Define the robust safety-constrained task Bellman operator $\mathcal{T}_\Omega: \mathcal{Q}_\Omega \rightarrow \mathcal{Q}_\Omega$:
    \begin{align}
        \mathcal{T}_{\Omega}[\mathfrak{q}](s, a, d) = r(s, a, d) + \gamma\mathbb{E}\left[ \max_{a'\in A^{\Omega}(s')}\min_{d'\in\mathcal{D}} \mathfrak{q}(s', a', d') \right], \quad (s, a, d)\in\mathcal{I} \label{eq:safety_constr_task_bellman}
    \end{align}
    Note that since $A^{\Omega}(s)\neq \emptyset$ for $s\in\Omega$, by contraction, $\mathcal{T}_{\Omega}[\cdot]$ has a unique fixed point $Q^*_{\Omega} = \mathcal{T}_{\Omega} Q^*_{\Omega}$, and thus, the optimal safety-constrained task value function is given by $V^*_{\Omega}(s) = \max_{a\in A^{\Omega}(s)}\min_{d\in\mathcal{D}} Q^*_{\Omega}(s, a, d)$,
    which is the maximum expected return that a safe policy can accrue by keeping the agent in $\Omega$. Moreover, let $\mathcal{Q}^* := \left\{ q\in\mathcal{Q}_{\mathrm{cfg}} : q\vert_{\mathcal{I}} = Q^*_{\Omega} \right\}$.
    Thus, every $q\in\mathcal{Q}^*\subseteq \mathcal{Q}$ corresponds to a policy that is safe and attains the maximal possible return among all safe policies.
\end{restatable}
Let $Q^{\phi, 1}_{\pi}$ denote the safety critic obtained by scaling the margin $g(\cdot)$ by $K=1$. By homogeneity, $Q^{\phi}_{\pi}:=Q^{\phi, K}_{\pi}=KQ^{\phi, 1}_{\pi}$. Next, let $\delta_{\Omega} := \inf_{q\in\mathcal{Q}_{\mathrm{cfg}}}\min_{\pi\in\Pi(q)}\min_{s\in\Omega} V^{\phi, 1}_{\pi} > 0$ represent the margin of the safe policy that comes closest to $\partial\Omega$.

\begin{restatable}[Convergence to optimal safety-constrained task performance]{theorem}{convergencefixedpt}\label{thm:convergence_fixed_pt}
Let the sufficient margin scale $K^*$ be defined as $K^* := \frac{R_{\mathrm{max}} + \overline{V}}{\delta_{\Omega}} = \frac{(2-\gamma)R_\mathrm{max}}{(1-\gamma)\delta_\Omega}$, with $\overline{V}:=\frac{R_\mathrm{max}}{1-\gamma}$. Let $\mathcal{U}\subset\{q:\rho(q)<\epsilon_0\}$. Then, 
under Theorems \ref{thm:ode}-\ref{thm:forward_invariance}, almost surely bounded iterates $(Y_n, Q_n)$ and $K\ge K^*\ge\underline{K}$, $\delta_{\Omega} > 0$, every internally chain-transitive set of the averaged slow-timescale differential inclusion contained in $\mathcal{Q}_{\mathrm{cfg}}$ is a subset of $\mathcal{Q}^*$: $\mathcal{U}\subseteq \mathcal{Q}^* := \left\{ q\in\mathcal{Q}_{\mathrm{cfg}}: q\vert_{\mathcal{I}}=Q^*_{\Omega} \right\}$, and thus, $\mathrm{dist}\left( Q_n, \mathcal{Q}^* \right) \rightarrow 0$ almost surely.
Moreover, on $\mathcal{U}$ the averaged slow-timescale residual on kernel-preserving coordinates $\mathcal{I}$ is given by $\left.\overline{G}(q)\middle\vert_{\mathcal{I}}\right. = \left\{ \mathcal{T}_{\Omega}[q\vert_{\mathcal{I}}] - q\vert_{\mathcal{I}} \right\}$.
Thus, every corresponding limiting controller attains the optimal task value $V^{*}_{\Omega}$ on $\Omega$ while selecting only kernel-preserving actions. (Proof in Appendix \ref{app:convergence_proofs}).
\end{restatable}

The sufficient margin scale $K^*$ 
ensures that the safety truncation is never active when $\delta_\Omega > 0$.
Consequently, if $\delta_\Omega \rightarrow 0$, 
then $K^* \rightarrow \infty$; if $\delta_\Omega = 0$, Theorem \ref{thm:convergence_fixed_pt} no longer guarantees convergence to the optimal safety-constrained task value for any finite $K$, motivating a finite-$K$ characterization.   
Let $\Gamma_{\Omega}$ be the set of kernel-preserving controller policies, and for $\zeta\ge 0$, define $V^{\ast}_{\Omega, \zeta}(s) := \sup_{\mu\in\Gamma_{\Omega, \zeta}} \inf_\nu V^{\psi}_{\mu, \nu}(s)$ and $\Gamma_{\Omega, \zeta} : = \left\{ \mu\in\Gamma_{\Omega} : \inf_{\nu}\inf_{s\in\Omega} V^{\phi, 1}_{(\mu, \nu)}(s) \ge \zeta \right\}$.
Let $\Delta_{\mathrm{marg}}(\zeta):=\Vert V^*_{\Omega} - V^*_{\Omega, \zeta} \Vert_\infty$, and let $V^*_{K}$ denote the optimal joint value of Eq. \ref{eq:joint_st_value_eqn} over kernel-preserving controllers, for some finite $K\ge \underline{K}$. 

\begin{restatable}[Finite-$K$ characterization]{theorem}{finitekchar}\label{thm:finite_k_char}
    Define $\zeta_K := \frac{R_{\mathrm{max}} + \overline{V}}{K} = \frac{(2-\gamma)R_{\mathrm{max}}}{(1-\gamma)K}$.
    Assume there exists $\overline{\zeta}$ such that $\Gamma_{\Omega, \zeta_K} \neq \emptyset$ for all $0 < \zeta_K < \overline{\zeta}$ (sufficiently large $K$). Then, $0\le V^*_{\Omega}(s) - V^*_{K}(s) \le \Delta_{\mathrm{marg}}(\zeta_K)$ for $ s\in\Omega$.
    Hence, if $\Delta_{\mathrm{marg}}(\zeta)\rightarrow 0$ as $\zeta\downarrow 0$, then $V^*_K\rightarrow V^*_{\Omega}$ as $K\rightarrow \infty$. (Proof in Appendix \ref{app:convergence_proofs}).
\end{restatable}

Thus, when $\delta_\Omega > 0$, with sufficiently large, finite-$K$, the tabular two-timescale recursion converges to the optimal safety-constrained task value; when $\delta_\Omega=0$, the finite-$K$ optimality gap vanishes as $K\rightarrow \infty$. 

\vspace{-2mm}
\section{Experiments}\label{sec:empirical_evaluation}\vspace{-2mm}
In this section, we perform empirical evaluations on our proposed framework, where we choose to answer the following questions:
(1) \textit{Does the two-timescale learning achieve stable convergence on continuous-control tasks via neural approximations?} (2) \textit{How does our joint policy compare against CMDP baselines in maintaining near-zero violations during evaluation without sacrificing task performance?} (3) \textit{How does the learned safe set dynamically expand during training? }


Appendix \ref{app:subsec_low_dim_tabular} discusses the exact tabular DP solution of the joint Bellman equation on a discrete double integrator, with comparisons to projection and safety-filtering baselines. Here, we instead discuss the results of scaling the optimal joint policy to high-dimensional continuous spaces where exact dynamic programming is intractable. To this end, we introduce \texttt{JointSAC}: a deep RL approximation that closely follows the ISAACS framework \citep{hsu2023isaacs} (built on SAC \citep{haarnoja2018soft}). Controller and disturbance policies are trained using \textit{gradient descent-ascent finite timescale separation} ($\tau$-GDA) \citep{fiez2021local}.

\vspace{-2mm}
\subsection{Experimental Setup}\label{subsec:exp_setup}

We benchmark our framework on high-dimensional continuous control tasks in \texttt{MuJoCo} \citep{todorov2012mujoco}, utilizing three \texttt{Gymnasium} locomotion environments \citep{towers2026gymnasium} and eight \texttt{SafeVelocity} environments from \texttt{SafetyGymnasium} \citep{ji2023safety}. We augment the environment with an adversarial disturbance that actively attacks the agent to either violate safety constraints or to prevent liveness.

\begin{figure}[!htb]
    \centering
    \includegraphics[width=0.9\linewidth, trim=0 0 0 0, clip]{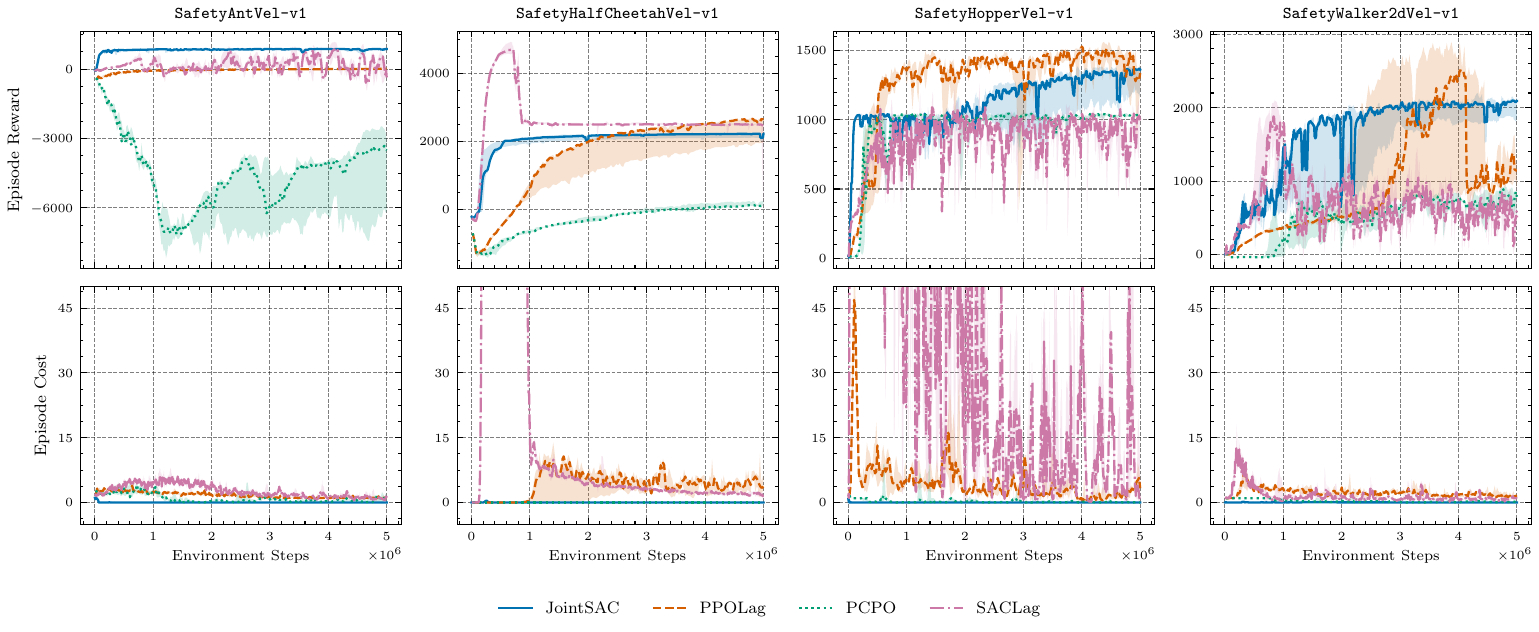}\vspace{-2mm}
    \caption{Training performance across 5 seeds (mean $\pm$ standard deviation). The top row indicates episodic returns, while the bottom row indicates episodic costs as training progresses. \texttt{JointSAC} demonstrates stable convergence with near-zero episode costs.}\vspace{-2mm}
    \label{fig:training_curves}
\end{figure}

\textbf{Baselines}:
We compare our method against least-restrictive safety filters (LRSF), and three well-established CMDP baselines. We specifically select these CMDP baselines, as they represent the two dominant paradigms of model-free Safe RL: \textit{Lagrangian dual-descent} (\texttt{PPOLag, SACLag}) and \textit{trust-region projections} \texttt{PCPO} from omnisafe \citep{ji2024omnisafe}. Additional experiments and details can be found in Appendix \ref{app:additional_exp_section} (including the safe set expansion); an ablation of episode return/cost as $K$ varies can be found in Appendix \ref{app:K_ablation}.
\vspace{-2mm}
\subsection{Main Results}\label{subsec:main_results}
Fig.~\ref{fig:training_curves} illustrates training curves across four \texttt{SafeVelocity} (\texttt{v1}) environments (5 seeds). \texttt{JointSAC} demonstrates stable convergence to high task rewards with near-zero costs, whereas baselines exhibit higher variance, with only a subset that meets the permissible cost threshold ($d=0$). Table~\ref{tab:robust_true} summarizes final rollout evaluations (300 evaluations per training seed, totaling $1,500$ distinct eval seeds). Across all 11 environments, \texttt{JointSAC} maintains near-zero violations (with maximum cost of $0.03\pm0.08$ on \texttt{Hopper-v4}), while baselines exhibit inconsistent, sometimes, major cost violations (\texttt{PCPO} walks backwards on \texttt{SafetyAntVel}). Overall, \texttt{JointSAC} yields stable, high task performance while enforcing near-zero safety violations.\vspace{-2mm} 

\begin{table*}[!htb]
\centering
\caption{Evaluation Metrics on Adversarial Environments with values indicating mean $\pm$ standard deviation across $5$ training seeds and $300$ rollouts each (totalling $1,500$ evals). Results indicate that \texttt{JointSAC} (Ours) has near-zero costs (maximum of $0.03 \pm 0.08$ on only \texttt{Hopper-v4}).}\vspace{-2mm}\
\label{tab:robust_true}
\resizebox{\textwidth}{!}{
\begin{tabular}{@{}lcccccccccc@{}}
\toprule
 & \multicolumn{5}{c}{\textbf{Reward}} & \multicolumn{5}{c}{\textbf{Cost}} \\
\cmidrule(lr){2-6} \cmidrule(lr){7-11}
\textbf{Environment} & \textbf{JointSAC (Ours)} & \textbf{Safety Filter} (LRSF) & \textbf{PCPO} & \textbf{PPO-Lag} & \textbf{SAC-Lag} & \textbf{JointSAC (Ours)} & \textbf{Safety Filter} (LRSF) & \textbf{PCPO} & \textbf{PPO-Lag} & \textbf{SAC-Lag} \\
\midrule
Hopper-v4 & $3266.89 \pm 103.80$ & $1169.92 \pm 326.17$ & $1373.82 \pm 303.63$ & $1074.26 \pm 562.70$ & $754.89 \pm 495.76$ & $0.03 \pm 0.08$ & $0.00 \pm 0.00$ & $0.00 \pm 0.00$ & $0.97 \pm 0.07$ & $0.40 \pm 0.54$ \\
InvertedPendulum-v4 & $999.43 \pm 1.28$ & $1000.00 \pm 0.00$ & $34.77 \pm 4.93$ & $506.30 \pm 464.77$ & $1000.00 \pm 0.00$ & $0.00 \pm 0.00$ & $0.00 \pm 0.00$ & $1.00 \pm 0.00$ & $0.60 \pm 0.55$ & $0.00 \pm 0.00$ \\
Walker2d-v4 & $2869.59 \pm 319.83$ & $973.26 \pm 3.10$ & $849.93 \pm 94.91$ & $3962.76 \pm 1138.50$ & $700.85 \pm 357.47$ & $0.00 \pm 0.00$ & $0.00 \pm 0.00$ & $0.01 \pm 0.01$ & $0.25 \pm 0.42$ & $0.27 \pm 0.39$ \\
\midrule 
SafetyAntVel-v0 & $941.08 \pm 8.26$ & $674.19 \pm 127.31$ & $-1509.56 \pm 1724.06$ & $112.66 \pm 230.86$ & $1251.74 \pm 562.78$ & $0.00 \pm 0.00$ & $0.05 \pm 0.04$ & $0.04 \pm 0.07$ & $0.89 \pm 0.38$ & $0.12 \pm 0.12$ \\
SafetyAntVel-v1 & $940.70 \pm 9.06$ & $632.65 \pm 62.56$ & $-2915.70 \pm 2404.71$ & $62.72 \pm 103.62$ & $582.13 \pm 963.46$ & $0.00 \pm 0.00$ & $0.00 \pm 0.00$ & $0.38 \pm 0.47$ & $1.03 \pm 0.15$ & $0.97 \pm 1.18$ \\
SafetyHalfCheetahVel-v0 & $1783.42 \pm 292.07$ & $986.44 \pm 76.15$ & $453.22 \pm 247.44$ & $2214.42 \pm 656.13$ & $2255.89 \pm 100.72$ & $0.00 \pm 0.00$ & $0.01 \pm 0.01$ & $0.01 \pm 0.02$ & $152.87 \pm 341.14$ & $0.15 \pm 0.34$ \\
SafetyHalfCheetahVel-v1 & $2238.10 \pm 107.59$ & $1152.01 \pm 110.50$ & $422.09 \pm 204.15$ & $2295.79 \pm 725.03$ & $1686.74 \pm 1505.66$ & $0.00 \pm 0.00$ & $0.00 \pm 0.00$ & $0.00 \pm 0.00$ & $0.36 \pm 0.59$ & $0.47 \pm 0.68$ \\
SafetyHopperVel-v0 & $1056.06 \pm 0.50$ & $974.38 \pm 24.15$ & $936.84 \pm 269.07$ & $1016.16 \pm 375.81$ & $1045.57 \pm 61.95$ & $0.00 \pm 0.00$ & $0.00 \pm 0.00$ & $0.90 \pm 1.99$ & $1.38 \pm 1.95$ & $17.06 \pm 32.27$ \\
SafetyHopperVel-v1 & $1329.36 \pm 89.75$ & $1001.40 \pm 13.29$ & $1039.24 \pm 8.99$ & $1274.01 \pm 557.27$ & $1027.49 \pm 18.75$ & $0.02 \pm 0.03$ & $0.00 \pm 0.00$ & $0.01 \pm 0.02$ & $4.11 \pm 8.84$ & $0.00 \pm 0.00$ \\
SafetyWalker2dVel-v0 & $1842.44 \pm 309.17$ & $547.07 \pm 138.62$ & $640.29 \pm 352.84$ & $1583.43 \pm 1248.16$ & $462.38 \pm 353.09$ & $0.01 \pm 0.01$ & $0.35 \pm 0.33$ & $0.22 \pm 0.36$ & $1.56 \pm 1.48$ & $0.76 \pm 0.61$ \\
SafetyWalker2dVel-v1 & $2056.53 \pm 332.37$ & $575.99 \pm 131.44$ & $828.76 \pm 141.58$ & $1178.47 \pm 1061.40$ & $433.04 \pm 366.89$ & $0.02 \pm 0.04$ & $0.21 \pm 0.19$ & $0.11 \pm 0.14$ & $1.52 \pm 1.52$ & $0.60 \pm 0.46$ \\
\bottomrule
\end{tabular}
}\vspace{-3mm}
\end{table*}

\section{Conclusion and Limitations}\label{sec:conclusion_limitations}\vspace{-2mm}
In this paper, we present a novel robust Bellman operator that jointly enables safety-constrained task optimization. We propose a two-timescale stochastic approximation algorithm to compute its fixed point and prove theoretical convergence to a set of optimal safety-constrained task value functions. The optimal joint policy theoretically guarantees safety while maximizing task returns. To scale the computation of the optimal fixed point to continuous-control tasks with a high-dimensional state space, we propose \texttt{JointSAC}. Empirical evaluations demonstrate that the neural approximation yields stable, low-variance training, while achieving near-zero safety violations.

\textbf{Limitations.}
Firstly, while our theoretical analysis ensures strict safety, scaling to high-dimensional continuous spaces necessitates neural approximations, violating strict forward invariance. Secondly, theoretical convergence relies on a heavily scaled safety margin ($K\ge K^*$). This can introduce numerical instabilities; in practice, $K$ needs to be heuristically tuned as high as practically possible. Finally, our approach assumes well-defined, continuous safety margins under full state observability, which are often difficult to perfectly specify in unstructured real-world environments.

\section{Acknowledgment}
We thank members of the Safe Robotics Lab for their insightful discussion and comments on our work. This work was made possible through the use of computational resources and support provided by Princeton Research Computing. Generative AI tools were used for grammar and spelling checks, word choice, rephrasing, and assistance with editing code and launching experiments.

\bibliography{iclr2027_conference}
\bibliographystyle{iclr2027_conference}

\appendix
\section{Preliminaries}
\label{app:preliminaries_section}
In this section, we extend Sec.~\ref{sec:preliminaries} to provide definitions of some of the notations and concepts used in the paper, followed by brief preliminaries on Reinforcement Learning, Differential Inclusions, Asymptotic Pseudotrajectories, Internally Chain-Transitive Sets, Constrained Markov Decision Processes, and Hamilton-Jacobi Reachability analysis.

\subsection{Definitions and Notations} \label{app:defn_notations}
\begin{definition}
    (State, Action, and Disturbance Spaces) 
    We consider a discounted MDP with finite state, action and disturbance spaces: $\mathcal{S} \subset \mathbb{R}^n$, $\mathcal{A} \subset \mathbb{R}^m$, and $\mathcal{D} \subset \mathbb{R}^l$ respectively, with the transition kernel $P(\cdot|s,a,d) \in \Delta(\mathcal{S})$. 
\end{definition}
\begin{definition}
    (Reward, Safety Margin, and the Viability Kernel) 
    The reward and the safety margin are given by $r: \mathcal{S} \times \mathcal{A} \times \mathcal{D} \rightarrow [0, R_{\text{max}}], R_{\text{max}} >0$, and $g: \mathcal{S} \rightarrow [-G_{\text{min}}, G_{\text{max}}]$, with $G_{\text{min}}, G_{\text{max}} > 0$ respectively. We subsume the extended class-$\mathcal{K}_e$ scaling function $\beta(\cdot)$ into $g(s)$ to align with task return scales (unless stated otherwise). The \textit{constraint set} is $\mathcal{F}^c := \{s \in \mathcal{S}: g(s)\ge0\}$, complementing the failure set $\mathcal{F}$. The \textit{viability kernel} $\Omega \subseteq \mathcal{F}^c$ is the subset of states from which a control policy exists that can keep the agent away from the failure set $\mathcal{F}$ for all times, regardless of the adversarial disturbance.
\end{definition}
\begin{definition}
    (Action and Disturbance Policies) 
    A deterministic joint policy $\pi = (\mu, \nu)$ pairs the mappings $\mu: \mathcal{S} \rightarrow \mathcal{A}$ and $\nu: \mathcal{S} \times \mathcal{A} \rightarrow \mathcal{D}$, yielding $\pi(s) = (\mu(s), \nu(s, \mu(s)))$.
\end{definition}
\begin{definition}
    (Value and Action-Value Functions) 
    the value and action-value function spaces are defined as $\mathcal{V}:=\mathbb{R}^{|\mathcal{S}|}$ and $Q:=\mathbb{R}^{|\mathcal{S}||\mathcal{A}||\mathcal{D}|}$, equipped with the supremum norm $\Vert\cdot\Vert_{\infty}$.
\end{definition}

\subsection{Reinforcement Learning} \label{app:rl}
Reinforcement Learning (RL)~\citep{sutton1998introduction, agarwal21} consists of data-driven methods enabling agents to compute approximations of the optimal value function and/or optimal policy to an optimal control problem. In RL, the objective is to maximize the expected cumulative sum of rewards over a finite horizon $T$, or the expected discounted cumulative sum in the infinite horizon setting. Accordingly, the general form of the objective function without disturbances is written as 
\par \vspace{-0\parskip} {\normalsize
\begin{equation}
    J(\pi) := \mathop{\mathbb{E}}_{\omega \sim \pi} \left[ \sum_{t=0}^{\infty} {\gamma}^t r(s_t, a_t)\middle\vert s_0\sim\mu_0\right]
\end{equation}}

where $\omega$ denotes a sampled trajectory $\omega=(s_0, a_0, s_1, a_1, \dots)$ generated by the policy $\pi$ and $\mu_0:\mathcal{S}\rightarrow[0, 1]$ denotes the initial state distribution. Considering stochastic dynamics, the dynamic programming principle associated with the action-value function $Q : \mathcal{S} \times \mathcal{A} \rightarrow \mathbb{R}$ representing the maximum expected cumulative discounted reward takes the form of a discrete-time Bellman equation~\citep{bellman1957dynamic}
\par \vspace{-0\parskip} {\normalsize
\begin{equation}
    Q(s,a) = r(s,a) + \gamma \mathbb{E}_{s' \sim p(\cdot \vert s,a)}\left[\max_{a' \in \mathcal{A}} Q (s', a')\right]
    \label{eq:st_rl_bellman}
\end{equation}}
For any $Q \in \mathcal{Q}$, the Bellman operator $\mathcal{T}_{\mathcal{M}}: \mathcal{Q} \rightarrow \mathcal{Q}$ associated with Eq.~\ref{eq:st_rl_bellman} is defined as:
\par \vspace{-0\parskip} {\normalsize
\begin{equation}
    \mathcal{T} [Q](s,a) :=r(s,a) + \gamma \mathbb{E}_{s' \sim p(\cdot \vert s,a)}\left[\max_{a' \in \mathcal{A}} Q (s', a')\right]
    \label{eq:st_rl_bell_op}
\end{equation}}
The Bellman operator defined in Eq.~\ref{eq:st_rl_bell_op} is a contraction mapping on $\mathcal{Q}$ (under the $L_\infty$ norm). By the Banach Fixed-Point Theorem~\citep{banach1922operations}, this implies that iteratively applying $\mathcal{T}_{\mathcal{M}}$ to any initial action-value function $Q$ converges to a unique fixed point $Q^*$, which is the solution to Eq.~\ref{eq:st_rl_bellman}. This theoretical guarantee underpins modern Deep RL, justifying the use of neural networks to approximate this fixed point in high-dimensional state spaces.

\subsection{Differential Inclusions, Asymptotic Pseudotrajectories, and Internally Chain-Transitive Sets}\label{app:di_apt_icts}
Let $E=\mathbb{R}^d$ be equipped with the supremum norm $\Vert\cdot\Vert_{\infty}$. Consider a set-valued map $\Psi:E\rightrightarrows E$. $\Psi$ is \textit{Marchaud} if it is upper semicontinuous, has nonempty compact convex values, and has at most linear growth. Under these conditions, the differential inclusion
\begin{align}
\dot{x}(t) \in \Psi(x(t)) \label{eq:marchaud_diff_inclusion}
\end{align}
admits absolutely continuous solutions from every initial condition \citep{aubin2012differential, benaim2005stochastic}. Since solutions need not be unique, Eq. \ref{eq:marchaud_diff_inclusion} induces a set-valued dynamical system $\Phi_t(x) := \left\{ z(t): z(0) = x, \dot{z}(\tau) \in \Psi(z(\tau))\ \text{a.e.} \right\}$

\begin{definition}(Asymptotic Pseudotrajectory \citep{benaim1996asymptotic})\label{defn:apt}
Let $\Phi_t$ denote the set-valued flow of Eq. \ref{eq:marchaud_diff_inclusion}. A continuous function $y:\left[0, \infty\right)\rightarrow E$ is an \textit{asymptotic pseudotrajectory} (APT) of $\Phi$, if for every T > 0,
\begin{align}
    \lim_{t\rightarrow\infty} \sup_{0\le h\le T} \mathrm{dist}\left(y(t+h), \Phi_h\left(x(t)\right)\right) = 0
\end{align}
\end{definition}
\begin{definition}(Internally Chain-Transitive Set \citep{benaim2005stochastic})\label{defn:ict}
    A compact invariant set $A\subseteq E$ is \textit{internally chain-transitive} (ICT), if for any $a, b\in A$ and any $\epsilon>0, T>0$, there exist points $a=x_0,x_1,\dots, x_k=b$ in A and times $t_1, \dots, t_k\ge T$ such that
    \begin{align}
        \mathrm{dist}\left(\Phi_{t_i}\left(x_{i-1}\right), x_i\right) < \epsilon \qquad i=1,\dots,k
    \end{align}
\end{definition}


\begin{theorem}(Limit set of a bounded APT \citep{benaim2005stochastic})\label{thm:APT_ICT}
Let $\Psi$ satisfy the preceding conditions, and let $y:[0, \infty)\rightarrow E$ be a bounded APT of $\dot{x}(t) \in \Psi(x(t))$. Then, its $\omega$-limit set:
\begin{align}
    \omega(y) = \bigcap_{t\geq 0} \overline{\left\{y(s): s\geq t\right\}}
\end{align}
is a non-empty, compact, invariant and \textit{internally chain-transitive} set-valued dynamics $\Phi$.
\end{theorem}
In the two-timescale setting considered in this work, the required APT property is obtained from the \textit{stochastic recursive inclusion} framework of \citep{10.1287/moor.2019.1037}: after the ``\textit{equilibration}'' of the fast timescale, the slow-timescale Bellman residual is averaged w.r.t admissible limiting occupation measures of the quasi-static fast timescale dynamics, yielding the averaged differential inclusion in Theorem \ref{thm:ode}. 


\subsection{Constrained Markov Decision Processes}\label{app:cmdps}
A Constrained Markov Decision Process (CMDP) \citep{altman1999constrained} extends the standard Markov Decision Process (MDP) framework by introducing auxiliary cost functions that encode safety constraints and perhaps, resource limitations. Formally, a CMDP is defined by the tuple $\mathcal{M}=\left(\mathcal{S, A, P}, r, c, \gamma, \mu_0\right)$, where $\mathcal{S}$ denotes the state space, $\mathcal{A}$ denotes the action space, $\mathcal{P}:\mathcal{S}\times\mathcal{A}\rightarrow\Delta(\mathcal{S})$ defines the transition dynamics, $r:\mathcal{S}\times\mathcal{A}\rightarrow\mathbb{R}$ denotes the task reward function, $c:\mathcal{S}\times\mathcal{A}\rightarrow\mathbb{R}$ denotes the constraint cost function, $\gamma\in\left[0, 1\right)$ denotes the discount factor and $\mu_0:\mathcal{S}\rightarrow[0, 1]$ denotes the initial state distribution. 

Let $\pi:\mathcal{S}\rightarrow\Delta(\mathcal{A})$ denote a stochastic policy, with the corresponding \textit{expected discounted cumulative return} $J_r(\pi)$ and \textit{expected discounted cumulative cost} $J_c(\pi)$ being defined as:
\begin{align}
    J_r(\pi) := \mathbb{E}_{\omega\sim\pi}\left[\sum_{t=0}^{\infty}\gamma^t r(s_t, a_t)\middle\vert s_0\sim\mu_0\right], \quad J_c(\pi) := \mathbb{E}_{\omega\sim\pi}\left[\sum_{t=0}^\infty \gamma^tc(s_t, a_t)\middle\vert s_0\sim\mu_0\right]
\end{align}
where $\omega=(s_0, a_0, s_1, a_1,\dots)$ denotes a trajectory generated by the policy $\pi$, with $s_0\in\mu_0$, $a_t\sim\pi(\cdot|s_t)$, and $s_{t+1} \sim\mathcal{P}(\cdot|s_t, a_t)$. Then, the objective in a CMDP is to find an optimal policy $\pi^*$ that maximizes the expected return, while ensuring that the cost remains below a predefined safety threshold $d\in\mathbb{R}$:
\begin{align}
    \max_{\pi\in\Pi} J_r(\pi) \quad \mathrm{s.t.} \quad J_c(\pi)\le d
\end{align}
where $\Pi$ is the set of all stationary stochastic policies. To solve this constrained optimization problem, a standard approach is to employ the \textit{Lagrangian relaxation} by introducing a dual variable $\lambda\ge 0$ (Lagrange multiplier). One can then obtain the unconstrained Lagrangian function, and the resulting \textit{min-max saddle-point problem}: 
\begin{align}
    \min_{\lambda\ge 0}\max_{\pi\in\Pi}\mathcal{L}(\pi, \lambda), \quad \mathrm{where}\quad \mathcal{L}(\pi, \lambda):=J_r(\pi) - \lambda(J_c(\pi)-d)
\end{align}
Under Slater's condition (assuming that at least one strictly safe policy exists, where $J_c(\pi)\le d$), strong duality holds, and optimizing this unconstrained minimax objective is equivalent to solving the primal CMDP. More details can be found in seminal works on CMDPs\citep{altman1999constrained, achiam2017constrained, tesslerreward, stooke2020responsive}.

\subsection{Hamilton-Jacobi Safety Analysis}\label{app:hji_safety_analysis}
Consider a continuous-time dynamical system $\dot{s}(t)=f(s(t), a(t), d(t))$, where $s\in\mathcal{S}, a\in\mathcal{A}$ and $d\in\mathcal{D}$ as in Definition \ref{subsec:defns_notation}. We assume that the specification of a failure set is given by a margin function $l:\mathcal{S}\rightarrow\mathbb{R}$ such that $\mathcal{F}^c:=\left\{s : g(s) \geq 0\right\}$. The goal of safety analysis is then to determine the \textit{safe set} $\Omega\subseteq\mathcal{F}^c\subseteq\mathcal{S}$ which is the set of all initial states from which there exists a control policy that can keep the system away from the failure set $\mathcal{F}$, at all times, regardless of the adversarial disturbance \citep{mitchell2005time}. This is formalized as a \textit{zero-sum differential game}, yielding the robust safety value function $V(s)$:  
\begin{align}
    V(s) = \inf_{\mathcal{N}[\pmb{d}](\cdot)}\sup_{\pmb{a}(.)} \inf_{t \geq 0} \ g\left(\xi_{s_0}^{\pmb{a}, \mathcal{N}[\pmb{d}]}(t)\right)
\end{align}
where $\xi_{s_0}^{\pmb{a}, \mathcal{N}[\pmb{d}]}(\cdot)$ denotes the trajectory starting at the state $s_0$, and evolving according to the control and disturbance trajectories $(\pmb{a}, \mathcal{N}[\pmb{d}])$, with $\mathcal{N}[\pmb{d}](\cdot)$ is defined to be the set of \textit{non-anticipative strategies} for the disturbance, as defined in \citet{evans1984differential} and \citet{fisac2017general}. By the principle of Dynamic Programming, $V(s)$ is the unique viscosity solution to the Hamilton-Jacobi-Isaacs (HJI) variational inequality: 
\begin{align}
    0 = \min \left\{ g(s) - V(s,t), \frac{\partial V}{\partial t}(s, t) + \max_{a \in \mathcal{A}}\min_{d\in\mathcal{D}} \nabla_s V^\top(s, t) f(s,a, d) \right\}\label{eq:HJI}
\end{align}
where $V(s,T) = g(s),  \forall s \in \mathcal{S}$, $t \in [0, T]$. In the infinite-horizon $T\rightarrow\infty$ case, $V(s, t)$ becomes independent of $t$ as long as there exists a nonempty safe set. Under stochastic dynamics, one can obtain the infinite-horizon Bellman equation as:
\begin{align}\tag{\ref{eq:HJI}}
    V(s) = \min \left\{ g(s), \max_{a \in \mathcal{A}} \min_{d\in\mathcal{D}} \mathbb{E}_{s'\sim p(\cdot|s, a, d)}\left[V(s')\right] \right\} 
\end{align}
where $s'\sim p(\cdot|s, a, d) \iff s' = s + f(s, a, d)\Delta t + \varepsilon_t$. However, the Bellman operator that governs Eq. \ref{eq:HJI} doesn't induce a contraction mapping in the space of safe value functions. To leverage techniques from TD-learning and deep RL, \citet{fisac2019bridging} introduces a discount factor $\gamma \in (0,1)$ that ensures contraction on the space of value functions. This results in the following \textit{time-discounted Bellman equation}:
\begin{align}\label{eq:disc_safe_vs}
    V(s) = (1-\gamma) g(s) + \gamma \min \left\{ g(s), \max_{a \in \mathcal{A}}\min_{d\in\mathcal{D}} \mathbb{E}_{s'\sim p(\cdot|s, a, d)}\left[V(s')\right]\right\}
\end{align}
For more details, please refer to prior work\citep{evans1984differential, mitchell2005time, fisac2017general, fisac2019bridging}.

\section{Proofs of Theoretical Results}\label{app:proofs_theoretical_results_section}
In this section, we provide detailed proofs and remarks on the theoretical results in Sec.~\ref{sec:defining_joint_action_value_function}, Sec.~\ref{subsec:objective_function}, Sec.~\ref{sec:two_timescale_approx}, Sec.~\ref{sec:forward_invariance}, and Sec.~\ref{sec:convergence}.
\subsection{Safety Action-Value of the Joint Policy (Sec.~\ref{subsec:contraction_safety_action_value})}\label{app:contraction_safe_action_value}
Given a joint policy $\pi$, the safety Bellman equation is defined as:
\begin{equation}
\tag{\ref{eq:safe_action_value}}
    Q^\phi_\pi(s,a,d) = (1-\gamma)g(s) + \gamma \min\left\{ g(s), \mathop{\mathbb{E}}_{s' \sim p(\cdot | s,a,d)} \left[\max_{a' \in \mu(s')} \min_{d' \in \nu(s',a')} Q^\phi_\pi (s' , a', d') \right]\right\}
\end{equation}
Note that $\nu(\cdot, a')$ refers to the specific sample $a'$ drawn from $\mu(\cdot)$ by the $\max_{a'\in\mu}$ operation. The corresponding safety Bellman operator $\mathcal{T}^\phi_\pi: \mathcal{Q} \rightarrow \mathcal{Q}$ is defined as:
\begin{equation}
\tag{\ref{eq:safety_bellman}}
    \mathcal{T}^\phi_\pi[Q^\phi_\pi](s,a,d) = (1-\gamma)g(s) + \gamma \min\left\{ g(s), \mathop{\mathbb{E}}_{s' \sim p(\cdot | s,a,d)} \left[\max_{a' \in \mu(s')} \min_{d' \in \nu(s',a')} Q^\phi_\pi(s' , a', d') \right]\right\}
\end{equation}

\safetybellmancontraction*
\begin{proof}
It will suffice to show for all tuples $(s,a,d) \in \mathcal{S} \times \mathcal{A} \times \mathcal{D}$,

$\bigg\vert\mathcal{T}^\phi_\pi\left[Q^\phi_\pi\right] (s,a,d) - \mathcal{T}^\phi_\pi[\widetilde{Q}^\phi_\pi](s,a,d) \bigg\vert$$\leq \gamma \norm{Q^\phi_\pi - \widetilde{Q}^\phi_\pi}_\infty$. We have: 
\begin{align*} 
\hspace{-0.4cm}\bigg\vert\mathcal{T}^\phi_\pi\left[Q^\phi_\pi\right] (s,a,d) - &\mathcal{T}^\phi_\pi[\widetilde{Q}^\phi_\pi](s,a,d) \bigg\vert = \bigg \vert (1-\gamma) g(s)\\
    &+ \!\gamma \!\min\left\{ \!g(s), \!\mathop{\mathbb{E}}_{s' \sim p(\cdot | s,a,d)} \left[\!\max_{a' \in \mu(s')} \min_{d' \in \nu(s',a')} \!Q^\phi_\pi (s' , a', d') \!\right]\!\right\} - \!(1-\gamma) g(s)\\
    &- \gamma \min\left\{ g(s), \mathop{\mathbb{E}}_{s' \sim p(\cdot | s,a,d)} \left[\max_{a' \in \mu(s')} \min_{d' \in \nu(s',a')} \widetilde{Q}^\phi_\pi (s' , a', d') \right]\right\}\bigg \vert\\
    & = \bigg \vert \gamma \min\left\{ g(s), \mathop{\mathbb{E}}_{s' \sim p(\cdot | s,a,d)} \left[\max_{a' \in \mu(s')} \min_{d' \in\nu(s',a')} Q^\phi_\pi (s' , a', d') \right]\right\}\\
    &- \gamma \min\left\{ g(s), \mathop{\mathbb{E}}_{s' \sim p(\cdot | s,a,d)} \left[\max_{a' \in \mu(s')} \min_{d' \in \nu(s',a')} \widetilde{Q}^\phi_\pi (s' , a', d') \right]\right\}\bigg \vert
\end{align*}
Observing that$\left\vert \min\left\{b,a\right\} - \min\left\{c,a \right\} \right\vert \leq \left\vert b -c \right\vert$ $\forall a,b,c \in \mathbb{R}$,
\begin{align*}           \hspace{-0.6cm}\bigg\vert\mathcal{T}^\phi_\pi\left[Q^\phi_\pi\right]& (s,a,d) - \!\mathcal{T}^\phi_\pi [\widetilde{Q}^\phi_\pi](s,a,d) \bigg\vert \\
&\leq \!\gamma \bigg \vert \!\mathop{\mathbb{E}}_{s' \sim p(\cdot | s,a,d)} \left[\max_{a' \in \mu(s')} \min_{d' \in \nu(s',a')}Q^\phi_\pi(s', a', d')\!-\! \max_{a' \in \mu(s')} \min_{d' \in \nu(s',a')} \widetilde{Q}^\phi_\pi(s', a', d')\!\right]\!\bigg \vert \\
     & \text{(Applying Jensen's inequality)}\\
     & \leq \gamma \mathop{\mathbb{E}}_{s' \sim p(\cdot | s,a,d)} \left[\bigg \vert \max_{a' \in \mu(s')} \min_{d' \in \nu(s',a')}Q^\phi_\pi(s', a', d')- \max_{a' \in \mu(s')} \min_{d' \in \nu(s',a')} \widetilde{Q}^\phi_\pi(s', a', d') \!\bigg \vert \right]
\end{align*}
Without loss of generality, let the first min-max term on the right-hand side be the larger one, and denote by $a^*$ a maximizer of this term. Then, we have:
\begin{align*}
\hspace{-0.1cm}\bigg\vert\!\mathcal{T}^\phi_\pi\left[Q^\phi_\pi\right] (s\!,a\!,d)\! -\! \mathcal{T}^\phi_\pi[\widetilde{Q}^\phi_\pi](s\!,a\!,d) \bigg\vert \!\leq \!\gamma\! \mathop{\mathbb{E}}_{s' \sim p(\cdot | s\!,a\!,d)} \left[\bigg \vert \min_{d' \in \nu(s'\!,a'\!)}\!Q^\phi_\pi(s'\!, a^*\!, d'\!)-\! \min_{d' \in \nu(s'\!,a'\!)} \!\widetilde{Q}^\phi_\pi(s'\!, a^*\!, d'\!) \bigg \vert \right]
\end{align*}
Similarly, let $d^*$ be a minimizer of the second, smaller term.
\begin{align*}
\hspace{-1.5cm}\bigg\vert\mathcal{T}^\phi_\pi\left[Q^\phi_\pi\right] (s,a,d) - \mathcal{T}^\phi_\pi[\widetilde{Q}^\phi_\pi](s,a,d) \bigg\vert &\leq \gamma \mathop{\mathbb{E}}_{s' \sim p(\cdot | s,a,d)} \left[\bigg \vert Q^\phi_\pi(s', a^*, d^*)-  \widetilde{Q}^\phi_\pi(s', a^*, d^*) \bigg \vert \right]\\
& \leq \gamma \sup_{s'} \max_{a' \in \mathcal{A}} \max_{d' \in \mathcal{D}} \bigg \vert Q^\phi_\pi(s',a',d') - \widetilde{Q}^\phi_\pi(s',a',d') \bigg \vert\\& = \gamma \norm{Q^\phi_\pi - \widetilde{Q}^\phi_\pi}_\infty
\end{align*} 
\end{proof}

\subsection{Joint Action-Value Function (Sec.~\ref{subsec:contraction_joint_action_value})}\label{app:contraction_joint_action_value}
For any fixed safety action-value function, $Y \in \mathcal{Q}$, the joint Bellman operator $\mathcal{T}_Y: \mathcal{Q} \rightarrow\mathcal{Q}$ is given by:
\begin{equation}\tag{\ref{eq:joint_bellman}}
\begin{aligned}
\mathcal{T}_Y[Q](s,a,d) = (1-\gamma) m(s,a,d)+ \gamma \min \Bigl\{ r(s,a,d) +\mathop{\mathbb{E}}_{s' \sim p(\cdot | s,a,d)} \left[V_Q(s')\right], Y(s,a,d) \Bigr\}
\end{aligned}
\end{equation}
where $m(s,a,d) = \min \{ r(s,a,d), g(s)\}$ and $V_Q(s) = \max_{a \in \mathcal{A}} \min_{d \in \mathcal{D}} Q(s,a,d) $.

\jointbellmancontraction*
\begin{proof}
    It will suffice to show for all tuples $(s,a,d) \in \mathcal{S} \times \mathcal{A} \times \mathcal{D}$,$\bigg\vert\mathcal{T}_Y\left[Q\right] (s,a,d) - \mathcal{T}_Y [\widetilde{Q}](s,a,d) \bigg\vert$ $\leq \gamma \norm{Q - \widetilde{Q}}_\infty$. We have: 

\begin{align*} 
    \hspace{-0.5cm}\bigg\vert\mathcal{T}_Y&\left[Q\right] (s,a,d) - \mathcal{T}_Y [\widetilde{Q}](s,a,d) \bigg\vert = \bigg \vert (1-\gamma) m(s,a,d)\\
    &+ \!\gamma \min \Bigl\{ r(s,a,d)\! + \!\mathop{\mathbb{E}}_{s' \sim p(\cdot | s,a,d)} \left[\max_{a' \in \mathcal{A}} \min_{d' \in \mathcal{D}} Q(s', a', d')\right], Y(s,a,d) \!\Bigr\} - \!(1-\!\gamma) m(s,a,d)\\
    &- \gamma \min \Bigl\{ r(s,a,d) + \mathop{\mathbb{E}}_{s' \sim p(\cdot | s,a,d)} \left[\max_{a' \in \mathcal{A}} \min_{d' \in \mathcal{D}} \widetilde{Q}(s', a', d')\right], Y(s,a,d) \Bigr\}\bigg \vert\\
    & = \bigg \vert \gamma \min \Bigl\{ r(s,a,d) + \mathop{\mathbb{E}}_{s' \sim p(\cdot | s,a,d)} \left[\max_{a' \in \mathcal{A}} \min_{d' \in \mathcal{D}} Q(s', a', d')\right], Y(s,a,d) \Bigr\}\\
    &- \gamma \min \Bigl\{ r(s,a,d) + \mathop{\mathbb{E}}_{s' \sim p(\cdot | s,a,d)} \left[\max_{a' \in \mathcal{A}} \min_{d' \in \mathcal{D}} \widetilde{Q}(s', a', d')\right], Y(s,a,d) \Bigr\}\bigg \vert
\end{align*}
Observing that $\left\vert \min\left\{b,a\right\} - \min\left\{c,a \right\} \right\vert \leq \left\vert b -c \right\vert$ $\forall a,b,c \in \mathbb{R}$,
\begin{align*}           
\hspace{-1cm}\bigg\vert\mathcal{T}_Y\left[Q\right] &(s,a,d) - \mathcal{T}_Y[\widetilde{Q}](s,a,d) \bigg\vert \leq \gamma \bigg \vert r(s,a,d) + \mathop{\mathbb{E}}_{s' \sim p(\cdot | s,a,d)} \left[\max_{a' \in \mathcal{A}} \min_{d' \in \mathcal{D}}Q(s', a', d')\right]\\
    & - r(s,a,d) - \mathop{\mathbb{E}}_{s' \sim p(\cdot | s,a,d)} \left[\max_{a' \in \mathcal{A}} \min_{d' \in \mathcal{D}} \widetilde{Q}(s', a', d')\right]\bigg \vert\\
    &= \gamma \bigg \vert \mathop{\mathbb{E}}_{s' \sim p(\cdot | s,a,d)} \left[\max_{a' \in \mathcal{A}} \min_{d' \in \mathcal{D}}Q(s', a', d')\right]
     - \mathop{\mathbb{E}}_{s' \sim p(\cdot | s,a,d)} \left[\max_{a' \in \mathcal{A}} \min_{d' \in \mathcal{D}} \widetilde{Q}(s', a', d')\right]\bigg \vert\\
     & = \gamma \bigg \vert \mathop{\mathbb{E}}_{s' \sim p(\cdot | s,a,d)} \left[\max_{a' \in \mathcal{A}} \min_{d' \in \mathcal{D}}Q(s', a', d')- \max_{a' \in \mathcal{A}} \min_{d' \in \mathcal{D}} \widetilde{Q}(s', a', d')\right]\bigg \vert \\
     & \text{(Applying Jensen's inequality)}\\
     & \leq \gamma \mathop{\mathbb{E}}_{s' \sim p(\cdot | s,a,d)} \left[\bigg \vert \max_{a' \in \mathcal{A}} \min_{d' \in \mathcal{D}}Q(s', a', d')- \max_{a' \in \mathcal{A}} \min_{d' \in \mathcal{D}} \widetilde{Q}(s', a', d') \bigg \vert \right]
\end{align*}
Without loss of generality, let the first min-max term on the right-hand side be the larger one, and denote by $a^*$ a maximizer of this term. Then, we have:
\begin{align*}
\hspace{-0.6cm}\bigg\vert\mathcal{T}_Y\left[Q\right] (s,a,d) - \mathcal{T}_Y [\widetilde{Q}](s,a,d) \bigg\vert \leq \gamma \mathop{\mathbb{E}}_{s' \sim p(\cdot | s,a,d)} \left[\bigg \vert \min_{d' \in \mathcal{D}}Q(s', a^*, d')- \min_{d' \in \mathcal{D}} \widetilde{Q}(s', a^*, d') \bigg \vert \right]
\end{align*}
Similarly, let $d^*$ be a minimizer of the second, smaller term.

\begin{align*}
\bigg\vert\mathcal{T}_Y\left[Q\right] (s,a,d) - \mathcal{T}_Y [\widetilde{Q}](s,a,d) \bigg\vert &\leq \gamma \!\mathop{\mathbb{E}}_{s' \sim p(\cdot | s,a,d)} \left[\bigg \vert Q(s', a^*, d^*)-  \widetilde{Q}(s', a^*, d^*) \bigg \vert \right]\\
&\! \leq \gamma \sup_{s'} \max_{a' \in \mathcal{A}} \max_{d' \in \mathcal{D}} \bigg \vert Q(s',a',d') - \widetilde{Q}(s',a',d') \bigg \vert = \gamma \norm{Q - \widetilde{Q}}_\infty
\end{align*}
\end{proof}
\subsection{Objective Function of Joint Action-Value Function (Sec.~\ref{subsec:objective_function})}\label{app:obj_function_derivation}
In this section, we derive the objective function from robust optimal stopping theory to provide a principled theoretical interpretation of the joint value function. To this end, we first prove supporting Lemmas.
\begin{lemma}\label{lem:safety_upper_bdd}
    (Global upper bound and non-negativity of $Q^\phi_\pi$ under Joint Trajectory Evaluation) For all $(s,a,d) \in \mathcal{S} \times \mathcal{A} \times \mathcal{D}$ and all greedy joint policies $\pi$, $Q^\phi_\pi (s,a,d) \leq g(s)$ independent of the policy evaluated. Consequently, for all states within the failure set, $s \in \mathcal{F}$, $Q^{\text{safe}}_{\pi^a, \pi^d} (s,a,d) < 0$, where the failure set is defined as the set $\mathcal{F} = \{ s \in \mathcal{S}: g(s)<0\}$.
\end{lemma}
\begin{proof}
    By definition,
    \begin{align*}
        Q^\phi_\pi (s,a,d) = (1-\gamma)g(s) + \gamma \min\left\{ g(s), \mathop{\mathbb{E}}_{s' \sim p(\cdot | s,a,d)} \left[\max_{a' \in \mu (s')} \min_{d' \in \nu(s',a')} Q^\phi_\pi (s' , a', d') \right]\right\}
    \end{align*}
    Noting that $\min \{x,y\} \leq x$, $Q^\phi_\pi (s,a,d) \leq (1-\gamma)g(s) + \gamma g(s) = g(s)$.
    For all states $s \in \mathcal{F}$, by definition, $g(s) < 0$. Furthermore, applying our previous result gives us $Q^\phi_\pi (s,a,d) < 0$. Finally, noting that $g(s)$ is independent of $\pi$, completes the proof.
\end{proof}
\begin{remark}
    As seen in Lemma~\ref{lem:safety_upper_bdd}, the upper bound of $Q^\phi_\pi$ relies solely on the minimum operator in the safety Bellman equation. Therefore, it holds universally for any arbitrary policy and does not require that the policy be greedy with respect to safety $Q$-function.
\end{remark}
\begin{remark}
    A positive immediate margin ($g(s) > 0$) does not guarantee viability. In inevitable failure states (unsafe states) where no safe policy exists, the long-term safety value correctly captures this future inevitable failure by evaluating to $Q^\phi_\pi(s,a,d) < 0$. Note that Lemma~\ref{lem:safety_upper_bdd} only provides a global upper bound and does not impose a lower bound, allowing $Q^\phi_\pi(s,a,d)$ to reflect the unsafe states.
\end{remark}
\optregions*
\begin{lemma}\label{lem:switching_time}
    (First Switching time as valid stopping time) Let $(\mathcal{W}, \mathcal{G}, \mathbb{P})$ be a probability space and let $(s_k)_{k=0}^T$ be a stochastic process taking values in a measurable state space $(\mathcal{S}, \mathcal{B}(\mathcal{S}))$. Let $\mathbb{F} = (\mathcal{G}_k)_{k=0}^T$ be the natural filtration where $\mathcal{G}_k = \sigma(s_0,\dots,s_k)$. Given a Borel-measurable safety maximizing region, $\Omega^\phi \in \mathcal{B}(\mathcal{S})$, define the first switching time $\tau : \mathcal{W} \rightarrow \{ 0, \cdots, T\}$ as:
    \[ \tau = \inf \{k \in \{0, \cdots, T  \} : s_k \in \Omega^\phi \}\]
    where $T$ is the time horizon ($T\rightarrow\infty$ for infinite horizon setting) and $\inf \emptyset = T$. Then the first switching time, $\tau$, is a stopping time with respect to $\mathbb{F}$.
\end{lemma}
\begin{proof}
    $\tau$ is a stopping time if and only if for every $k \in \{ 0, \cdots, T\}$, the set of $\{ \tau \leq k \} \in \mathcal{G}_k$. We start by observing that the event $\{ \tau \leq k \}$ is equivalent to union of all events where the trajectory switched the first time from task optimizing region to safety maximizing region, $\{ \tau \leq k \} = \bigcup_{j=0}^k \{s_j \in \Omega^\phi\}$. Noting that the state $s_j$ is $\mathcal{G}_j$-measurable by definition of natural filtration, $\{s_j \in \Omega^\phi\} \in \mathcal{G}_j$ because $\Omega^\phi$ is a Borel set in $\mathcal{S}$. Moreover, we have $\{s_j \in \Omega^\phi\}\in \mathcal{G}_k$ due to monotonicity of filtration (i.e., $\mathcal{G}_j \subseteq \mathcal{G}_k \; \forall j \leq k$). Finally, because $\sigma$-algebras are closed under finite union, we get $\left\{\tau \leq k\right\} \in \mathcal{G}_k $. This concludes the proof that $\tau$ is a valid stopping time.
\end{proof}

\begin{remark}
    (Reduction of the Joint Action-Value Function to Only Task Action-Value Function Prior to Switching) \\
    Prior to the first switching time $\tau$, the trajectory remains entirely within the task optimizing region $\Omega^\psi$. In this region, by definition, $V_q^\psi(s)=r_\pi(s) + \mathbb E[V_q(s')] < V_\pi^\phi(s)$. Furthermore, recalling from Lemma~\ref{lem:safety_upper_bdd} that the safety value is globally bounded by the immediate margin, and the fact that, since the successor state remains in the domain $V_q(s')\ge 0$, we have the strict ordering $r_{\pi}(s) \le V^\psi_\pi(s) < V^\phi_\pi(s) \leq g(s)$. This implies, $r(s,a) < g(s)$. Consequently, $V_q(s) = (1-\gamma)r_\pi(s) + \gamma V_q^\psi(s) = r_\pi(s) + \gamma\mathbb{E}[V_q(s')]$.
    This ensures that before entering the safety-maximizing region, the joint formulation evaluates only the task objective while remaining implicitly upper-bounded by $g(s)$.
\end{remark}
Let the cumulative return from time $t=0$, evaluated at a stopping time $\tau$ be given by $\mathcal{R}_\tau := \sum_{t=0}^{\tau -1} \gamma^t r(s_t) + \gamma^\tau \left((1-\gamma) \min \{r(s_\tau),g(s_\tau)\} +\gamma V^\phi_Y(s_\tau) \right)$, where $V^\phi_Y(s_\tau) = Y(s_\tau, \mu(s), \nu(s, \mu(s)))$. For a process observed from the next state $s_1$ (which we will \textit{alias} as $s_0'$), let $\tau'$ denote a stopping time. Then the cumulative return from $t=1$ is defined as : $\mathcal{R}_\tau' :=\sum_{t=0}^{\tau' -1} \gamma^t r(s'_t) + \gamma^{\tau'} \left((1-\gamma) \min \{r(s'_{\tau'}),g(s'_{\tau'})\} +\gamma V^\phi_Y(s'_{\tau'}) \right)$. Then, consider the following lemma:
\begin{lemma}
    Let $\Lambda_k$ be the set of stopping times $\tau$ bounded by time horizon $k$, i.e., $0 \leq \tau \leq k$. Let $\mathcal{R}_{\tau} = (\mathcal{R}_t)_{t \leq \tau}$ be the stochastic process representing the cumulative discounted reward. Let $V_{Q_k}(s_1) $ be the value function evaluated at the next state $s_1$, defined as the optimal stopping problem $V_{Q_k}(s_1) := \mathop{\inf_{\tau' \in \Lambda_{k}}} \;\mathop{\mathbb{E}}_{\omega}\left[\mathcal{R}'_{\tau'} \middle\vert s'_0 = s_1\right]$. Then the following equality holds, evaluated from $s_0=s$:
\par \vspace{-0\parskip} {\normalsize\begin{align*}
r(s) + \gamma\mathop{\mathbb{E}}_{\omega(s_1)} \left[V_{Q_k}(s_1) \middle\vert s_0 = s\right] \  = \inf_{\tau \in \Lambda_{k+1}, \; \tau \geq 1} \;  \mathop{\mathbb{E}}_{\omega} \left[ \mathcal{R}_{\tau}\middle\vert s_0 = s \right] 
\end{align*}}
\par \vspace{-0\parskip} {\normalsize where $\omega(s_1)$ denotes the $s_1$ element of the trajectory $\omega$.}
\label{lem:interchange_inf_and _exp}
\end{lemma} \vspace{-0\parskip}
\begin{proof}
    We will prove this by showing inequality in both directions. We first prove:
    \par \vspace{-0\parskip} {\normalsize\begin{align*}
    r(s) + \gamma\mathop{\mathbb{E}}_{\omega(s_1)} \left[V_{Q_k}(s_1) \middle\vert s_0 = s\right] \geq \inf_{\tau \in \Lambda_{k+1}, \; \tau \geq 1} \;  \mathop{\mathbb{E}}_{\omega} \left[ \mathcal{R}_{\tau}\middle\vert s_0 = s \right] 
    \end{align*}}
    \par \vspace{-0\parskip} {\normalsize To prove the above inequality, we will construct a global stopping time by \textit{pasting together} $\epsilon$-optimal local stopping times.} 
    \begin{enumerate}
    \item Fix an arbitrarily small $\epsilon > 0$. For every realization $z$ of the state $s_1$, there exists a valid stopping time $\tau'_z \in \Lambda_{k}$ such that its expected value is $\epsilon$-close to the optimal value $V_{Q_k}(z)$, i.e., $\mathbb{E}_\omega [\mathcal{R}'_{\tau'_z} | s'_0 = z] \leq V_{Q_k}(z) + \epsilon$.
    \item Define a global candidate stopping time $\tau^* \in \Lambda_{k+1}$ by waiting one step and then \textit{executing} $\tau'_z$, i.e., $\tau^* := 1+ \tau'_z$. By the recursive definition of cumulative return, $\mathcal{R}_{\tau^*} = r(s) + \gamma\mathcal{R}'_{\tau'_z}$.
    \item We now invoke Lemma 2.8 in~\citet{aguilar2009tsystemslowersnellenvelope} (Stability under Pasting, from optimal robust stopping time theory), which states that the set of stopping times is closed under such concatenation. Since $\tau'_z$ depends on the history after $t=1$, the constructed variable $\tau^*$ is non-anticipative and is a valid stopping time in $\Lambda_{k+1}$ with $\tau^* \geq 1$.
    \item The global infimum over $\Lambda_{k+1}$ must be less than or equal to the expectation of the specific candidate:
    \par \vspace{-0\parskip} {\normalsize\[\inf_{\tau \in \Lambda_{k+1}, \; \tau \geq 1} \;  \mathop{\mathbb{E}}_{\omega} \left[ \mathcal{R}_{\tau}\middle\vert s_0 = s \right] \leq \mathop{\mathbb{E}}_{\omega} \left[ \mathcal{R}_{\tau^*}\middle\vert s_0 = s \right]\]}
    \item Finally, applying the Law of Iterated Expectations and our $\epsilon$-bound yields:
    \par \vspace{-0\parskip} {\normalsize \begin{align*}
        \inf_{\tau \in \Lambda_{k+1}, \; \tau \geq 1} \;  \mathop{\mathbb{E}}_{\omega} \left[ \mathcal{R}_{\tau}\middle\vert s_0 = s \right] &\leq \mathop{\mathbb{E}}_{\omega(s_1)}\left[ r(s) + \gamma \mathbb{E}_\omega [\mathcal{R}'_{\tau'_{s_1}} | s'_0 = s_1] \right]\\
        &\leq \mathop{\mathbb{E}}_{\omega(s_1)}\left[r(s) +  \gamma\left(V_{Q_k}(s_1) + \epsilon\right) \right]\\
        &=r(s) + \gamma\mathop{\mathbb{E}}_{\omega(s_1)}\left[ V_{Q_k}(s_1)\right] + \gamma\epsilon
    \end{align*}}
    \end{enumerate}
    Since this inequality holds for any arbitrary small $\epsilon>0$, letting $\epsilon\rightarrow 0$ yields the first direction.
    
    Next, we prove the opposite direction:
    \par \vspace{-0\parskip} {\normalsize \begin{align*}
    r(s) + \gamma\mathop{\mathbb{E}}_{\omega(s_1)} \left[V_{Q_k}(s_1) \middle\vert s_0 = s\right] \leq \inf_{\tau \in \Lambda_{k+1}, \; \tau \geq 1} \;  \mathop{\mathbb{E}}_{\omega} \left[ \mathcal{R}_{\tau}\middle\vert s_0 = s \right]
    \end{align*}}
    \begin{enumerate}
        \item Let $\tau \in \Lambda_{k+1}$ be any arbitrary global stopping time with $\tau \geq 1$.
        \item Conditioning on the first transition $s_1 = z$, the remainder of the trajectory induces a local stopping time $\tau'_z = \tau - 1$. The induced stopping time satisfies $0\leq\tau'_z \leq k$, and thus, $\tau'_z \in \Lambda_k$.
        \item By the definition of value function as an infimum over $\Lambda_k$, we have $\mathbb{E}_\omega [\mathcal{R}'_{\tau'_z} | s'_0 = z] \geq V_{Q_k}(z)$.
        \item. Using Law of Iterated Expectations, 
        \par \vspace{-0\parskip} {\normalsize
        \begin{align*}
            \mathop{\mathbb{E}}_{\omega} \left[ \mathcal{R}_{\tau}\middle\vert s_0 = s \right] &= r(s) + \gamma \mathop{\mathbb{E}}_{\omega(s_1)}\left[\mathbb{E}_{\omega}\left[\mathcal{R}'_{\tau'_{s_1}}\middle\vert s'_0=s_1\right]\right] \\
            &\geq r(s) + \gamma\mathop{\mathbb{E}}_{\omega(s_1)}\left[ V_{Q_k}(s_1)\right]
        \end{align*}}
    \item Finally, since this inequality holds for any arbitrary $\tau \in \Lambda_{k+1}, \ \ (\tau\geq1)$, it must hold for the infimum over all such $\tau$, completing the proof.
    \end{enumerate} 
\end{proof}
\begin{remark}
    (On Non-Anticipativity) The invocation of Lemma 2.8 in~\citet{aguilar2009tsystemslowersnellenvelope} guarantees that our constructed variable $\tau^*$ uses only current information, since the specific local stopping time $\tau'_z$ is solely determined by the realization of state $s_1$ observed at time $t=1$. Lemma~\ref{lem:interchange_inf_and _exp} justifies the interchange of the infimum and expectation operators. Intuitively, it proves that the optimal value of the global stopping time is equivalent to the expectation of local optimal values of the local stopping times.
\end{remark}
\begin{theorem}\label{thm:fixed_traj_objective_fun}
    (Objective Function for a fixed trajectory) Let $\omega = (s_0, s_1, \cdots, s_T)$ be a trajectory generated by action policy, $\mu(s)$, and disturbance policy $\nu(s,\mu(s))$, with associated rewards $r(s_t)$. Here, $T$ denotes the time horizon. Then the objective function that the joint action-value function optimizes is given by:
    \[
    J(s) := \inf_{\tau \in \Lambda_T} \mathop{\mathbb{E}}_\omega \left[\sum_{t=0}^{\tau -1} \gamma^t r(s_t) + \gamma^\tau \left((1-\gamma) \min \{r(s_\tau),g(s_\tau)\} +\gamma V^\phi_Y(s_\tau) \right) \bigg\vert s_0 = s \right]
    \]
\end{theorem}
\begin{proof}
    We begin by observing that the first switching time $\tau$ from target optimizing region to safety maximizing region is a valid stopping time (Refer to Lemma~\ref{lem:switching_time}). We prove the theorem by induction on the time horizon $T$.
    \begin{enumerate}
        \item \textbf{Base Case ($T=1$):} Consider a time horizon $T=1$. The set of valid stopping times is given by $\Lambda_1 = \{ 0,1\}$. When $\tau=0$, the initial state $s_0 \in \Omega^\phi$ and the trajectory is terminated immediately, yielding the return $\mathcal{R}_0 = (1-\gamma) \min \{r(s_0) , g(s_0) \} + \gamma V^\phi_Y(s_0)$. When $\tau = 1$, the agent takes one step and is stopped and thus, the next state $s_1 \in \Omega^\phi$, with $s_0 \in \Omega^\psi$. This yields the expected return $\mathbb{E}[\mathcal{R}_1] = r(s_0) + \gamma \mathbb{E}_{\omega(s_1)}[(1-\gamma) \min \{r(s_1) , g(s_1) \} + \gamma V^\phi_Y(s_1) | s_0=s]$. As the value function is the infimum over $\Lambda_1$, it evaluates to the minimum of these two choices:
        \begin{align*}
        V_{Q_1}(s) = \min \Biggl\{ &(1-\gamma) \min \{r(s_0) , g(s_0) \} + \gamma V^\phi_Y(s_0),\\
        & r(s_0) + \gamma \mathbb{E}_{\omega(s_1)}[(1-\gamma) \min \{r(s_1) , g(s_1) \} + \gamma V^\phi_Y(s_1) | s_0=s] \Biggr\}
        \end{align*}
        Since the minimum is equivalent to the infimum over the set of valid stopping times $\tau \in \Lambda_1$, we have $V_{Q_1}(s) = \inf_{\tau \in \Lambda_1}\mathbb{E}_\omega [\mathcal{R}_\tau | s_0 =s]$.
        \item \textbf{Induction Hypothesis ($T=k$):} Assume for a time horizon $T=k$, the value function is given by:
        \[
        V_{Q_k}(s) = \inf_{\tau \in \Lambda_k} \mathop{\mathbb{E}}_\omega [\mathcal{R}_\tau|s_0=s]
        \]
        \item \textbf{Inductive Step ($T= k+1$):} Consider the time horizon $T=k+1$. According to the Bellman equation, the value at the current state is the minimum of stopping immediately ($\tau=0$) or continuing for one step and behaving optimally thereafter (which is horizon $k$ starting at $t=1$):
        \begin{align*}
            V_{Q_{k+1}}(s) = \min \Biggl\{ &(1-\gamma) \min \{r(s_0) , g(s_0) \} + \gamma V^\phi_Y(s_0),\\
            & r(s_0) + \gamma \mathbb{E}_{\omega(s_1)}[V_{Q_k}(s_1) | s_0=s]\Biggr\}
        \end{align*}
        Applying Lemma~\ref{lem:interchange_inf_and _exp} to the above yields:
        \begin{align*}
            V_{Q_{k+1}}(s) = \min \Biggl\{ &(1-\gamma) \min \{r(s_0) , g(s_0) \} + \gamma V^\phi_Y(s_0),\inf_{\tau \in \Lambda_{k+1}, \; \tau \geq 1} \;  \mathop{\mathbb{E}}_{\omega} \left[ \mathcal{R}_{\tau}\middle\vert s_0 = s \right]\Biggr\}
        \end{align*}
    \end{enumerate}
Finally, observing that the first term is exactly the return $\mathcal{R}_0$ corresponding to the case where $\tau = 0$, the outer minimum operator simplifies into the infimum over the strict subset $\tau\geq1$, yielding an infimum over the entire domain $0\leq \tau\leq k+1$:
\par \vspace{-0\parskip} {\normalsize
\begin{align*}
V_{Q_{k+1}}(s) &= \inf_{\tau\in\Lambda_{k+1}}\mathbb{E}_{\omega}\left[\mathcal{R}_\tau\middle\vert s_0=s\right]\\
&=\inf_{\tau \in \Lambda_T} \mathop{\mathbb{E}}_\omega \left[\sum_{t=0}^{\tau -1} \gamma^t r(s_t) + \gamma^\tau \left((1-\gamma) \min \{r(s_\tau),g(s_\tau)\} +\gamma V^\phi_Y(s_\tau) \right) \bigg\vert s_0 = s \right]
\end{align*}}
\end{proof}
\objectivefunc*
\begin{proof}
    This follows directly from Theorem~\ref{thm:fixed_traj_objective_fun}. Without loss of generality, the Theorem used state-based rewards $r(s_t)$ in the main proof for notational simplicity. The extension to state-action-based rewards $r(s_t, a_t, d_t)$ holds because the expectation is taken over the distribution of trajectories.
\end{proof}
\begin{restatable}[Boundary Condition Without Joint Actions]{remark}{objectivefuncremark}
\label{remark:objective_fun_remark}
    When the stopping time $\tau$ occurs such that no subsequent action is executed at the terminal state $s_\tau$, the state-action-disturbance reward $r(s_\tau, a_\tau, d_\tau)$ is not accrued. Consequently, the terminal condition collapses entirely to the safety value function, evaluating to $(1-\gamma) g(s_\tau)+ \gamma V^\phi_Y(s_\tau)$, with no contribution from the task reward.
\end{restatable}

\subsection{Fast Timescale: Safety Critic Tracking (Sec.~\ref{subsec:fast_timescale})}\label{app:fast_timescale_proofs}
\asspmartingalediffnoise*
\asspsteps*
The updates with step‑size sequences $\{\alpha_n\},\{\beta_n\}$ are given by:
\begin{align}
  Y_{n+1} &= Y_n + \beta_n \left[ \hat{\mathcal{T}}^{\phi}_{n+1, \pi_{Q_n}}[Y_n] - Y_n + \xi^{\phi}_{n+1}\right] \label{eq:fast}\\
  Q_{n+1} &= Q_n + \alpha_n \left[ \hat{\mathcal{T}}_{n+1, Y_n}[Q_n] - Q_n + \xi_{n+1} \right] \label{eq:slow}
\end{align}
where $\xi^\phi_{n+1}$ and $\xi_{n+1}$ are martingale difference perturbations~\citep{borkar2009stochastic}, with $\mathbb{E}[\xi^{\phi}_{n+1}|\mathcal{H}_n] = \mathbb{E}[\xi_{n+1}|\mathcal{H}_n] = 0$. As the map $Q\mapsto Y_{\pi_{Q}}$ need not be continuous or Lipschitz, the standard two-timescale stochastic approximation need not apply. Thus, we analyze the fast timescale recursion as a stochastic recursive inclusion in Section \ref{subsec:fast_timescale} using the \textit{set-valued two-timescale framework} ~\citep{10.1287/moor.2019.1037, borkar2023stochastic}.

\safetycritictracking*
\begin{proof}
    We first establish some properties of the fast timescale differential inclusion. 
    \begin{itemize}
        \item[-] Let $y\in\mathcal{Y}\subseteq \mathbb{R}^{|\cS||\mathcal{A}||\cD|}$ denote the fast timescale safety critic with the limiting dynamics $h_{\pi}(y):=\mathcal{T}^{\phi}_\pi[y] - y$. Here, $\pi$ is quasi-static, as induced by the quasi-static joint critic $q\in\mathcal{Q}$. Thus, note that for any $y\in\mathcal{Y}$, the operator $\mathcal{T}^{\phi}_\pi[y]$ induces a $\gamma$-contraction, as discussed in Section \ref{subsec:contraction_safety_action_value}, and thus possesses a unique fixed point.
        \item[-] Next, note that, since $\cS, \mathcal{A}, \cD$ are finite, $\Pi(q)$ is finite and nonempty for every $q$. For a given $q$, any action that is not maximin-greedy at $q$ has a strictly positive gap from the maximum, and any disturbance that is not minimizing on some greedy action has a strictly positive gap from the minimum. Consequently, there exists $\vartheta(q) > 0$ such that $\Vert q'-q \Vert_{\infty} < \vartheta(q)$, which implies that $\Pi(q')\subseteq\Pi(q)$. Hence, $\Pi(\cdot)$ is upper semicontinuous. Moreover, by continuity of $h_\pi(\cdot)$, the map $\Psi(\cdot, \cdot)$ (as defined in Eq. \ref{eq:set_valued_fast_tracking}) is Marchaud, since it has nonempty compact convex values, and is upper semicontinuous with at most linear growth.
        \item[-] Let $\overline{g}_K := K\max_{s\in\mathcal{S}}|g(s)|$ with $B:=\left[-\overline{g}_K, \overline{g}_K\right]^{|\cS||\mathcal{A}||\cD|}$. Note that if $y\in B$, then $-\overline{g}_K\le \mathbb{E}\left[ y(s, \pi(s) \right]\le \overline{g}_K$. Hence, $\mathcal{T}^{\phi}_\pi[B]\subseteq B, \,\forall \pi$. By convexity, $B$ is forward invariant for the differential inclusion $\dot{y}\in\Psi(q, y)$. Moreover, any solution to the differential inclusion can be represented almost everywhere as:
        \begin{align}
            \dot{y}(t) = \sum_{\pi\in\Pi(q)} \lambda_{\pi}(t)\left( \mathcal{T}^{\phi}_{\pi}[y(t)] - y(t) \right),\quad \text{with }\ \lambda_{\pi}(t) \ge 0, \sum_{\pi}\lambda_{\pi}(t) = 1
        \end{align}
        Next, let $y(\cdot), z(\cdot)$ be two trajectories with the same weight $\lambda_{\pi}(\cdot)$, and let $e(t) := y(t) - z(t)$. Let $I(t) := \left\{ i:|e_i(t)| = \Vert e(t)\Vert_{\infty}\right\}$ be the coordinate set on which $|e(t)|$ achieves the supremum norm. For the supremum norm, the upper Dini derivative satisfies:
        \begin{subequations}
        \begin{align}
            \mathbb{D}^+\Vert e(t) \Vert_\infty &\le \max_{i\in I(t)} \mathrm{sgn}(e_i(t)) \dot{e}(t) \\
            \implies\mathbb{D}^+\Vert y(t) - z(t) \Vert_\infty &\le -(1-\gamma) \Vert y(t) - z(t) \Vert_\infty \ \ (\text{by $\gamma$-contraction of $\mathcal{T}^\phi_\pi$}) \\
            \therefore \Vert y(t) - z(t) \Vert_{\infty} &\le e^{-(1-\gamma)(t-s)}\Vert y(s) - z(s) \Vert_\infty \ \ (\text{by Gr\"{o}nwall})
        \end{align}
        \end{subequations}
        Thus, choosing $z(\cdot)\in B$ and the fact that $\mathcal{T}^{\phi}_\pi[B]\subseteq B$ shows that every fast timescale trajectory $y(\cdot)$ (driven by the same weights $\lambda_\pi$) for a quasi-static $q$ is attracted to the bounded set $B$.
    \end{itemize}
    \begin{enumerate}
        \item Since $\Psi(q, \cdot)$ is a Marchaud map, the differential inclusion for a quasi-static $q$ admits solutions from every initial condition. As the compact set $B$ is forward invariant, and globally attracting, the limiting dynamics $y(\cdot)$ admit a nonempty compact global attracting set $\Lambda(q)\subseteq B$ \citep{benaim2005stochastic}.
        \item Let $q_m\rightarrow q$ and let $y_m\in\Lambda(q_m)$ with $y_m\rightarrow y$. Note that for sufficiently large $m$, $\Pi(q_m)\subseteq \Pi(q)$, and thus, $\Psi(q_m, \cdot)\subseteq\Psi(q, \cdot)$. Next, for any given $y$, $\{ h_{\pi}(y):\pi\in\Pi(q_m) \} \subseteq \{ h_{\pi}(y):\pi\in\Pi(q) \}$. Since the convex hull preserves set inclusions, $\Psi(q_m, y) \subseteq \Psi(q, y)$ for every $y$, and thus, every complete solution of $q_m$-differential inclusion is also a complete solution of the $q$-differential inclusion. Consequently, $\Lambda(q_m)\subseteq\Lambda(q)$ for sufficiently large $m$. Moreover, since $\Lambda(q)$ is closed, $y\in\Lambda(q)$, and thus $\mathrm{Gr} \Lambda$ is closed. This implies that the map $q\mapsto\Lambda(q)$ is upper semicontinuous.
        \item Finally, under Assumptions \ref{ass:martingale_noise}-\ref{ass:steps} and by the two-timescale stochastic recursive inclusion framework \citep{10.1287/moor.2019.1037, ramaswamy2016stochastic}, the limit set of the coupled iterates satisfies $\omega(Y, Q)\subseteq \mathrm{Gr}\Lambda \ \ \mathrm{a.s.}$ Since the iterates are bounded and $\mathrm{Gr}\Lambda$ is closed, 
        \begin{align}
            \mathrm{dist}\left( (Y_n, Q_n), \mathrm{Gr}\Lambda \right) \rightarrow 0 \quad \mathrm{a.s.,}
        \end{align}
    \end{enumerate}

\end{proof}

\begin{remark}
    (The Quasi-Static Joint Critic) Intuitively, Lemma~\ref{lem:safety_critic_tracking} guarantees that the safety critic learns much faster than the joint critic. This implies that the joint critic is provided with an accurate estimate of its corresponding policy's safety evaluation before it updates its estimate and changes the policy.
\end{remark}

\subsection{Slow Timescale: Averaged Differential Inclusion (Sec.~\ref{subsec:slow_timescale})}\label{app:slow_timescale_proofs}

Section \ref{subsec:fast_timescale} shows that the coupled iterates asymptotically approach the graph of the compact, globally attracting set $\Lambda(q)$ for some quasi-static $q$. By Assumption \ref{ass:steps}, the slow-timescale iterate $Q_n$ remains quasi-static while the fast timescale dynamics evolve. Thus, the increments of $Q_n$ depend on the \textit{long-run average} of $G(Y_n, Q_n)$. These averages are encoded by the occupation measures in $\mathfrak{D}(q)$, leading to the set-valued $\overline{G}(q)$ in Eq. \ref{eq:G_hat_q}. 

\ode*
\begin{proof}
    Let $M_{n+1}:= \hat{\mathcal{T}}_{n+1, Y_n}[Q_n] - \mathcal{T}_{Y_n}[Q_n] + \xi_{n+1}$ be the total martingale-difference error (error due to the stochastic Bellman target and additional perturbation noise). Then, Eq. \ref{eq:slow} can be written as:
    \begin{align}
        Q_{n+1} := Q_n + \alpha_n\left[ G(Y_n, Q_n) + M_{n+1} \right], \quad \text{where } G(y, q):=\mathcal{T}_y[q] - q
    \end{align}
    \begin{enumerate}
        \item In the finite tabular setting, $G$ is continuous in both arguments; $V_q(s) = \max_a \min_d q(s, a, d)$ is a $1$-Lipschitz map in $q$, and $x\mapsto \min{\cdot, x}$ is $1$-Lipschitz map. Consequently, 
        \begin{align}
            \Vert G(y, q) - G(y', q')\Vert_\infty \leq \gamma\Vert y-y' \Vert_\infty + (1+\gamma)\Vert q-q' \Vert_\infty
        \end{align}
        Thus, $G(\cdot, \cdot)$ is continuous and has at most linear growth. 

        \item Next, let $\mathfrak{D}(q)$ denote the admissible limiting occupation measure set, appearing in the two-timescale stochastic recursive inclusion theorem of \citep{10.1287/moor.2019.1037}; these measures account for the long-run fast timescale behavior at a given $q$; the slow-timescale differential inclusion is then formed through probability-measure averaging \citep{10.1287/moor.2019.1037}:
        \begin{align}
            \overline{G}(q) = \left\{ \int_{\Lambda(q)} G(y, q)\mathbb{P}(\dd y) : \mathbb{P}\in\mathfrak{D}(q) \right\}    
        \end{align}
        \item Next, by Assumptions \ref{ass:martingale_noise}-\ref{ass:steps}, the \textit{accumulated noise} on the slow-timescale vanishes on every finite time-window:
        \begin{align}
            \lim_{n\rightarrow\infty} \sup_{m\geq n: \sum_{k=n}^m \alpha_k \le T} \left\Vert \sum_{k=n}^m \alpha_k M_{k+1} \right\Vert_\infty = 0\quad \text{a.s.}
        \end{align}
        for every $T>0$. Hence, by the slow-timescale APT result of \citep{10.1287/moor.2019.1037}, the continuous affine interpolation $\overline{Q}(\cdot)$ is almost surely an asymptotic pseudotrajectory of $\dot{q}(t) \in \overline{G}(q(t))$. Finally, by assumption, the iterates are almost surely bounded, and thus, by the limit-set theorem for bounded asymptotic pseudotrajectories \citep{benaim2005stochastic}, $\mathcal{U}:=\omega(\overline{\mathcal{Q}})$ is nonempty, compact, invariant and internally chain transitive for the differential inclusion $\dot{q} \in \overline{G}(q)$.
    \end{enumerate}
\end{proof}

\subsection{Safe Configuration Set and Forward Invariance (Sec.~\ref{sec:forward_invariance})}\label{app:forward_invariance}

The state, action and disturbance spaces are finite, so the action-value space $\mathcal{Q} = \mathcal{B}(\mathcal{S} \times \mathcal{A} \times \mathcal{D})$ is identified with $\mathbb{R}^N, N = |\mathcal{S} \times \mathcal{A} \times \mathcal{D}|$, and the limiting dynamics are an inclusion in finite dimension. The extended class-$\mathcal{K}_e$ function is written explicitly as $Kg(s)$, since the margin scale affects the bounds. We write $\norm{\cdot}_\infty$ for the supremum-norm. Recalling also that the reward function is $r:\mathcal{S}\times \mathcal{A} \times \mathcal{D} \rightarrow [0, R_{\text{max}}]$ and $\Omega \subseteq \mathcal{F}^c = \{ s : g(s)>0 \}$.

For a fixed critic $y \in \mathcal{Q}$, let $\mathcal{T}_y[q]$ denote the joint Bellman operator of Eq.~\ref{eq:joint_bellman}, and for policy $\pi$ let $\mathcal{T}^\phi_\pi[y]$ denote the safety Bellman operator of Eq.~\ref{eq:safety_bellman}. 

\kernelpreservingactions*
Since $\Omega$ is the set of states from which a control policy exists keeping the agent safe at all times regardless of the adversarial disturbances, taking $a = \mu(s)$ for such a policy shows $\mathcal{A}^\Omega(s) \neq \emptyset \; \forall s \in \Omega$. This is a property of the viability kernel. Since the transition kernel is stochastic, the kernel preservation is \textit{almost sure}, and the safety conclusions are almost surely with respect to the transition randomness.
\indexsetenlargements*
The index sets are determined by the transition kernel and by $\Omega$; they depend on neither $q$ nor any policy. Throughout we take $\kappa >0$ and $\epsilon_0 \in (0, \kappa/2)$. The upper bound on $\epsilon_0$ ensures the two families of constraints remain strictly separated on all of $\mathcal{C}_{\epsilon_0}$. 

We fix a deterministic Borel-measurable selection in Eq.~\ref{eq:greedy_mu_nu}: we choose a single action $\mu_q(s)$ from the maximizers of $\min_d q(s, \cdot, d)$, and a single disturbance $\nu_q(s,\mu_q(s))$ from the minimizers of $q(s, \mu_q(s),\cdot)$. Existence of such selections is guaranteed by the Kuratowski-Ryll-Nardzewski theorem. Fixing the selection makes $q \mapsto \pi_q$ single-valued but not continuous; this is why fast timescale is analyzed through $\Pi(q)$ and $\Lambda(q)$ rather than through a single equilibrium map, and why the slow limit is the inclusion Eq.~\ref{eq:slow_inclusion}.
\begin{align}\label{eq:slow_inclusion}
    \dot{q}(t) \in \overline{G}(q(t))
\end{align}
One might expect a condition of the form $\min_d q(s,a,d) \le -\kappa$ for kernel-exiting $a$; but the disturbance attaining the inner minimum of $q$ need not be the one under which the action leaves the kernel (the truncation bound in Lemma~\ref{lem:target_bounds} is available only for the latter). Fixing $\hat{d}$ removes the mismatch while still yielding $\min_d q(s,a,d) \leq q(s,a,\hat{d}(s,a)) \leq -\kappa$. 

\safecfg*
The intuition is as follows: the lower bound states that the joint action value has not \textit{underestimated} kernel-preserving actions, and the upper bound states that it has not \textit{overestimated} actions that leave the kernel. We now define the distance to the safe configuration set as $\rho(q):= \text{dist}(q, \mathcal{Q}_\text{cfg})$, the supremum-norm distance.

\subsubsection{Geometry of the Safe Configuration Set}\label{subsubsec:geometry_safe_set}
\begin{lemma}[Structure of $\mathcal{C}_\epsilon$]\label{lem:box_geometry}
    (i) $\mathcal{I} \cap \mathcal{J} = \emptyset$ and $\mathcal{Q}_{\text{cfg}} \neq \emptyset$ (ii) Every $\mathcal{C}_\epsilon$ is a closed convex polyhedron. (iii) For every $q$,  
    \begin{equation*}\label{eq:app_rho_form}
        \rho(q) = \max \Big\{ \max_{x \in \mathcal{I}} (-q(x))_+ \;, \; \max_{x \in \mathcal{J}} (q(x) +\kappa)_+\Big\}
    \end{equation*}
    and consequently $\mathcal{C}_\epsilon = \{ q: \rho(q) \leq \epsilon \}$.
\end{lemma}
\begin{proof}
    (i) A triple of $\mathcal{I}$ has $a \in \mathcal{A}^\Omega(s)$ and a triple of $\mathcal{J}$ has $a \notin \mathcal{A}^\Omega(s)$, so the index sets are disjoint. Next, set $q(x) = 0$ on $x \in \mathcal{I}$, $q(x) = - \kappa$ on $x \in \mathcal{J}$ and arbitrary values elsewhere, therefore this point lies in $\mathcal{Q}_{\text{cfg}}$.\\
    (ii) Each condition is a closed half-space of the form $\{q:\pm q(x) \leq \beta\}$, and a finite intersection of half-spaces is a closed convex polyhedron.\\
    (iii) Let $p \in \mathcal{Q}_{\text{cfg}}$. We first check if the distance $\norm{ q - p}_\infty$ is at least the asserted quantity. For $x \in \mathcal{I}$ with $q(x) < 0$ we have $p(x) \ge 0$, hence $\norm{q(x) - p(x)}_\infty \ge -q(x)$; for $x \in \mathcal{J}$ with $q(x) > -\kappa$ we have $p(x) \le -\kappa$, hence $\norm{q(x) - p(x)}_\infty \ge q(x) + \kappa$. Thus $\norm{q - p}_\infty$ is at least the asserted quantity. Conversely, the coordinatewise clip defined by $p(x) = \max(q(x),0)$ on $x \in \mathcal{I}$, $p(x) = \min(q(x),-\kappa)$ on $ x \in \mathcal{J}$ and $p(x) = q(x)$ on the remaining coordinates belongs to $\mathcal{Q}_{\text{cfg}}$ and attains it. Comparison with the defining inequalities of $\mathcal{C}_\epsilon$ gives the last claim.
\end{proof}
The intuition is that $\mathcal{Q}_{\text{cfg}}$ \textit{is a box}: its constraints are coordinate-separated, so its metric projection is a clip and its distance function has a closed form. By Lemma~\ref{lem:box_geometry}(iii), the family $\mathcal{C}_\epsilon$ is precisely the family of enlargements of $\mathcal{Q}_{\text{cfg}}$.

\begin{lemma}[Regularity of the distance]\label{lem:dist_regularity}
    Under the supremum norm, for any convex $C \subseteq \mathbb{R}^N$, the function $\text{dist}(\cdot, C)$ is convex and $1$-Lipschitz.
\end{lemma}
\begin{proof}
    For convexity: Given $q, q' \in \mathcal{Q}$, $\lambda \in [0,1]$ and $\eta >0$, choose $p , p' \in C$ with $\norm{q-p}_\infty \leq \text{dist}(q,C)+\eta$ and similarly for $p'$. By convexity of $C$, $\lambda p + (1-\lambda)p' \in C$. Then, applying to distance function, $\text{dist}(\lambda q + (1-\lambda)q', C) \leq \norm{(\lambda q + (1-\lambda)q') - (\lambda p + (1-\lambda)p')}_\infty = \norm{\lambda (q-p) + (1-\lambda)(q'-p')}_\infty\leq \lambda \text{dist}(q, C) + (1-\lambda)\text{dist}(q', C) +\eta$. Then taking $\eta \rightarrow 0$, gives us the claim of convexity.\\
    For 1-Lipschitz: As $\text{dist}(q,C) \leq \norm{q - p}_\infty \leq \norm{q - q'}_\infty + \norm{q' - p}_\infty \forall p \in C$, taking the infimum over $p$ and exchanging $q$ and $q'$ yields $|\text{dist}(q,C) - \text{dist}(q',C)| \leq \norm{q - q'}_\infty$.
\end{proof}
\begin{lemma}[Convex distance decay]\label{lem:distance_decay}
    Let $C \subseteq\mathbb{R}^N$ be closed, convex and nonempty, $\rho_C: = \text{dist}(\cdot, C)$ be the supremum-norm, $\theta\in [0,1)$, $R \in (0,\infty]$, and $\Phi: \mathbb{R}^N \rightarrow\mathbb{R}^N$ be a set-valued map. Suppose that for every $x$ with $\rho_C(x) < R$, every $v \in \Phi(x)$ has the form $v=z-x$ with $\rho_C(z) \leq \theta \rho_C(x)$. Then every absolutely continuous $x(\cdot)$ with $x(t) \in \Phi(x(t))$ almost everywhere and $\rho_C(x(0)) < R$ satisfies
    \begin{equation*}
        \rho_C(x(t)) \leq \rho_C(x(0))e^{-(1-\theta)t} \quad \forall \; t\geq 0
    \end{equation*}
\end{lemma}
\begin{proof}
    Let $u(t) := \rho_C(x(t))$. Since $\rho_C$ is $1$-Lipschitz (Lemma~\ref{lem:dist_regularity}), $u$ is absolutely continuous and hence differentiable almost everywhere. Let $t^* := \sup \{ \tau \geq 0: u(t) < R \; \forall\; t \leq \tau \}$. Notice that as $u(0)<R$ and $u$ is continuous, $t^* >0$. Consider $t < t^*$ at which $u$ is differentiable and fix $v := \dot{x}(t)$ such that $v = z-x(t)$ with $\rho_C(z)\leq \theta u(t)$. For $h \in (0,1)$, absolute continuity gives $x(t+h) = x(t) + hv +o(h)$. By the Lipschitz property, we get, 
    \[
    u(t+h) \leq \rho_C(x(t)+hv) + o(h) = \rho_C ((1-h)x(t) + hz) +o(h)
    \]
    Convexity of $\rho_C$ gives us,
    \[
    u(t+h) \leq (1-h) u(t) +h \rho_C(z) +o(h) \leq (1-h)u(t) +h \theta u(t) +o(h)
    \]
    Dividing by $h$ and letting $h \to 0$ yields us $\dot{u}(t) \leq- (1-\theta) u(t)$ almost everywhere on $[0, t^*)$, so Gr\"onwall's inequality gives us $u(t) \leq u(0) e^{-(1-\theta)t}$. In particular $u$ is non-increasing, so $u(t) \leq u(0) < R$. If $t^* < \infty$, continuity gives us $u(t^*) \leq u(0) < R$ and hence $u<R$ on neighbourhood of $t^*$, contradicting the definition of $t^*$ as a supremum. Therefore $t^* = \infty$.
\end{proof}
\begin{lemma}[Selection gap]\label{lem:selection_gap}
    Let $0 \leq \epsilon \leq \epsilon_0$ and $q \in \mathcal{C}_\epsilon$. Then,
    \[
    \max_{a \in \mathcal{A}^\Omega(s)} \min_d q(s,a,d) - \max_{a \notin \mathcal{A}^\Omega(s)} \min_d q(s,a,d) \geq \kappa - 2\epsilon > 0 \; \forall \; s \in \Omega
    \]
    and consequently every maximin greedy policy $\pi =(\mu, \nu) \in \Pi_q$ satisfies $\mu(s) \in \mathcal{A}^\Omega(s)$. For $q \in \mathcal{Q}_\text{cfg}$, the gap is atleast $\kappa$.
\end{lemma}
\begin{proof}
    Noting that $\mathcal{A}^\Omega(s) \neq \emptyset \; \forall \; s \in \Omega$, so the first maximin is not over an empty set. For $a \in \mathcal{A}^\Omega(s)$, we have $(s,a,d) \in \mathcal{I} \; \forall \; d \in \mathcal{D}$, by definitions of kernel-preserving action set and index set. Therefore, $\min_d q(s,a,d) \geq -\epsilon$. We now consider two cases for the second maximin
    \begin{itemize}
        \item If $\mathcal{A}^\Omega(s) = \mathcal{A}$, then the second maximin is over an empty set and so $\max_{a \notin \mathcal{A}^\Omega(s)} \min_d q(s,a,d) = -\infty$ by convention. Hence, the claims of the lemma immediately follow.
        \item If $\mathcal{A}^\Omega(s) \neq \mathcal{A}$, for $a \notin \mathcal{A}^\Omega(s)$, the triple $(s,a,\hat{d}(s,a)) \in \mathcal{J}$ by definition of $\mathcal{J}$. Hence, $\min_d q(s,a,d) \leq q(s,a,\hat{d}(s,a)) \leq -\kappa +\epsilon$ and the gap is therefore at least $\kappa - 2\epsilon$. Noting that this selection gap of $\kappa - 2\epsilon$ is strictly positive since $\epsilon< \epsilon_0 < \kappa/2$ completes the proof.
    \end{itemize}
\end{proof}
Lemma~\ref{lem:selection_gap} depends only on the coordinate inequalities definition $\mathcal{C}_\epsilon$, and does not depend on any property of the critic. This is what allows us to use this in the proof of forward invariance without creating circular dependency. Note that this is stated for the whole family $\mathcal{C}_\epsilon$ and $\forall\; \pi \in \Pi_q$ because the lemma applies to arbitrary greedy branch and at configurations in $\mathcal{C}_{\epsilon_0}$ rather than only in $\mathcal{C}_0$. A further consequence is that on $\mathcal{Q}_\text{cfg}$ no tie between kernel-preserving and kernel-exiting action can arise, so choice of tie-breaking rule in the selection convention does not affect the following safety conclusions.

\subsubsection{Safety Certification of the fast attractor}
As the coupled iterated approach the graph of $\Lambda(\cdot)$, $\mathrm{Gr}\,\Lambda$, rather than a single critic, safety certification (both positive and negative) must hold throughout $\Lambda(q)$, uniformly over $q \in \mathcal{C}_{\epsilon_0}$. For $\sigma >0$, we define the certificate as,
\[
\mathcal{Y}_\sigma = \{ y \in \mathcal{Q}: y(x) \geq 0 \; \forall\; x \in \mathcal{I} \; , \; y(x) \leq -\sigma \; \forall\; x \in \mathcal{J} \},
\]
a closed convex polyhedron by Lemma~\ref{lem:box_geometry}(ii).

Only the negative half of the certificate that constrains the safety critic on kernel-exiting triples is assumed as follows.
\attractorcertification*
The corresponding condition on $\mathcal{I}$ is not assumed because it is preserved by every branch of the fast inclusion (proved in the following lemma and corollary) and therefore the certificate holds throughout $\Lambda(q)$.
\begin{lemma}[Branchwise absorption of the positive constraints]\label{lem:branchwise_absorption}
    Let $q \in \mathcal{C}_{\epsilon_0}$ and $\pi = (\mu, \nu) \in \Pi(q)$. If $y(x) \geq 0 \; \forall \; x \in \mathcal{I}$, then $\mathcal{T}^\phi_\pi[y](x) \geq 0 \; \forall \; x \in \mathcal{I}$. 
\end{lemma}
\begin{proof}
    Let $x = (s,a,d) \in \mathcal{I}$. Since $s \in \Omega \subseteq \mathcal{F}^c$, we have $Kg(s) > 0$ by definition of the scaled safety margin function. As $a \in \mathcal{A}^\Omega(s)$, every successor state $s'\in \Omega \; \forall \;d \in \mathcal{D}$. By Lemma~\ref{lem:selection_gap}, $\mu(s') \in \mathcal{A}^\Omega(s') \implies (s', \mu(s'), \nu(s', \mu(s'))) \in \mathcal{I} \implies y(s', \mu(s'), \nu(s', \mu(s'))) \geq 0 $. Hence, $\mathbb{E}_{s'}[y(\cdot)] \geq 0$. Now, substituting these inequalities in the safety Bellman equation, Eq.~\ref{eq:safety_bellman}, we get
    \[
    \mathcal{T}^\phi_\pi[y](x) = (1-\gamma) Kg(s) +\gamma \min\{ Kg(s), \mathbb{E}_{s'}[y(\cdot)] \} \geq 0
    \]
\end{proof}
\begin{corollary}[Positive certification]\label{cor:positive_certification}
    For every $q \in \mathcal{C}_{\epsilon_0}$ and every $y \in \Lambda_1(q)$, $y(x) \geq 0 \; \forall \; x \in \mathcal{I}$. Together with Assumption~\ref{ass:attractor_certication} this gives $\Lambda_1(q) \subseteq \mathcal{Y}_{\sigma_\Lambda}$.
\end{corollary}
\begin{proof}
Let $\mathcal{P} := \{ y:y(x) \geq 0 \; \forall \; x \in \mathcal{I} \}$, a closed convex polyhedron, and let $\rho_\mathcal{P}$ denote the supremum-norm distance to $\mathcal{P}$. 
\begin{enumerate}
    \item As in Lemma~\ref{lem:box_geometry}, $\rho_\mathcal{P}(y) = \max_{x \in \mathcal{I}} (-y(x))_+$. If $y$ violates the constraints of $\mathcal{P}$ by at most $\epsilon \geq 0$, then $\mathcal{T}^\phi_\pi[y](x) \geq \gamma \min\{Kg(s), -\epsilon\} \geq -\gamma \epsilon \; \forall \; x \in \mathcal{I}, \; \pi \in \Pi(q)$ as in Lemma~\ref{lem:branchwise_absorption}. This implies $\rho_\mathcal{P}(\mathcal{T}^\phi_\pi[y]) \leq \gamma \rho_\mathcal{P}(y)$.
    \item Next, every element of $\Psi(q,y)$ has the form $w - y$ with $\rho_\mathcal{P}(w) \leq \gamma \rho_\mathcal{P}(y)$, because of convexity of the relaxation of $\mathcal{P}$. Lemma~\ref{lem:distance_decay} therefore applies, with $C = \mathcal{P}$, $\theta = \gamma$, and $R = \infty$, at every $y$.
    \item For $y \in \Lambda_1(q)$, by Eq.~\ref{eq:set_valued_fast_tracking} there is a bounded complete solution $y(\cdot)$ of $\dot{y} \in \Psi(q,y)$ with $y(0) = y$, and by shift invariance $y(s) \in \Lambda_1(q) \; \forall \; s \in \mathbb{R}$.
    \item Noting that $\Lambda_1(q)$ is compact and $\rho_\mathcal{P}$ is continuous, we get $M:= \max_{w \in \Lambda_1(q)} \rho_\mathcal{P}(w)$ is finite and $\rho_\mathcal{P}(y(-t)) \leq M \; \forall \; t>0$.
    \item Applying Lemma~\ref{lem:distance_decay} on interval $[-t, 0]$ with initial point of $y(-t)$, we get,
    \[
    \rho_\mathcal{P} (y(0)) \leq \rho_\mathcal{P} (y(-t))\, e^{-(1-\gamma)t} \; \forall \; t>0
    \]
    \item Taking the limit of $t \to \infty$, we get, $\rho_\mathcal{P}(y(0)) = 0$ and $\mathcal{P}$ is closed. Hence, $y \in \mathcal{P}$.
\end{enumerate} 
\end{proof}
\begin{remark}[On the selection of $\hat{d}$ and index sets]  The selection $\hat{d}$ is unconstrained, and can be made to satisfy Assumption~\ref{ass:attractor_certication}; this assumption is a condition on the \textit{existence of a suitable selection}; however, it is not a condition on every kernel-exiting disturbance. Enlarging $\mathcal{J}$ to contain every kernel-exiting triple would shrink $\mathcal{Q}_{\text{cfg}}$ and strengthen Assumption~\ref{ass:attractor_certication} without improving any conclusion, since Lemma~\ref{lem:selection_gap} requires only $\min_d q(s,a,d) \leq -\kappa$, which a single coordinate already delivers.
\end{remark}
\begin{remark}[Only the condition on $\mathcal{J}$ is assumed]\label{rem:asymmetry}
    We now want to reason whether the Corollary~\ref{cor:positive_certification} extends to $\mathcal{J}$. Towards this end, let us first consider $x = (s,a,\hat{d}(s,a)) \in \mathcal{J}$, where $s \in \Omega$ by definition of index set. This means $Kg(s)>0$ and hence substituting in Eq.~\ref{eq:safety_bellman}, we get $(1-\gamma)Kg(s) >0$. If $\mathcal{T}^\phi_\pi[y](x) \leq -\sigma$, then $\mathbb{E}_{s'}[y(\cdot)] \leq Kg(s)$ and $\mathbb{E}_{s'}[y(\cdot)] \leq \frac{-\sigma + (1-\gamma)Kg(s)}{\gamma}$. This second term of minimum in Eq.~\ref{eq:safety_bellman} is controlled by the values of $y(x')$, where $x' =(s',\cdot, \cdot)$ and $s' \notin \Omega$, at successor states outside of $\Omega$. Note that the index sets, $\mathcal{I}$ and $\mathcal{J}$ do not provide any information on the value of $y$ at such successor states outside of $\Omega$. Hence, we get a structural asymmetry: positive certification propagates since kernel-preserving action keeps every successor in $\Omega$ where we have the constraints, whereas negative certification must propagate into a region where nothing has been assumed. Certifying the condition on $J$ therefore requires the Assumption~\ref{ass:attractor_certication} and Corollary~\ref{cor:positive_certification} cannot extend to $\mathcal{J}$. 
\end{remark}
\begin{proposition}[Margin Separation]\label{prop:margin_separation}
    Let Assumption~\ref{ass:attractor_certication} hold and set $\Theta := ((1-\gamma)R_{\text{max}} + \kappa)/\gamma$. Then, for every
    \par \vspace{-2\parskip} {\normalsize\begin{equation}
        K \ge \underline{K} := \frac{(1-\gamma) R_{\text{max}} + \kappa}{\gamma \sigma_\Lambda}
    \end{equation}}
    one has $y(x) \le -\Theta$ for all $x \in \mathcal{J}$, all $y \in \Lambda(q)$, and all $q \in \mathcal{C}_{\epsilon_0}$. If $\mathcal{J} = \emptyset$ the bound on $K$ and the conclusion holds trivially. 
\end{proposition}
\begin{proof}
    \begin{enumerate}
        \item The safety Bellman operator is positively homogeneous of degree one in the margin- replacing $g$ by $Kg(s)$ and $Q^\phi$ by $KQ^\phi$ leaves the Eq.~\ref{eq:safety_bellman} invariant for all $K>0$.
        \item $\Pi(q)$ is greedy with respect to $q$ and is therefore independent of the positive margin scale.
        \item Noting that positive scaling commutes with the closed convex hull and graph completion, $\Psi_K(q, Ky) = K\Psi_1(q,y)$.
        \item $y(\cdot)$ is a bounded complete solution at unit scale if and only if $Ky(\cdot)$ is a complete solution at positive scale $K$.
        \item Taking time zero values in Eq.~\ref{eq:set_valued_fast_tracking}, gives us $\Lambda_K(q) = K \Lambda_1(q)$.
        \item By Assumption~\ref{ass:attractor_certication}, every $y_1 \in \Lambda_1(q)$ satisfies $y_1(x) \leq -\sigma_\Lambda \forall x \in \mathcal{J}$. 
        \item Similarly, for every $y = K y_1 \in \Lambda_K(q)$ satisfies $y(x) \leq -K \sigma \leq - \Theta$ whenever $K \sigma \geq \Theta$, which is the stated bound.
    \end{enumerate}
\end{proof}
Note only the Assumption~\ref{ass:attractor_certication} interacts with $K$. The condition on $\mathcal{I}$ obtained in Corollary~\ref{cor:positive_certification} is positive scale-invariant, that is, multiplying by $K>0$ does not change the sign.
\subsubsection{
Uniform violation contraction}
\begin{lemma}[Target Bounds]\label{lem:target_bounds}
    Let $0 \le \epsilon \leq \epsilon_0$, $q \in \mathcal{C}_\epsilon$, and $y \in \Lambda(q)$. Then (i) $T_y[q](x) \leq -\kappa$ for every $x \in \mathcal{J}$, and (ii) $T_y[q](x) \geq -\gamma\epsilon$ for every $x \in \mathcal{I}$.
\end{lemma}
\begin{proof}
(i) Let $x = (s,a,d) \in \mathcal{J}$. Bounding each $\min$ in Eq.~\ref{eq:joint_bellman} by one of its arguments and using $m(s,a,d) = \min\{r, Kg(s)\} \leq r \leq R_{\text{max}}$,
\[
T_y[q](x) = (1-\gamma) m(s,a,d) + \gamma \min\{ r + \mathbb{E}[V_q(s')], y(x) \} \leq (1-\gamma)R_{\text{max}} + \gamma y(x).
\]
Since $q \in \mathcal{C}_{\epsilon_0}$ and $y \in \Lambda(q)$, Proposition~\ref{prop:margin_separation} gives $y(x) \leq -\Theta$, so $T_y[q](x) \leq (1-\gamma)R_{\text{max}} - \gamma\Theta = -\kappa$ (again using definition of $\Theta$ as defined in Proposition~\ref{prop:margin_separation}). The bound depends only on $R_{\text{max}}$, $\gamma$ and $\Theta$, and is therefore uniform in $q$ and in $y \in \Lambda(q)$.\\
(ii) Let $x = (s,a,d) \in \mathcal{I}$. Since $s \in \Omega \subseteq F^c$ we have $g(s) > 0$, and $r \ge 0$, so $m(s,a,d) \ge 0$ and the first term of Eq.~\ref{eq:joint_bellman} is non-negative. For the continuation term, $a \in \mathcal{A}^\Omega(s)$, so every successor $s'$ lies in $\Omega$ -  against every disturbance, not only the one appearing in $x$. As $s' \in \Omega$, the set $\mathcal{A}^\Omega(s') \neq \emptyset$; fix $a^+ \in \mathcal{A}^\Omega(s')$. Since $\mathcal{I}$ contains every disturbance coordinate, $(s',a^+,d') \in \mathcal{I}$ for all $d' \in D$, and $q \in \mathcal{C}_\epsilon$ gives $q(s',a^+,d') \geq -\epsilon$ for each such $d'$. Hence
\[
V_q(s') = \max_{a'} \min_{d'} q(s',a',d') \geq \min_{d'} q(s',a^+,d') \geq -\epsilon
\]
the first inequality because a maximum over $a'$ is at least the value at the particular action $a^+$. Therefore $r + E[V_q(s')] \geq -\epsilon$, while $y(x) \geq 0$ by Corollary~\ref{cor:positive_certification} on $\mathcal{I}$. Since $\min\{X,Y\} \geq \min\{x,y\}$ whenever $X \geq x$ and $Y \geq y$,
\[
T_y[q](x) \ge (1-\gamma) \cdot 0 + \gamma \cdot \min\{-\epsilon, 0\} = -\gamma\epsilon
\]
\end{proof}
Part (ii) of the above proof never identifies which action the greedy policy of Eq.~\ref{eq:greedy_mu_nu} selects at $s'$, nor which disturbance is applied there: it requires only that some kernel-preserving action exist at $s'$ and carry value at least $-\epsilon$. This is why the lower constraints propagate without any hypothesis on the disturbance policy, whereas the upper constraints require Assumption~\ref{ass:attractor_certication} on $\mathcal{J}$.
\begin{corollary}[The joint operator contracts the violation]\label{cor:violation_contraction}
    For $0 \leq \epsilon \leq \epsilon_0$ and every $y \in \Lambda(q)$, $T_y[\mathcal{C}_\epsilon] \subseteq \mathcal{C}_{\gamma\epsilon}$.
\end{corollary}
\begin{proof}
    Let $q \in \mathcal{C}_\epsilon$ and $y \in \Lambda(q)$. Lemma~\ref{lem:target_bounds}(ii) gives $T_y[q](x) \ge -\gamma\epsilon$ for $x \in \mathcal{I}$, and Lemma~\ref{lem:target_bounds}(i) gives $T_y[q](x) \le -\kappa \le -\kappa + \gamma\epsilon$ for $x \in \mathcal{J}$. These are precisely the defining inequalities of $\mathcal{C}_{\gamma\epsilon}$.
\end{proof}
This is the property on which the invariance argument depends on: the joint Bellman operator maps each relaxed constraint set into a smaller one, uniformly over the critics in the fast attractor. It is \emph{not} a consequence of Proposition~\ref{prop:joint_bellman_contraction}, which holds the safety action-value fixed. Here, the critic ranges over $\Lambda(q)$ while $q$ itself moves and the composed map $q \mapsto T_{\mathbf{Y}_q}[q]$ need not induce a contraction. Corollary~\ref{cor:violation_contraction} asserts the weaker but sufficient property that the operator contracts the \emph{distance} to $\mathcal{Q}_{\text{cfg}}$.

We now proceed to transfer the property from the individual operators to their \textit{occupation-measure average} in the next section of the proof.

\subsubsection{Averaging over the fast attractor preserves the violation contraction}
\begin{lemma}[Averaging preserves the violation contraction]\label{lem:averaged_contraction}
    Let $\rho(q) = \epsilon < \epsilon_0$. Then every $v \in \overline{G}(q)$ has the form $v = z -q$ with $z \in \mathcal{C}_{\gamma \epsilon}$.
\end{lemma}
\begin{proof}
\begin{enumerate}
    \item By Eq.~\ref{eq:G_hat_q}, there is a probability measure, $\mathbb{P} \in \mathfrak{D}(q)$ supported on $\Lambda(q)$ with
    \[
    v = \int_{\Lambda(q)} (T_y[q] -q) \mathbb{P}(dy) = \int_{\Lambda(q)} T_y[q] \mathbb{P}(dy) - \int_{\Lambda(q)} q \mathbb{P}(dy) = \left(\int_{\Lambda(q)} T_y[q] \mathbb{P}(dy)\right)- q
    \]
    \item Comparing this to $v = z-q$, gives us that $z = \int_{\Lambda(q)} T_y[q] \mathbb{P}(dy)$. In words, $z$ is the averaged target.
    \item From Corollary~\ref{cor:violation_contraction}, we know $T_y[q] \in \mathcal{C}_{\gamma \epsilon}$ for every $y \in \Lambda(q)$.
    \item From Lemma~\ref{lem:box_geometry}(ii), we know that $\mathcal{C}_{\gamma \epsilon}$ is closed and convex. By extension, the averaged target $z \in \mathcal{C}_{\gamma \epsilon}$.
\end{enumerate} 
\end{proof}
Averaging over occupation measures still preserves the violation contraction because of the convexity of $\mathcal{C}_\epsilon$, since $\text{conv}(\mathcal{C}_\epsilon) = \mathcal{C}_\epsilon$.
\subsubsection{Forward Invariance}
\forwardinvariance*
\begin{proof}
    \begin{enumerate}
        \item Solutions of Eq.~\ref{eq:slow_inclusion} are by definition absolutely continuous functions of time satisfying $\dot q(t) \in \overline{G}(q(t))$ for almost every $t$.
        \item Such solutions exist because $\overline{G}$ is upper semicontinuous with nonempty compact convex values, $\Lambda$ is nonempty, compact and upper semicontinuous by Lemma~\ref{lem:safety_critic_tracking} and the barycentre map preserves these properties.
        \item By Lemma~\ref{lem:averaged_contraction}, for every $q$ with $\rho(q) < \epsilon_0$ each $v \in \overline{G}(q)$ has the form $v = z - q$ with $z \in \mathcal{C}_{\gamma\rho(q)}$.
        \item By Lemma~\ref{lem:box_geometry}(iii), $\rho(z) \le \gamma\rho(q)$.
        \item The hypothesis of Lemma~\ref{lem:distance_decay} therefore holds with $C = \mathcal{Q}_{\text{cfg}}$, $\theta = \gamma$ and $R = \epsilon_0$. Since $\rho(q(0)) < \epsilon_0$, applying Lemma~\ref{lem:distance_decay}, yielding:
        \par \vspace{-2\parskip} {\normalsize\begin{equation}\label{eq:rho_decay}
            \rho(q(t)) \;\le\; \rho(q(0))\, e^{-(1-\gamma)t} \quad \forall \; t \ge 0.
        \end{equation}}
        \item If $q(0) \in \mathcal{Q}_{\text{cfg}}$ then $\rho(q(0)) = 0$, so $\rho \equiv 0$ and $q(t) \in \mathcal{Q}_{\text{cfg}}\; \forall\;t$ by Lemma~\ref{lem:box_geometry}(iii). The exponential attractivity statement is the bound in Eq.~\ref{eq:rho_decay}.
    \end{enumerate}
\end{proof}
\subsubsection{The safe configuration set certifies safety}
\begin{lemma}[Robust invariance of $\Omega$]\label{lem:invariance_omega}
    Let $q \in \mathcal{Q}_{\text{cfg}}$. Then (i) $\inf_{s\in\Omega} V^\phi_{\pi_q}(s) \ge 0$; and (ii) for every $s_0 \in \Omega$ and every disturbance sequence in $\mathcal{D}$, the state trajectory generated by $\mu_q$ satisfies $\mathbb{P}(s_t \in \Omega \; \forall \; t \ge 0) = 1$.
\end{lemma}
\begin{proof} We will first prove Part(i) below.
\begin{enumerate}
    \item The fixed point $Q^\phi_{\pi_q}$ of the safety Bellman operator $\mathcal{T}^\phi_{\pi_q}$ satisfies $\mathcal{T}^\phi_{\pi_q}[Q^\phi_{\pi_q}] - Q^\phi_{\pi_q} = 0 \in \Psi(q, Q^\phi_{\pi_q})$ by Eq.~\ref{eq:set_valued_fast_tracking}.
    \item So $Q^\phi_{\pi_q}$ is an equilibrium of the frozen inclusion. Therefore, by the second consequence of Eq.~\ref{eq:set_valued_fast_tracking}, $Q^\phi_{\pi_q} \in \Lambda(q)$.
    \item Now, applying Corollary~\ref{cor:positive_certification}, gives us $Q^\phi_{\pi_q} (x) \geq 0 \; \forall \; x \in \mathcal{I}$.
    \item By Lemma~\ref{lem:selection_gap}, $\mu_q(s) \in \mathcal{A}^\Omega(s)$ and therefore $(s,\mu_q(s), \nu_q(s, \mu_q(s))) \in \mathcal{I} \; \forall \;\nu_q(s, \mu_q(s)) \in \mathcal{D}$.
    \item Noting that $\mu_q$ and $\nu_q$ are deterministic, the maximum and minimum in the definition of the safety value function is over singletons, giving us
    \[
    V^\phi_{\pi_q} = \max_{a \in \mu_q(s)} \min_{d \in \nu_q(s, \mu_q(s))} Q^\phi_{\pi_q}(s,a,d) = Q^\phi_{\pi_q}(s,\mu_q(s),\nu_q(s, \mu_q(s))) \geq 0
    \]
    \item As the above reasoning holds $\forall \; s \in \Omega$, taking infimum over $\Omega$, yields, $\inf_{s \in \Omega}V^\phi_{\pi_q} \geq 0$.
\end{enumerate} 
We will next prove Part(ii). Let $E_t := \{ s_t \in \Omega \}$. We show $\mathbb{P}(E_t) = 1$ for every $t$ by induction.
\begin{enumerate}
    \item $\mathbb{P}(E_0) = 1$ by hypothesis.
    \item Suppose $\mathbb{P}(E_t) = 1$. On $E_t$, Lemma~\ref{lem:selection_gap} gives $\mu_q(s_t) \in \mathcal{A}^\Omega(s_t)$, so by the definition of $\mathcal{A}^\Omega$, $\mathbb{P}(s_{t+1} \in \Omega \mid s_t, \mu_q(s_t), d) = 1\; \forall \;d \in \mathcal{D}$; taking expectations over $E_t$ yields $\mathbb{P}(E_{t+1}) = 1$, whichever disturbance is applied.
    \item Since the time index is discrete, $\{ s_t \in \Omega \; \forall\; t \ge 0 \} = \bigcap_{t \ge 0} E_t$ is a countable intersection of probability-one events and therefore has probability one.
\end{enumerate}
\end{proof}
\inducedsafety*
\begin{proof}
    From Theorem~\ref{thm:forward_invariance}, we get $q(t) \in \mathcal{Q}_\text{cfg} \; \forall \; t\geq 0$ because $q(0) \in \mathcal{Q}_\text{cfg}$. Finally, applying Lemma~\ref{lem:invariance_omega} at $q= q(t)$, both assertions follow.
\end{proof}
\begin{remark}[Scope]
    Corollary~\ref{cor:induced_policy_safety} concerns the trajectory generated by the policy induced at a given configuration. Safety throughout training would additionally require safe initialization together with a safe exploration mechanism, neither of which is assumed here.
\end{remark}
\begin{remark}[Intuition for the Geometric Argument]
    The drift points from the current configuration toward an average of its Bellman images under the critics in $\Lambda(q)$, each of which lies closer to $\mathcal{Q}_{\text{cfg}}$ by the factor $\gamma$. A step along the drift is therefore a chord from $q$ into $\mathcal{C}_{\gamma\rho(q)}$, and convexity of the distance function converts this into decay of $\rho$. Forward invariance is thus a statement about convex geometry.
\end{remark}
\subsection{Convergence to Optimal Safety-Constrained Task Performance (Sec.~\ref{sec:convergence})}\label{app:convergence_proofs}
\optjointvaluefunc*

Let $Q^{\phi, 1}_{\pi}$ denote the safety critic obtained by scaling the margin $g(\cdot)$ by $K=1$. By homogeneity, $Q^{\phi}_{\pi}:=Q^{\phi, K}_{\pi}=KQ^{\phi, 1}_{\pi}$. Next, let $\delta_{\Omega} := \inf_{q\in\mathcal{Q}_{\mathrm{cfg}}}\min_{\pi\in\Pi(q)}\min_{s\in\Omega} V^{\phi, 1}_{\pi} > 0$ represent the lowest safety value (represents the \textit{closest} a safe policy ever comes to $\partial\Omega$).

\convergencefixedpt*

\begin{proof}
    For any $q_0\in\mathcal{U}$, since $\mathcal{U}$ is compact and $\mathcal{U}\subseteq \{q:\rho(q)<\epsilon_0\}$, applying Theorem~\ref{thm:forward_invariance} yields $\mathcal{U}\subseteq\mathcal{Q}_{\mathrm{cfg}}$. Then, $\mathcal{U}\subseteq\mathcal{Q}_{\mathrm{cfg}}$ is a compact invariant set of the averaged slow-timescale differential inclusion $\dot{q}(t) \in \overline{G}(q(t))$. We will first prove that $\mathcal{U}\subseteq\mathcal{Q}^*$.
    \begin{enumerate}
        \item By definition of $\delta_\Omega$, for every $q\in\mathcal{Q}_{\mathrm{cfg}}$, every $\pi\in\Pi(q)$, and every $s\in\Omega$, we have $V_\pi^{\phi, 1}(s) \ge \delta_\Omega$, and thus, $g(s) \ge \delta_{\Omega}$ for $s\in\Omega$. We will now show that every $y\in\Lambda(q)$ satisfies $y(x_i) \ge K\delta_\Omega$ for $x_i\in\mathcal{I}$.

        Define $\zeta(y):= \max_{x_i\in\mathcal{I}} \left( K\delta_\Omega - y(x_i) \right)^+$ which, by definition gives $y(x_i) \ge K\delta_\Omega - \zeta(y)$. Next, note that, since $x_i=(s, a, d) \in\mathcal{I}$, by definition, $s\in \Omega$, and $a\in A^{\Omega}(s)$, which implies that $P(s'\in\Omega\mid s, a, d) = 1$. Thus, almost surely, $\mathbb{E}[y(s', \pi(s'))] \ge K\delta_\Omega - \zeta(y)$. Applying the above facts to the safety Bellman operator yields:
        \begin{subequations}
        \begin{align}
            \mathcal{T}^\phi_{\pi}[y](x_i) &= (1-\gamma)Kg(s) + \gamma\min\left\{ Kg(s), \mathbb{E}[y(s', \pi(s'))] \right\} \\
            & \ge (1-\gamma)Kg(s) + \gamma\left( K\delta_\Omega - \zeta(y) \right) \\
            &= K\delta_\Omega - \gamma\zeta(y)
        \end{align}
        \end{subequations}
        Since this holds for all policies $\pi\in\Pi(q)$, every convex combination has the same lower bound. Next, under Assumption \ref{ass:steps}, the induced limiting fast timescale dynamics is given by $\dot{y}(x_i) = \mathcal{T}^\phi_\pi[y](x_i) - y(x_i)$. Plugging the above result yields $\dot{y}(x_i) \ge (1-\gamma)\zeta(y)$. Next, let $E(t):=\{x_i\in\mathcal{I}: x_i = \argmax_{x_i\in\mathcal{I}} K\delta_\Omega - y(x_i)(t)\}$. On these coordinates, the upper Dini derivative satisfies:
        \begin{subequations}
        \begin{align}
            \mathbb{D}^+\zeta(y(t)) &\le \max_{i\in E(t)} \mathbb{D}^+ \left( K\delta_\Omega - y(x_i)(t) \right)^+ \\
            &\le -(1-\gamma)\zeta(y(t)) \label{eq:dini_zeta_y}\\
            \therefore \,\, \zeta(y(t))&\le e^{-(1-\gamma)(t-s)}\zeta(y(s)), \,\, t\ge s \quad \text{(By Gr\"{o}nwall)}
        \end{align}
        \end{subequations}
        For contradiction, assume $\zeta(y(0)) > 0$, which yields $\zeta(y(s)) \ge e^{(1-\gamma)(0-s)}\zeta(y(0))$. As $s\rightarrow -\infty$ (since $s<t$), $\zeta(y(s)) \rightarrow\infty$ which is a contradiction, since $y(\cdot)$ is bounded, and thus, $\zeta(y(\cdot))$ must also be bounded. Thus, $\zeta(y(0)) = 0$, and since $y(0)$ is an arbitrary point of $\Lambda(q)$, we must have $y(x_i)\ge K\delta_\Omega$, for all $x_i\in\mathcal{I}$ and for all $y\in\Lambda(q)$.

        \item Next, let $\overline{V}:=\frac{R_{\text{max}}}{1-\gamma}$ be the maximum task return possible. For a quasi-static $q\in\mathcal{U}, y\in\Lambda(q)$ and $x_i=(s, a, d)\in\mathcal{I}$, since $K^* \ge K$, we have:
        \begin{align}
            Kg(s) \ge K\delta_\Omega \ge R_{\text{max}} + \overline{V} \ge R_{\text{max}} \ge r(x_i) \label{eq:safety_trunc_inactive}
        \end{align}
        which implies that the safety truncation is inactive. Moreover, we showed above that $y(x_i) \ge K\delta_\Omega$, and thus, the joint Bellman operator of Eq. \ref{eq:joint_bellman} reduces to:
        \begin{align}
            \mathcal{T}_y[q](x_i) = r(x_i) + \gamma\mathbb{E}\left[ V_q(s') \right]
        \end{align}
        Since $a\in A^\Omega(s)$, $s'\in\Omega$ almost surely; moreover, since $q\in\mathcal{Q}_{\text{cfg}}$ implies that every maximin greedy controller action at $s'\in\Omega$ is kernel preserving, we have:
        \begin{align}
            V_q(s') = \max_{a\in A^\Omega(s)}\min_{d'\in\mathcal{D}} q(s', a', d')
        \end{align}
        thus, yielding Eq. \ref{eq:safety_constr_task_bellman}. Consequently, with $\mathfrak{q}=q\vert_{\mathcal{I}}$, we have: 
        \begin{align}
            G(y,q)\vert_{\mathcal{I}} &= \mathcal{T}_\Omega[\mathfrak{q}] - \mathfrak{q} \quad \forall y\in\Lambda(q) \\
            \hspace{-4cm}\text{Averaging over any }\mathbb{P}\in\mathfrak{D}(q) &\text{yields:} \\
            \overline{G}(q)\vert_\mathcal{I} &= \left\{ \mathcal{T}_\Omega[q\vert_\mathcal{I}] - q\vert_{\mathcal{I}} \right\}, \quad q\in\mathcal{U} \label{eq:overline_G_sigleton}
        \end{align}
        \item Next, note that the operator $\mathcal{T}_\Omega$ is a $\gamma$-contraction in the supremum norm:
        \begin{align}
            \left\Vert \mathcal{T}_\Omega[\mathfrak{q}] - \mathcal{T}_\Omega[\mathfrak{q'}] \right\Vert_\infty \leq \gamma\left\Vert \mathfrak{q} - \mathfrak{q}' \right\Vert_\infty
        \end{align}
        and thus, possesses a unique fixed point $Q^*_\Omega$. Moreover, let $q(\cdot)$ be any complete solution contained in $\mathcal{U}$. Let $\mathfrak{q}(t):=q(t)\vert_{\mathcal{I}}$. Then, by Eq. \ref{eq:overline_G_sigleton}, $\dot{\mathfrak{q}}(t) = \mathcal{T}_{\Omega}[\mathfrak{q}(t)] - \mathfrak{q}(t)$. Now, since $\mathcal{T}_\Omega[Q^*_{\Omega}] = Q^*_{\Omega}$, its contraction yields:
        \begin{align}
            \mathbb{D}^+\left\Vert \mathfrak{q}(t) - Q^*_{\Omega} \right\Vert_\infty \le -(1-\gamma)\left\Vert \mathfrak{q}(t) - Q^*_{\Omega} \right\Vert_\infty \label{eq:dini_mathfrak_q_Q_star}
        \end{align}
        If $\mathfrak{q}(0)\ne Q^*_\Omega$, then by Gr\"{o}nwall:
        \begin{align}
            \left\Vert \mathfrak{q}(0) - Q^*_\Omega \right\Vert_\infty \le e^{-(1-\gamma)(0-s)}\left\Vert \mathfrak{q}(s) - Q^*_\Omega \right\Vert_\infty, \quad s<t=0
        \end{align}
        which contradicts the compactness of $\mathcal{U}$. Thus, $q\vert_{\mathcal{I}} = Q^*_\Omega$ for all $q\in\mathcal{U}$. Moreover, since $\mathcal{U}\subseteq \mathcal{Q}_{\mathrm{cfg}}$, we have
        \begin{align}
            \mathcal{U} \subseteq \left\{ q\in\mathcal{Q}_{\mathrm{cfg}} : q\vert_\mathcal{I} = Q^*_\Omega \right\} = \mathcal{Q}^*
        \end{align}
        \item Finally, by Theorem \ref{thm:ode}, the slow-timescale continuous affine interpolation $\overline{Q}(\cdot)$ has a nonempty compact invariant internally chain-transitive $\omega$-limit set. Moreover, by Theorem \ref{thm:forward_invariance}, this limit set lies in $\mathcal{Q}_\mathrm{cfg}$. Hence, by the previous argument, we have almost surely $\omega(\overline{Q}) \subseteq \mathcal{Q}^*$. Since the iterates $\{Q_n\}$ are almost surely bounded, and $\mathcal{Q}^*$ is closed, every subsequential limit of $\{Q_n\}$ belongs to $\mathcal{Q}^*$. Hence, 
        \begin{align}
            \mathrm{dist}\left( Q_n, \mathcal{Q}^* \right) \rightarrow 0\qquad\mathrm{a.s.}
        \end{align}
    \end{enumerate}
    Thus, every corresponding limiting controller attains the optimal task value $V^{*}_{\Omega}$ on $\Omega$ while selecting only kernel-preserving actions.
\end{proof}

Intuitively, the sufficient margin scale $K^*$ scales the safety critic \textit{just enough} (i.e., based on $\delta_{\Omega}$) to ensure that the safety truncation is never active (as illustrated in Fig. \ref{fig:fig_joint_and_region}). Consequently, if $\delta_\Omega \rightarrow 0$, i.e., if the safe policy approaches $\partial\Omega$ arbitrarily closely, then $K^* \rightarrow \infty$. If $\delta_\Omega = 0$, Theorem \ref{thm:convergence_fixed_pt} no longer guarantees convergence to the optimal safety-constrained task value for any finite $K$; we can characterize the resulting finite-$K$ conservatism as follows. Let $\Gamma_{\Omega}$ be the set of kernel-preserving controller policies, and for $\zeta\ge 0$, define:
\begin{align}
   \Gamma_{\Omega, \zeta} : = \left\{ \mu\in\Gamma_{\Omega} : \inf_{\nu}\inf_{s\in\Omega} V^{\phi, 1}_{(\mu, \nu)}(s) \ge \zeta \right\}  \qquad V^{\ast}_{\Omega, \zeta}(s) := \sup_{\mu\in\Gamma_{\Omega, \zeta}} \inf_\nu V^{\psi}_{\mu, \nu}(s)
\end{align}
Let $\Delta_{\mathrm{marg}}(\zeta):=\Vert V^*_{\Omega} - V^*_{\Omega, \zeta} \Vert_\infty$, and let $V^*_{K}$ denote the optimal joint value of Eq. \ref{eq:joint_st_value_eqn} over kernel-preserving controllers, for some finite $K\ge \underline{K}$. 

\finitekchar*

\begin{proof}
For a given policy $\pi=(\mu, \nu)$, let $V^K_{\mu, \nu}$ denote the unique fixed point of the finite-$K$ joint Bellman equation of Eq. \ref{eq:joint_bellman}, written as:
\begin{align}
    \mathcal{T}_{V^{\phi, 1}_{\mu, \nu}}[v](s) = (1-\gamma) \min\{r_{\mu, \nu}(s), Kg(s)\} + \gamma\min\left\{ r_{\mu, \nu}(s) + \mathbb{E}[v(s')], KV^{\phi, 1}_{\mu, \nu}(s) \right\}
\end{align}
Define $V^*_K(s):=\sup_{\mu\in\Gamma_\Omega}\inf_\nu V^K_{\mu,\nu}(s)$. Then, we have the following:
\begin{enumerate}
    \item Let $\mathcal{T}^\psi_{\mu, \nu}[v](s):= r_{\mu, \nu}(s) + \gamma\mathbb{E}[v(s')]$ be the task Bellman operator of Eq. \ref{eq:robust_det_bell_eqn}. Note that by definition of $\min\{\cdot\}$, we have $\mathcal{T}_{V^{\phi, 1}_{\mu, \nu}}[v] \le \mathcal{T}^\psi_{\mu, \nu}[v]$; taking $\inf_\nu$ and $\sup_{\mu\in\Gamma_\Omega}$ on both sides yield:
    \begin{align}
        V^*_{K}(s) \le \sup_{\mu\in\Gamma_\Omega}\inf_\nu V^\psi_{\mu, \nu}(s) = V^*_\Omega(s).
    \end{align}
    Thus, finite-$K$ truncation can only \textit{decrease the task return} obtained by the agent. 

    \item Next, define:
    \begin{align}
        K^* := \frac{R_{\mathrm{max}} + \overline{V}}{\delta_{\Omega}} = \frac{(2-\gamma)R_\mathrm{max}}{(1-\gamma)\delta_\Omega}, \qquad \overline{V}:=\frac{R_\mathrm{max}}{1-\gamma} \label{eq:optimal_K_star}
    \end{align}
    Next, note that, for $\mu_{\zeta_K}\in\Gamma_{\Omega, \zeta_K}$, $V^{K}_{\mu_{\zeta_K}, \nu} = V^\psi_{\mu_{\zeta_K, \nu}}$ i.e., if the controller $\mu$ is uniformly "$\zeta_K$ units" away from $\partial\Omega$, scaling by $K$ prevents safety truncation to be active. This argument follows analogously from Eq. \ref{eq:optimal_K_star} and Eq. \ref{eq:safety_trunc_inactive}. Then, the following holds:
    \begin{subequations}
    \begin{align}
        V^*_{K}(s) &= \sup_{\mu\in\Gamma_\Omega}\inf_\nu V^{K}_{\mu, \nu}(s) \\
        & \ge \sup_{\mu\in\Gamma_{\Omega, \zeta_K}}\inf_\nu V^K_{\mu, \nu}(s) \quad \text{(finite-$K$ suboptimality)} \\
        &= \sup_{\mu\in\Gamma_{\Omega, \zeta_K}}\inf_\nu V^{\psi}_{\mu_{\zeta_K},\nu}(s) \\
        &= V^*_{\Omega, \zeta_K}(s)
    \end{align}
    \end{subequations}
    Combining the above inequalities, we have:
    \begin{subequations}
    \begin{align}
        V^*_{\Omega, \zeta_K}(s) \le &V^*_K(s) \le V^*_\Omega(s) \\
        \hspace{-5cm}\implies 0\le V^*_\Omega(s) - V^*_K(s) &\le V^*_\Omega - V^*_{\Omega, \zeta_K}(s) \\
        & \le \Delta_{\mathrm{marg}}(\zeta_K)
    \end{align}
    \end{subequations}
    \item Finally, as $K\rightarrow\infty$, $\zeta_K = \frac{R_{\mathrm{max}} + \overline{V}}{K}\rightarrow 0$, and thus, if $\Delta_{\text{marg}}(\zeta)\rightarrow 0$ as $\zeta\downarrow 0$, then:
    \begin{align}
        \left\Vert V^*_{\Omega} - V^*_K \right\Vert_\infty \le \Delta_{\text{marg}}(\zeta_K) \rightarrow 0 \implies V^*_K\rightarrow V^*_\Omega\qquad\text{uniformly on $\Omega$}
    \end{align}

    \item To characterize $\overline{G}(\cdot)$ in this regime, consider the following analysis. For some $q\in\mathcal{Q}_{\mathrm{cfg}}$, let $\mathscr{q}=q\vert_\mathcal{I}$ with $x_i=(s,a, d)\in\mathcal{I}$. Since $q\in\mathcal{Q}_\mathrm{cfg}$, we have $\mathcal{T}_{\Omega}[\mathscr{q}](x_i) = r(x_i) + \gamma\mathbb{E}[V_q(s')]$. Then, using the property $\min\{a, b\} = a - [a-b]^+$, we have:
    \begin{subequations}
    \begin{align}
        \mathcal{T}_{y}[q](x_i) &= (1-\gamma)\min\{ r(x_i), Kg(s) \} + \gamma\min\{ r(x_i) + \mathbb{E}[V_q(s')], y(x_i) \} \\
        &= \mathcal{T}_{\Omega}[\mathscr{q}](x_i) - (1-\gamma)\left[ r(x_i) - Kg(s) \right]^+ - \gamma\left[ r(x_i) + \mathbb{E}[V_q(s')] - y(x_i) \right]^+ 
    \end{align}
    \end{subequations}
    Hence, the slow-timescale Bellman residual is given by:
    \begin{subequations}
    \begin{align}
        &\qquad G_K(y, q)\vert_\mathcal{I} = \mathcal{T}_\Omega[\mathscr{q}] - \mathscr{q} - e_K(y, q) \\
        \text{where, }\quad e_K(y, q)&:= (1-\gamma)\left[ r(x_i) - Kg(s) \right]^+ + \gamma\left[ r(x_i) + \mathbb{E}[V_q(s')] - y(x_i) \right]^+
    \end{align}
    \end{subequations}
    Averaging over an admissible occupation measure yields:
    \begin{subequations}
    \begin{align}
        &\qquad \overline{G}_K(q)\vert_\mathcal{I} = \left\{\mathcal{T}_\Omega[\mathscr{q}] - \mathscr{q} - e^{\mathbb{P}}_K(q): \mathbb{P}\in\mathfrak{D}_K(q)\right\} \\
        \text{with }\ e^\mathbb{P}_K(q)&:= (1-\gamma)\left[ r(x_i) - Kg(s) \right]^+ + \gamma\int_{\Lambda_K(q)}\left[ r(x_i) + \mathbb{E}[V_q(s')] - y(x_i) \right]^+ \mathbb{P}(dy)
    \end{align}
    \end{subequations}
    Note that when $\delta_\Omega > 0$ i.e., when $K\ge K^*$, we have $y(x_i) \ge r(x_i) + \mathbb{E}[V_q(s')]$ with $Kg(s) \ge r(x_i)$, yielding $e_K(y, q) = e_K^{\mathbb{P}}(q) = 0$, resulting in $\left.\overline{G}(q)\middle\vert_{\mathcal{I}}\right. = \left\{ \mathcal{T}_{\Omega}[q\vert_{\mathcal{I}}] - q\vert_{\mathcal{I}} \right\}$.
\end{enumerate}

\end{proof}

\subsection{Extension to Asynchronous settings}\label{app:asynchronous_ext}
In this section, we discuss how our proposed framework extends to asynchronous coordinate updates using the \textit{asynchronous stochastic-approximation} construction of \citep{perkins2013asynchronous}. Let $\mathcal{E}_Q$ and $\mathcal{E}_Y$ denote the finite coordinate sets of joint and safety critics; at iteration $n$, only subsets $E^Q_n\subseteq\mathcal{E}_Q$ and $E^Y_n\subseteq\mathcal{E}_Y$ need to be updated. For each coordinate $x\in \mathcal{S}\times\mathcal{A}\times\mathcal{D}$, let:
\begin{align}
    N_n^Q(x):=\sum_{k=1}^n \mathbbm{1}_{\left\{x\in E^Q_k \right\}}, \quad N_n^Y(x):= \sum_{k=1}^n \mathbbm{1}_{\left\{ x\in E_k^Y \right\}}
\end{align}
be its local update counters, with corresponding local step sizes $\alpha_{N_n^Q(x)}$ and $\beta_{N^Y_n(x)}$ respectively. We assume the standard asynchronous regularity conditions of \citep{perkins2013asynchronous}: 
\begin{itemize}
    \item[-] update process is recurrent so that every coordinate gets a uniformly asymptotic fraction of updates,
    \item[-] the local step-size sequences satisfy the required regularity and noise conditions,
    \item[-] the effective two-timescale separation is preserved.
\end{itemize}
Then, there exists $\epsilon_Q, \epsilon_Y > 0$ such that the \textit{effective relative update rates} belong to:
\begin{align}
    \mathfrak{M}^m_{\epsilon_Q} := \left\{ M=\mathrm{diag}(\lambda_1, \dots, \lambda_m) : \lambda_i\in[\epsilon_Q, 1] \right\} \quad \text{with $\epsilon=\epsilon_Q$ or $\epsilon_Y$}
\end{align}
Consequently, the limiting mean fields get premultiplied by positive diagonal rescaling matrices. In particular, for a quasi-static $q$, the fast-timescale differential inclusion becomes:
\begin{align}
    \dot{y}(t) \in \mathfrak{M}^{|\mathcal{E}_Y|}_{\epsilon_Y} \Psi(q, y(t)) := \left\{ Mh_\pi : M\in\mathfrak{M}^{|\mathcal{E}_Y|}_{\epsilon_Y}, h_\pi(y(t)) \in \Psi(q, y(t)) \right\}
\end{align}
and its globally attracting set replaces $\Lambda(q)$ in the two-timescale analysis. The corresponding slow-timescale limiting field is obtained analogously by averaging the diagonally rescaled Bellman residual over the admissible occupation measures.

Note that this rescaling (due to asynchronity) does not alter any qualitative conclusions made in the synchronous setting. As every diagonal entry is positive, the coordinatewise inequalities retain their sign, with only their contraction rates being reduced at most by the corresponding $\epsilon_Q, \epsilon_Y$ factors. For instance, Eq. \ref{eq:dini_zeta_y} and Eq. \ref{eq:dini_mathfrak_q_Q_star} become:
\begin{align}
    \mathbb{D}^+\zeta(y(t)) \le -\epsilon_Y\left(1-\gamma\right)\zeta(y(t)), \qquad \mathbb{D}^+\Vert \mathfrak{q}(t) - Q^*_\Omega \Vert_\infty \le -\epsilon_Q(1-\gamma)\Vert \mathfrak{q}(t) - Q^*_\Omega \Vert_\infty
\end{align}
Hence, the asynchronous extension preserves the limiting safety and optimality results; only the effective convergence rates alter.

\section{Additional Experimental Details and Results}\label{app:additional_exp_section}
In this section, we provide more details of Sec.~\ref{sec:empirical_evaluation}, where we choose to answer the following questions:
(1) \textit{Does the proposed two-timescale learning framework achieve stable convergence on continuous-control tasks with TD bootstrapping via deep neural approximations?} (2) \textit{How does our joint policy perform against well-established CMDP-based Safe RL benchmarks, particularly concerning its ability to achieve near-zero violations during evaluation without sacrificing task performance?} (3) \textit{How does the set of safe states under the joint policy dynamically expand during the learning process?}


While our theoretical analysis establishes that the joint policy corresponding to the exact fixed point of $Y^* = Q^\phi_{\pi_{Q^*}}, \quad Q^* = \mathcal{T}_{Y^*} Q^*$ \textit{hard-guarantees} safety and \textit{optimal} safety-constrained task performance within the safe set, computing it using exact Dynamic Programming solutions is seldom feasible for high-dimensional continuous state spaces; hence, to scale our framework, we adopt techniques from deep RL to obtain \textit{neural approximations} to our joint value function and policy. In particular, we closely follow the ISAACS framework \citep{hsu2023isaacs} with \textit{gradient descent-ascent finite timescale separation} ($\tau$-GDA) \cite{fiez2021local} (between the adversarial disturbance and controller), built on the Soft Actor-Critic algorithm \citep{haarnoja2018soft}, and call it \texttt{JointSAC}.

\subsection{Low-dimensional Tabular Experiment}\label{app:subsec_low_dim_tabular}
In this subsection, we consider a low-dimensional dynamical system, namely the \textit{double integrator}, and provide \textit{exact} computation of the Joint Value function using tabular Dynamic Programming.
\textbf{Double-integrator System:}
We consider a deterministic double integrator with state $z=(x,v)$ and control $a$:
\begin{align}
v^{+}=v+\Delta t\,a \qquad x^{+}=x+\Delta t\,v^{+} \qquad |a|\leq 1.
\end{align}
The safety margin is defined as $g(x,v)=1-|x|$, so the system must remain in $|x|<1$. The state domain is given by the box $x\in[-2.5,2.5]$, $v\in[-3,3]$. The system is rewarded for moving to $x_g=1.2$, with the reward signal:
\begin{align}
r(x,v,a) = \exp\!\left( -\frac{|x-x_g|^2}{\sigma_x^2} -\frac{v^2}{\sigma_v^2} -\lambda_a |a|^2 \right) \in(0,1],
\end{align}
where $\sigma_x=0.8$, $\sigma_v=1.0$ indicate the \textit{bump width} of the reward signal over $x_g=1.2$, $v_g=0$ respectively, and $\lambda_a=1.0$ denotes the cost on control effort. Thus, the task encourages exiting the right boundary, while safety requires braking early enough to remain inside $|x|<1$.

\medskip
\noindent
\textbf{Discretization:}
We use $\Delta x=0.01$, $\Delta v=0.05$, $\Delta a=0.25$, and $\Delta t=0.2$, yielding $60{,}379$ states and $9$ actions. All value functions and policies are computed by full tabular dynamic programming to numerical convergence, with no neural network approximations.

\medskip
\noindent
\textbf{Safe set $\Omega$ and admissible actions: $A_{\Omega}(z)$:}
For reference, the continuous-time safe set of the disturbance-free double integrator is
\begin{align}
\Omega_{\mathrm{ct}} = \left\{(x,v): |x|\leq 1,\quad \left|x+\frac{v|v|}{2}\right|\leq 1 \right\}.
\end{align}
The experiment itself uses the exact kernel $\Omega_h$ of the discretized dynamics. It is computed by finite backward reachability:
\begin{align}
S_0 = z:g(z)>0, \qquad S_{k+1} = S_k \cap \left\{ z:\exists a\in A,\ f_h(z,a)\in S_k \right\}.
\end{align}
Because the state space is finite, this sequence terminates at its greatest fixed point $\Omega_h$. The admissible action set is then
\begin{align}
A_{\Omega_h}(z)
=
\left\{
a\in A:f_h(z,a)\in\Omega_h
\right\},
\end{align}
which is nonempty for every $z\in\Omega_h$ by construction.

\medskip
\noindent
\textbf{Decoupled baselines:}
We independently compute an unconstrained task critic $Q_{\mathrm{task}}$ and a task-oblivious optimal safety critic $Q_{\mathrm{safe}}^*$. We compare against three runtime compositions:
\begin{enumerate}
    \item A ``\textit{hard mask}'':
    \begin{align}
    \mu_{\mathrm{mask}}(z) \in \argmax_{a:Q_{\mathrm{safe}}^*(z,a)\geq 0} Q_{\mathrm{task}}(z,a).
    \end{align}
    We verify exhaustively that
    \begin{align}
    \left\{a:Q_{\mathrm{safe}}^*(z,a)\geq 0\right\} = A_{\Omega_h}(z)
    \end{align}
    for every viable state-action pair. Hence, any performance difference is caused by the ranking of admissible actions, not by an inaccurate safety critic or a mismatched feasible set.

    \item A least-restrictive filter, which executes the unconstrained greedy action when it is admissible and otherwise switches to $\arg\max_a Q_{\mathrm{safe}}^*(z,a)$.

    \item A nearest-safe-action, Q-CBF-style projection, which selects the admissible action closest to the unconstrained greedy action.
\end{enumerate}

The oracle value $Q^*$ is obtained by directly running value iteration, restricted to $A_{\Omega_h}(z)$ (unlike the hard mask, where $Q_{\mathrm{task}}$ is first obtained).

\medskip
\noindent
\textbf{Computation of the joint value:}
On the fast timescale (inner loop), we compute Eq.~(5) in the paper for a given joint value function (and its joint policy $\pi_{Q_k}$) until convergence, and on the slow timescale (outer loop), we compute Eq.~(6) in the paper, given the previously converged safety critic $Y_k$. This implements the two-timescale separation corresponding to $\alpha_n/\beta_n\rightarrow 0$.

The following table provides optimality gap characterization (mean, median, max, and violations) for each method, where the optimality gap is defined as
\begin{align}
\Delta = V_{\Omega_h}^{*} - V^{\mathrm{method}}(z):
\end{align}

\begin{center}
\begin{tabular}{lrrrr}
\hline
\textbf{Policy}
& \textbf{Mean}
& \textbf{Median}
& \textbf{Max}
& \textbf{Violations} \\
\hline
Oracle
& 0 & 0 & 0 & 0 \\

Hard mask
& 0.301 & 0.304 & 0.507 & 0 \\

$Q$-CBF
& 0.301 & 0.304 & 0.507 & 0 \\

Least-restrictive
& 0.452 & 0.442 & 0.820 & 0 \\

\textbf{Joint} $(K=2\times10^{4})$
& 0 & 0 & 0 & 0 \\

\textbf{Joint} $(K=1\times10^{4})$
& 0.541 & 0.5366 & 3.0394 & 0 \\

\textbf{Joint} $(K=5\times10^{3})$
& 1.1076 & 1.1 & 3.5627 & 0 \\

\textbf{Joint} $(K=1\times10^{3})$
& 5.6149 & 5.6476 & 11.2038 & 0 \\
\hline
\end{tabular}
\end{center}

We see that an exact runtime composition of independently computed reward and safety critics is already near-optimal. Our joint framework recovers the viability-restricted optimum at a large margin scale ($K$). The improvement over the hard mask and the least-restrictive filter is likely due to propagating the policy-conditioned safety truncation backward through the Bellman updates; earlier actions are ranked according to their safety-constrained evaluation (conditional on overcoming finite-$K$ conservatism).

\subsection{Safe Set Expansion}\label{subsec:safe_set_expansion}
Since the safety value function $V^\phi_\pi$ evaluates the reward-maximizing joint policy $\pi$, in theory, the safe set $\left\{s:V^\phi_\pi(s) > 0\right\}$ expands in the direction of maximum task return, before asymptotically reaching the maximal safe set. We demonstrate this on \texttt{InvertedPendulum-v4}, rewarding positive cart velocity within $\pm0.8$m limits. Thus, the agent slides towards the positive $x$-limit, after which it stays there, balancing the pole. The safe set expansion towards $0.8$m during training is illustrated in Fig.~\ref{fig:safet_set_expansion}.

\begin{figure}[!htb] \vspace{-3mm}
    \centering
    \includegraphics[width=0.24\linewidth]{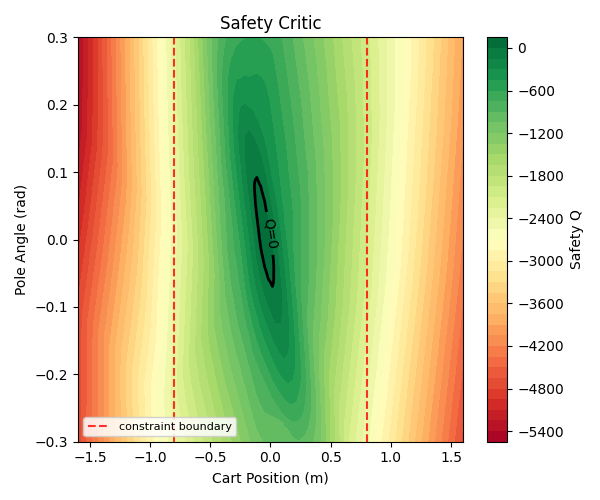}
    \includegraphics[width=0.24\linewidth]{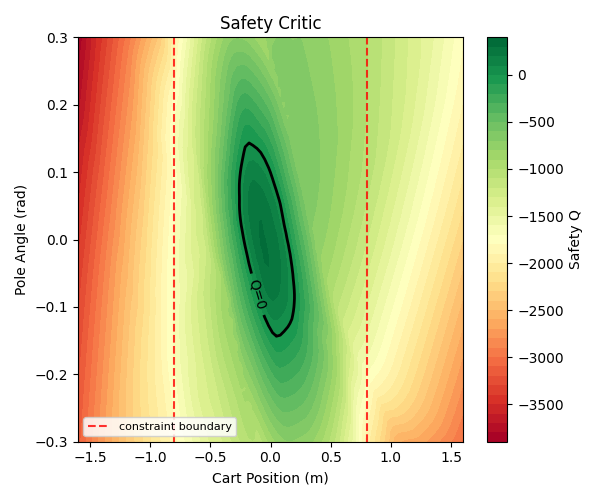}
    \includegraphics[width=0.24\linewidth]{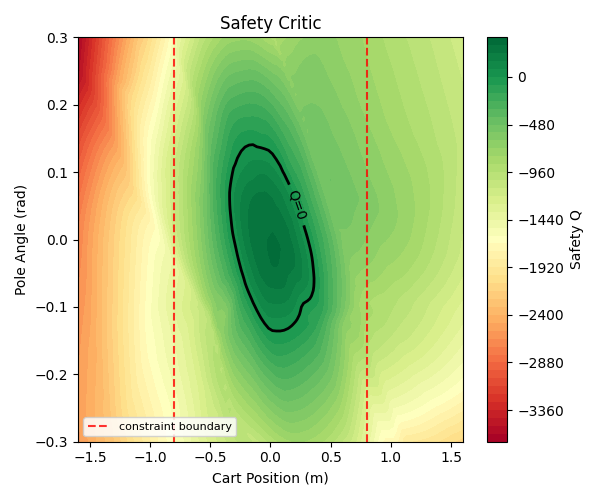}
    \includegraphics[width=0.24\linewidth]{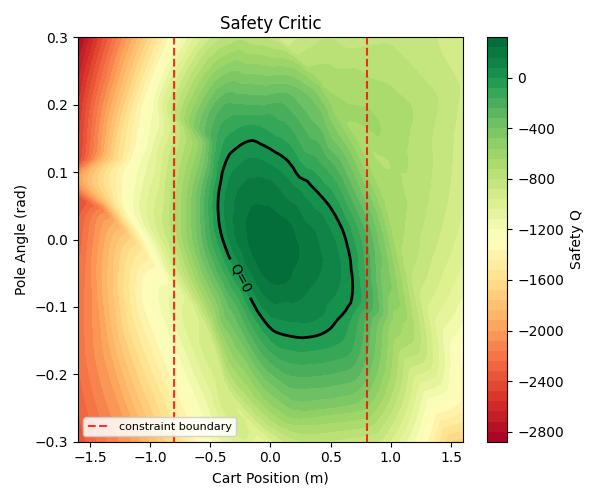}\vspace{-2mm}
    \caption{Safe set expands initially in the direction of high rewards during training (\texttt{InvertedPendulum-v4} rewarded for positive cart velocity), before converging to the maximal safe set. Constraint boundary (red dotted line) and $0$-level set of safety critic (black line) are visualized.}\vspace{-4mm}
    \label{fig:safet_set_expansion}
\end{figure}

\subsection{Continuous Control Experimental Setup}\label{app:exp_setup}
We evaluate our proposed framework across a suite of high-dimensional continuous control tasks implemented in the \texttt{MuJoCo} simulator \citep{todorov2012mujoco}, utilizing environments from both \texttt{Gymnasium} \citep{towers2026gymnasium} and \texttt{SafetyGymnasium} \citep{ji2023safety}. The primary objective in all tasks is to maximize returns while strictly satisfying safety constraints at all times.

\textbf{Simulation Tasks} We consider three locomotion environments from \texttt{Gymnasium MuJoCo} suite  \citep{towers2026gymnasium} and eight single agent environments from the \texttt{SafeVelocity} suite from \texttt{Safety-Gymnasium}\citep{ji2023safety}. 
\begin{enumerate}
    \item \texttt{Gymnasium} (Liveness constraints): Here, safety is implicitly defined by the environment's \textit{liveness bounds}, which mandate postural stability to avoid premature episode termination while moving forward. To rigorously test the safety guarantees of our formulation, we augment the environment with a \textit{physical adversary}. The adversary applies continuous bounded disturbances directly to the agent's torso. Thus, while the controller attempts to safely maximize forward locomotion, the adversary actively seeks to destabilize the agent and prematurely terminate the episode.
    \item \texttt{SafeVelocity} (Conflicting Objectives): To evaluate our framework under explicitly conflicting reward-cost signals, we utilize the \texttt{SafeVelocity} suite of \texttt{SafetyGymnasium}. Here, agents receive monotonically increasing rewards for forward velocity, but incur safety violations if they exceed a \textit{predefined safety limit}. As with standard Gymnasium tasks, we subject these agents to adversarial torso disturbances, causing the agent to either prematurely terminate (violating liveness bounds) or push the agent over the speed limits.
\end{enumerate}
\textbf{Baselines}
We compare our method against a safety filter baseline (least-restrictive safety filter) and three well-established CMDP baselines. We specifically select these CMDP baselines, as they represent the two dominant paradigms of model-free Safe RL: \textit{Lagrangian dual-descent} and \textit{trust-region projections}:
\begin{enumerate}
    \item \textbf{Safety Filters (\texttt{LRSF})}: Safety filters \citep{hsu2024safety_filter, fisac2019bridging} serve as the standard Hamilton-Jacobi-reachability-based baseline, where a \textit{task-oblivious} safety value function (often termed a ``safety critic'') is learned for an entire \textit{operational design domain} (ODD), where the constraint set (and hence its complement, the failure set) is often encoded by a margin function. Given any \textit{safety-oblivious} task policy, one can query for \textit{the safety} of its actions (by querying the safety-action value function). If the action is deemed unsafe, only then does the \textit{filter intervene} by choosing a greedy action w.r.t. the safety-action value function. Hence, this style of intervention is called \textit{least-restrictive} (LR-safety filters). 
    \item \textbf{Projection-based Constrained Policy Optimization (\texttt{PCPO})}: \texttt{PCPO} \citep{yang2020projection} serves as the standard projection-based baseline, employing a two-step trust region framework. It first optimizes the reward advantage function subject to a Kullback-Leibler (KL) divergence constraint, ignoring safety entirely. 
    In the second step, it explicitly projects this intermediate policy onto a feasible safe set. To preserve task performance, this projection minimizes either an $L^2$ norm or a KL-divergence distance penalty with a second-order Taylor expansion.
    \item \textbf{PPO-Lagrangian (\texttt{PPOLag})}: \citet{ray2019benchmarking, ji2024omnisafe} combine the Lagrangian multiplier approach with Proximal Policy Optimization (PPO) to obtain \texttt{PPOLag}. It enforces safety bounds using dynamic penalty coefficients, solving the equivalent min-max optimization problem via alternating gradient ascent on the reward objective and dual gradient descent on the constraint violation costs.  
    \item \textbf{SAC-Lagrangian \texttt{(SACLag)}}: This baseline (\texttt{SACLag}) \citep{ray2019benchmarking, ji2024omnisafe} is
     conceptually identical to \texttt{PPOLag}, but replaces the underlying architecture with Soft-Actor Critic \citep{haarnoja2018soft}. 
\end{enumerate}
\textbf{Robust Adversarial Integration}: To incorporate adversarial disturbances, we modify all baselines into adversarial \textit{zero-sum} formulations using the principles of Robust Adversarial RL \citep{pinto2017robust}. We introduce a secondary adversarial disturbance $\nu(s',a')$ that can observe the controller's actions $a' \in \mu(s')$ and injects \textit{adversarial} bounded perturbations into the environment dynamics. This yields robust versions of all three baselines, where the controller seeks to maximize the dynamic Lagrangian objective (maximize rewards, minimize violation costs), while the adversarial disturbance actively attacks and minimizes it.

\textbf{Metrics} We use cumulative \textit{Episodic Return} and \textit{Episodic Cost} tracked across training to evaluate stable convergence. We also use \textit{cumulative reward} and \textit{episode cost} as the evaluation metrics. For Lagrangian-based methods, a permissive cost violation threshold of $d=0$ is allowed. We take safety as the primary criterion for evaluations, and pursue higher rewards on the basis of meeting safety requirements.

\textbf{Compute} 
We train our algorithms on a \texttt{slurm} cluster with a single Nvidia L40, 48GB of VRAM. Each run takes 4 CPU cores. We implement the code in \texttt{PyTorch}, and all experiments take up to 4 days to fully train for $5$M timesteps.

\textbf{Hyperparameters} 
\begin{itemize}
    \item We choose seeds 5 training seeds: $[0, 5, 10, 15, 20]$ as the algorithms in \texttt{SafetyGymnasium} have been reported to be stable on \href{https://github.com/PKU-Alignment/safety-gymnasium/tree/main/benchmarks#important-notes}{these seeds}. 
    \item We tune the scaling of the margin function $\beta(g(s)) = K\cdot g(s)$ based on the maximum task return possible for each environment. Additionally, we pick the maximum value of $K$ such that there are no safety violations / algorithmic instabilities. Specifically, we choose the following:
    
    \noindent
    \begin{enumerate}
        \item \texttt{Hopper-v4}: $800.0$
        \item \texttt{Walker2d-v4}: $1200.0$
        \item \texttt{InvertedPendulum-v4}: $200.0$
        \item \texttt{SafetyWalker2dVelocity}: $1200.0$
        \item \texttt{SafetyHopperVelocity}: $800.0$
        \item \texttt{SafetyHalfCheetahVelocity}: $1200.0$
        \item \texttt{SafetyAntVelocity}: $1200.0$
    \end{enumerate}
    We pick the disturbance magnitude (applied to the torso) that is meaningfully challenging to the agent. Any higher value resulted in high safety violations for all algorithms:
    \begin{enumerate}
        \item \texttt{InvertedPendulum-v4}: $5.0$ 
        \item \texttt{Walker2d-v4}: $30.0$
        \item \texttt{Hopper2d-v4}: $7.5$
        \item \texttt{SafetyWalker2dVelocity}: $25.0$
        \item \texttt{SafetyHopperVelocity}: $5.0$
        \item \texttt{SafetyHalfCheetahVelocity}: $20.0$
        \item \texttt{SafetyAntVelocity}: $2.5$
    \end{enumerate}
    \hfill
    
    \item For 
    For \texttt{JointSAC}, we use a safety-critic learning rate of $1.0\times10^{-3}$, a joint-critic learning rate of $5\times10^{-4}$, a controller actor learning rate of $1\times10^{-4}$, a disturbance learning rate of $5\times10^{-4}$, and autotune alpha (entropy term) with a learning rate of $5\times10^{-4}$. We implement timescale separation through the critic learning rates: the joint critic is updated at half the learning rate of the safety critic, with both critics updated at every step (this led to stable training while completing in reasonable amount of time, requiring around $5$M training steps; \textit{polyak averaging} ($\tau = 5\times10^{-3}$) is used only for the target networks of both critics. Finally, we use different discount factors $\gamma$ for the joint Bellman equation and $\widetilde{\gamma}$ for the safety Bellman equation; we anneal $\widetilde{\gamma}$ from $0.85$ to $\gamma=0.99$ while training progresses, as discussed in \citep{fisac2019bridging}.
\end{itemize}

\subsection{Margin Scale $K$ Variation}\label{app:K_ablation}

We perform an ablation experiment on the variation of episode rewards and cost as $K$ increases progressively (in steps of $200$, from $K=200$ to $K=2000$). Results are illustrated in Fig.~\ref{fig:K_ablation} for $4$ environments - \texttt{Hopper, Walker2d} (from \texttt{Gymnasium} suite), and \texttt{SafetyHalfCheetah, SafetyWalker2d} (from \texttt{SafetyGymnasium} suite). Training was conducted on $5$ seeds, and each training seed was evaluated for $300$ unique episodes, totaling $1,500$ unique evaluation episodes.

\begin{figure}[!htb]
    \centering
    \includegraphics[width=0.24\linewidth]{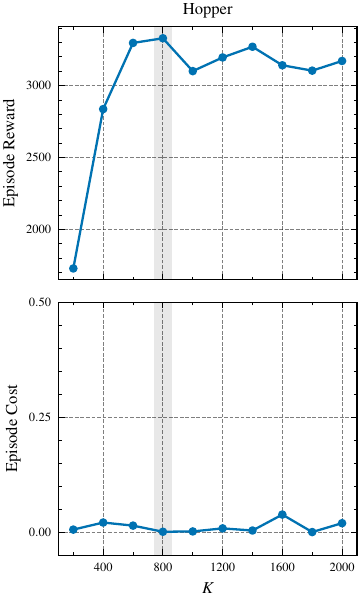}
    \includegraphics[width=0.24\linewidth]{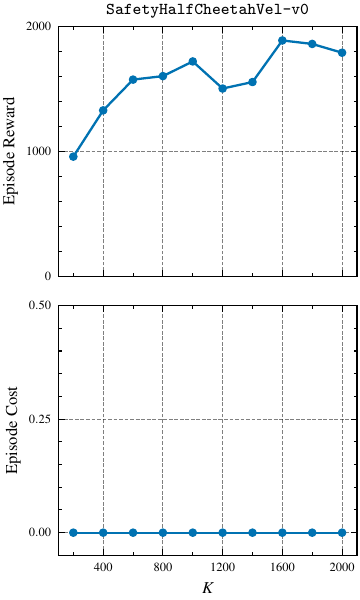}
    \includegraphics[width=0.24\linewidth]{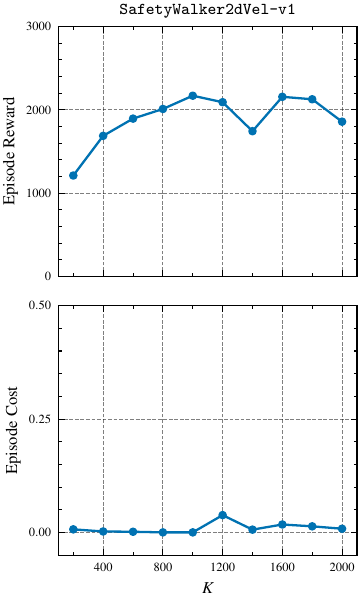}
    \includegraphics[width=0.24\linewidth]{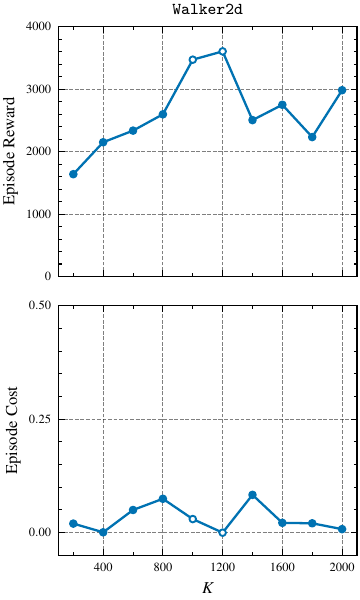}
    \caption{Figure shows variation of episode reward and episode cost as $K$ increases from $200$ to $2000$. While episode reward generally tends to increase with an increase in $K$, scaling beyond a point can introduce \textit{training instabilities} with neural approximations, as seen with  \texttt{Walker2d} environment.}
    \label{fig:K_ablation}
\end{figure}

The general trend agrees with our previous discussion in Sections~\ref{sec:two_timescale_approx}-\ref{sec:convergence}: scaling $K$ \textit{generally improves} episode reward (as the region $\Omega^{\psi}$ increases), while safety is \textit{generally preserved}, with small variations. However, one can also observe that increasing $K$ too much may lead to \textit{training instabilities} (as seen in \texttt{Walker2d} environment). 

\subsection{Additional Results}\label{app:additional_results}
The training curves for 5 seeds on:
\begin{itemize}
    \item three locomotion tasks from \texttt{Gymnasium} suite incorporating adversarial disturbance is presented in Fig.~\ref{fig:training_curves_gym_robust}
    \item eight single agent environments from \texttt{SafeVelocity} (\texttt{v0}) suite from \texttt{Safety-Gymnasium} incorporating adversarial disturbance is presented in Fig.~\ref{fig:training_curves_safetygym_robust}
\end{itemize}
These results align with the results presented in Sec.~\ref{subsec:main_results}. Throughout training, \texttt{JointSAC} demonstrates stable convergence to high task rewards, while consistently incurring near-zero episode costs. 
A single cost violation incurs $+1$ penalty each timestep. The baseline methods exhibit higher variance, with only a subset of them reliably converging below the permissible cost threshold ($d=0$). 

For evaluations, we perform $300$ rollout evaluations for each of the $5$ training seeds (totaling $1,500$ distinct episodes per environment), with results summarized in Table~\ref{tab:robust_true}. 
Results indicate that \texttt{JointSAC} maintains near-zero costs across all 11 environments (maximum of $0.03\pm0.08$ on the \texttt{Hopper-v4} adversarial environment), while the baseline exhibits varied cost violations, occasionally achieving high episode costs for some environments and staying below the permissible cost violation for others. Overall, it is seen that \texttt{JointSAC} yields stable performance while maintaining near-zero violations.



\begin{figure}[!htb]
    \centering
    \includegraphics[width=1.0\linewidth]{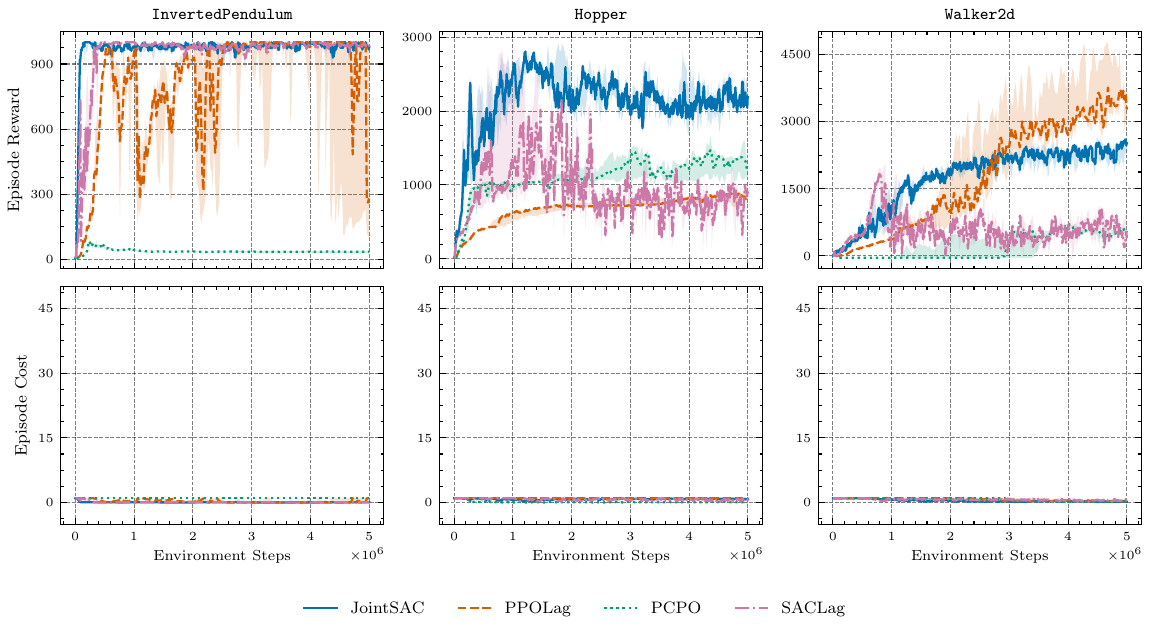}
    \caption{Training performance across 5 seeds (mean $\pm$ standard deviation) on three locomotion tasks from \texttt{Gymnasium} suite with adversarial disturbances. The top row indicates episodic returns, while the bottom row indicates episodic costs as training progresses. \texttt{JointSAC} demonstrates stable convergence with near-zero episode costs on two out of three tasks. 
    A single cost violation incurs +1 penalty each timestep. }
    \label{fig:training_curves_gym_robust}
\end{figure}

\begin{figure}[!htb]
    \centering
    \includegraphics[width=1.0\linewidth]{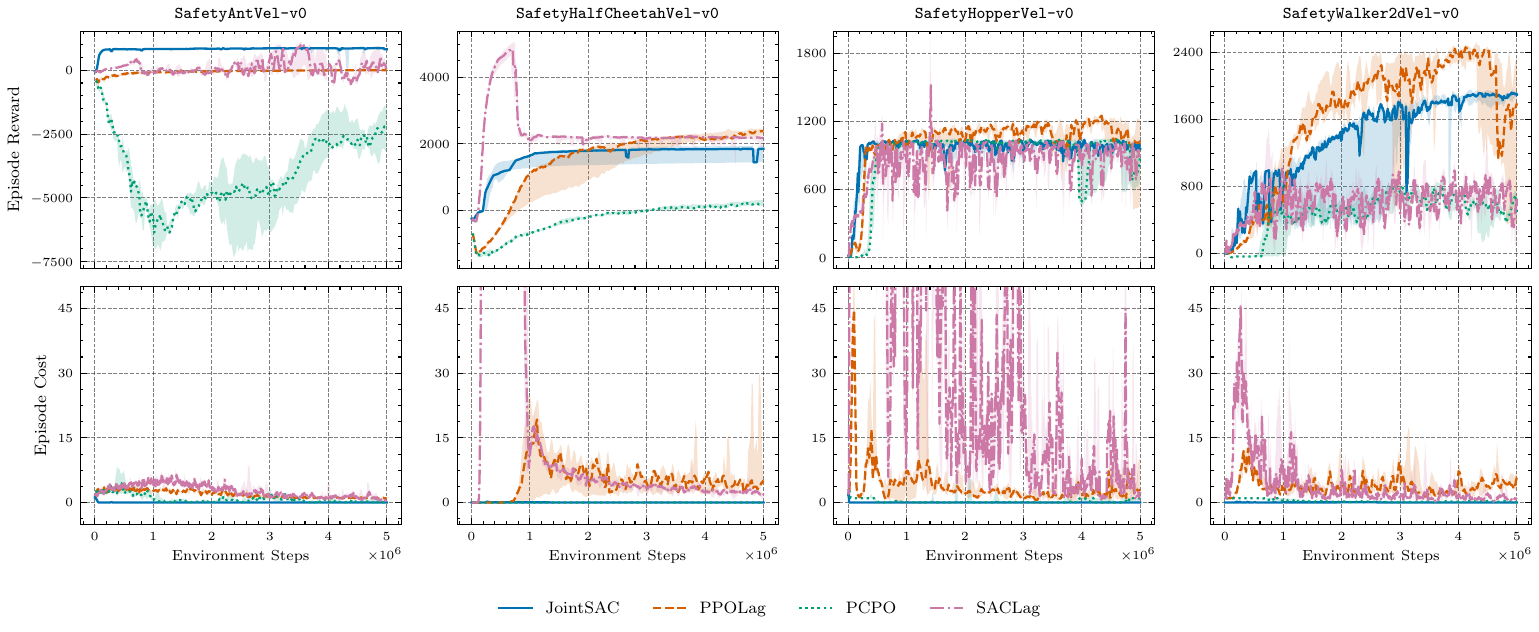}
    \caption{Training performance across 5 seeds (mean $\pm$ standard deviation) on eight single agent environments from \texttt{SafeVelocity} (\texttt{v0}) suite from \texttt{SafetyGymnasium} with adversarial disturbances. The top two rows indicate episodic returns, while the bottom two rows indicate episodic costs as training progresses. \texttt{JointSAC} demonstrates stable convergence with near-zero episode costs.}
    \label{fig:training_curves_safetygym_robust}
\end{figure}

\end{document}